\documentclass{article}

\usepackage{microtype}
\usepackage{graphicx}
\graphicspath{{figuretype3/}}
\usepackage{subfigure}
\usepackage{booktabs} 

\usepackage[utf8]{inputenc} 
\usepackage[T1]{fontenc}    
\usepackage{url}            
\usepackage{booktabs}       
\usepackage{amsfonts}       
\usepackage{nicefrac}       
\usepackage{microtype}      

\usepackage{enumitem} 
\usepackage{amsmath, amssymb, bm} 
\usepackage{amsthm} 
\usepackage{setspace}

\newtheorem{theorem}{Theorem}
\newtheorem{lemma}{Lemma}

\usepackage{multirow}

\newcommand{\argmin}{\mathop{\arg\min}}
\newcommand{\argmax}{\mathop{\arg\max}}
\def \X {\mathcal{X}}
\def \Y {\mathcal{Y}}
\def \D {\mathcal{D}}
\def \R {\mathcal{R}}
\def \S {\mathcal{S}}
\def \P {\mathcal{P}}
\def \G {\mathcal{G}}
\def \C {\mathcal{C}}
\def \A {\mathcal{A}}
\def \L {\mathcal{L}}
\def \RR {\mathbb{R}}
\def \EE {\mathbb{E}}
\def \II {\mathbb{I}}
\def \Ra {\mathfrak{R}}

\usepackage{hyperref}

\usepackage[accepted]{icml2020}

\icmltitlerunning{Progressive Identification of True Labels for Partial-Label Learning}

\begin{document}
	
	\twocolumn[
	\icmltitle{Progressive Identification of True Labels for Partial-Label Learning}



\icmlsetsymbol{equal}{*}
\icmlsetsymbol{intern}{$\dagger$}

\begin{icmlauthorlist}
	\icmlauthor{Jiaqi Lv}{intern,seu}
	\icmlauthor{Miao Xu}{riken,queen}
	\icmlauthor{Lei Feng}{ntu}
	\icmlauthor{Gang Niu}{riken}
	\icmlauthor{Xin Geng}{seu}
	\icmlauthor{Masashi Sugiyama}{riken,ut}
\end{icmlauthorlist}

\icmlaffiliation{seu}{School of Computer Science and Engineering, Southeast University, Nanjing, China}
\icmlaffiliation{riken}{RIKEN Center for Advanced Intelligence Project, Tokyo, Japan}
\icmlaffiliation{queen}{The University of Queensland, Australia}
\icmlaffiliation{ntu}{School of Computer Science and Engineering, Nanyang Technological University, Singapore}
\icmlaffiliation{ut}{University of Tokyo, Tokyo, Japan}

\icmlcorrespondingauthor{Xin Geng}{xgeng@seu.edu.cn}

\icmlkeywords{Machine Learning, ICML}

\vskip 0.3in
]



\newcommand{\icmlIntern}{\textsuperscript{$\dagger$}Preliminary work was done during an internship at RIKEN AIP. }
\printAffiliationsAndNotice{\icmlIntern}  

\begin{abstract}
	\emph{Partial-label learning} (PLL) is a typical weakly supervised learning problem, where each training instance is equipped with a set of candidate labels among which only one is the true label. Most existing methods elaborately designed learning objectives as constrained optimizations that must be solved in specific manners, making their computational complexity a bottleneck for \emph{scaling up to big data}. The goal of this paper is to propose a novel framework of PLL \emph{with flexibility on the model and optimization algorithm}. 
	More specifically, we propose a novel estimator of the classification risk, theoretically analyze the \emph{classifier-consistency}, and establish an estimation error bound. Then we propose a \emph{progressive identification} algorithm for approximately minimizing the proposed risk estimator, where the update of the model and identification of true labels are conducted in a seamless manner. The resulting algorithm is model-independent and loss-independent, and compatible with stochastic optimization. Thorough experiments demonstrate it sets the new state of the art.
\end{abstract}

\section{Introduction}

The increasing demand for massive data used in training modern machine learning models, such as deep neural networks (DNNs), makes it inevitable that not all examples are equipped with strong supervision information. Weak supervision may result in a detrimental effect on model training and has been widely studied \cite{zhou2017brief,patrini2017making,kamnitsas2018semi,lu2018minimal,gong2019loss}. We focus on an important weakly supervised learning problem called \emph{partial-label learning} (PLL) \cite{cour2011learning, wang2019discriminative, lyu2019gm}, where each training instance is associated with \emph{a set of candidate labels} among which \emph{exactly one} is true, reducing the overhead of finding exact label from ambiguous candidates. This problem arises in many real-world tasks such as automatic image annotation \cite{chen18learning}, web mining \cite{luo2010learning}, etc.

Related research on PLL was pioneered by a \emph{multiple-label learning} (MLL) method \cite{jin2003learning}. Despite the same form of supervision information, i.e., a set of candidate labels, that each training instance is assigned with, a vital difference between MLL and PLL is that the goal of PLL is identifying the \emph{only one} true label among candidate labels whereas for MLL identifying an \emph{arbitrary} label in the candidate label set is acceptable, i.e., treating all candidate labels equally. In practice, \citet{jin2003learning} formulated an \emph{unconstrained} learning objective to minimize the KL divergence between the class prior and the model-based conditional distribution, and solved it by EM algorithm, resulting in a procedure iterating between estimating the prior distribution and training the model.  

\citet{jin2003learning} used the class prior as the fitting target of the model, in which the label with maximum prior probability can be naturally regarded as the ``true label''. Coincidentally, the true label of training instance in PLL is obscured by candidate labels that hinder the learning. The key to success is also identifying the true label.  
Therefore, the foundational work enlightened the successors to identify the true label reasonably. Along this line, great efforts have been made in designing learning objectives with some elaborate \emph{constraints} for making assumptions on the model \cite{chen13dictionary, chen14ambiguously} or better leveraging the prior knowledge \cite{Nguyen2008classification, zhang2015solving, yu2017maximum, gong2017regularization, feng2019partial} so as to recover the true labeling information. 

However, the constrained learning objectives of the existing methods are coupled to some specific optimization algorithms. A non-linear time complexity in the total volume of data becomes a major limiting factor when these methods envision large amounts of training data. Furthermore, \citet{bottou2007the} proved \emph{stochastic optimization}, which few of the existing methods are compatible with, is the best choice for large-scale learning problems considering the estimation-optimization tradeoff. Therefore, the PLL problem is still practically challenging in \emph{scaling up to big data}. In order to magnify PLL's potential on big data, we rethink critically whether the original unconstrained method in \citet{jin2003learning} should be discarded, or is worth further studying in the deep learning era?

Our answer is in two folds. First, the decade-old method restricted to a simple linear model fed by the handcrafted features, which may be incompetent to represent and discriminate. Moreover, few PLL work hitherto can be generalized to deep network architectures. We would like to advance the previous work to enjoy the leading-edge models and optimizers from deep learning communities \cite{Goodfellow2016}. Second, the essence of our investigation is an algorithm agnostic in classification models. Thus we focus on a method that does not benefit purely from the network architecture, but also an innovative algorithm design.

In this paper, we aim at proposing a method of PLL with flexibility on the model and stochastic optimization. Towards this goal, we first derive a \emph{classifier-consistent} risk estimator, and then propose a novel algorithm that is compatible with \emph{arbitrary multi-class classifier} (ranging from linear to deep model) and stochastic optimizers. Experiments verify its superiority on synthetic and real-world partial-label datasets.
Our contributions can be summarized as follows:
\begin{itemize}[topsep=0ex,itemsep=-1ex,leftmargin=*]
	\item Theoretically, we propose a classifier-consistent risk estimator for PLL, i.e., the classifier learned from partially labeled data converges to the optimal one learned from ordinarily labeled data under mild conditions. Then, we establish an estimation error bound for it.
	\item Practically, we propose a \emph{progressive identification} method. The proposed method operates in a mini-batched training manner where the update of the model and the identification of true labels are accomplished seamlessly. This method is model-independent and loss-independent, as well as viable for any stochastic optimization (e.g., \citealp{robbins1951stochastic, duchi2011adaptive}).
\end{itemize}

\section{Background}
In this section, we formalize ordinary multi-class classification, partial-label learning, and complementary-label learning, and briefly review the related work.

\subsection{Ordinary Multi-Class Classification}
In ordinary multi-class classification, let $\X \subseteq \RR^d$ be the instance space and $\Y = [c]$ be the label space, where $d$ is the feature space dimension, $[c] := \{1,2,\ldots,c\}$ and $c > 2$ is the number of classes. Let $p(x,y)$ be the underlying joint density of random variables $(X, Y) \in \X \times \Y$. 
The target is to learn a classifier $\bm g: \X \rightarrow \RR^c$ that minimizes the estimator of the classification risk:
\begin{equation}\label{eq:risk-ordinary}
\R(\bm g) = \EE_{(X,Y) \sim p(x,y)}[\ell(\bm g(X), \bm{e}^Y)],
\end{equation}
in which $\ell: \RR^c \times \bm{e}^{\Y} \rightarrow \RR$ is a proper loss, i.e., it is continuous non-negative and $\ell(\hat\eta,\eta)=0$ only when $\hat\eta=\eta$. $\bm{e}^{\Y} = \{\bm{e}^i:i \in\Y\}$ denotes the standard canonical vector in $\RR^c$, i.e., the $i$-element in $\bm{e}^i$ equals 1 and others equal 0.
Typically, the predicted label $\widehat Y$ is assumed to take the following form:
\begin{equation*}
\widehat Y = \argmax_{i\in\Y} g_i(X),
\end{equation*}
where $g_k(\cdot)$ is the $k$-th element of $\bm g(\cdot)$ that is interpreted as an estimate of the true label probability, i.e., $g_k(X)=p(Y=k|X)$. The hypothesis in $\G$ with minimal error $\bm g^* = \argmin_{\bm g\in\G}\R(\bm g)$ is the ordinary optimal classifier. We say a method is classifier-consistent if the learned classifier by the method converges to $\bm g^*$ by increasing the example size \cite{patrini2017making, xia2019anchor}.

Since $p(x,y)$ is usually unknown, the expectation in Eq.~(\ref{eq:risk-ordinary}) is typically approximated by the average over the training examples $\{(x_i, y_i)\}^n_{i=1} \stackrel{\text{i.i.d.}}{\sim} p(x,y): \widehat{\R}(\bm g) = \frac{1}{n} \sum_{i=1}^n[\ell(\bm g(x_i), \bm{e}^{y_i})]$, and $\widehat{\bm g} = \arg\min_{\bm g\in\G} \widehat{\R}(\bm g)$ is returned by the empirical risk minimization (ERM) principle \cite{vapnik1998statistical}.

\subsection{Partial-Label Learning}
Different from ordinary multi-class classification, in PLL, the labelling information is weakened into candidate label set $S\in\S$, where $\S=\{\P({\Y})/\emptyset/\Y\}$ is the power set of $\Y$ except for the empty set and the whole label set. Hence we need to train a classifier with access only to the partially labeled examples $(X,S)$ drawn from the distribution with probability density $p(x,s)$, which is a marginal density of the complete density $p(x,y,s)$. Note that $p(x,y,s)$ can be decomposed into $p(x,s)$ and the class-probability density $p(y|x,s)$, or equivalently into the marginal density $p(x,y)$ and the class-conditional density $p(s|x,y)$. Then the basic definition for PLL that the true label $Y$ of an instance $X$ must be in its candidate label set $S$ can be formulated as
\begin{equation}\label{eq:keydef}
{\rm Pr}_{(X,Y) \sim p(x,y), S \sim p(s|x,y)}(Y \in S) = 1.
\end{equation}
The PLL risk estimator is defined over $p(x,s)$:
\begin{equation}\label{eq:origin-pll-risk}
\R_{\rm PLL}(\bm g) = \EE_{(X,S) \sim p(x,s)}[\ell_{\rm PLL}(\bm g(X),S)],
\end{equation}
where $\ell_{\rm PLL}: \RR^c \times \P({\Y}) \rightarrow \RR$. The goal of PLL is still inducing a multi-class classifier to assign the true labels for the unseen instances.

To estimate the risk in partially labeled data, defining $\ell_{\rm PLL}$ matters. \citet{jin2003learning} defined $\ell_{\rm PLL}(\bm g(X),S)={\rm KL}\big[ \widehat{\bm p}(y|X) || \bm g(X) \big]$, in which $\widehat{\bm p}(y|X)$ represents the unknown prior distribution, and iteratively solved it. Following \citet{jin2003learning}, many \emph{EM-based} methods have proposed. They added some constraints to the optimization objective for better leveraging implicit information over feature space and label space, and diambiguated candidate labels via iterative refining procedure.
\citet{chen13dictionary, chen14ambiguously} proposed methods based on dictionary learning, which seeked the sparsest dictionary mapping the feature to class prior, and then \citet{shrivastavav15non} extended the linear dictionary by the kernel trick. These methods were optimized by K-SVD algorithm \cite{aharon2011the}. \citet{feng2019partial} formulated a convex-concave problem and proved its optimization is equivalent to a set of QP problems. \citet{feng2019partialb} also formulated a QP problem to enlarge the gap of label priors between two instances with completely different candidate labels.
In addition, some \emph{SVM-based} methods with the maximum margin constraints \cite{Nguyen2008classification, yu2017maximum} that were solved by the off-the-shelf implementation on multi-class SVM \cite{fan2008liblinear}, and a few \emph{non-parametric} methods \cite{Hullermeier2006learning, zhang2015solving, gong2017regularization} were proposed. 

These methods were solved in specific low-efficient manners and incompatible with high-efficient stochastic optimization. Thus they could hardly handle large-scale datasets. To the best of our knowledge, only the latest work \cite{yao2020network, yao2020deep} used DNNs with stochastic optimizers as the backbone of their algorithms, however, they restricted the networks to some specific architectures while our method is flexible on the learning models, and they also lacked a theoretical understanding of their methods.

\subsection{Complementary-Label Learning}
\emph{Complementary-label learning} (CLL) \cite{ishida2017learning,yu2018learning,ishida2019complementary} uses the supervised information specifying a class that an example does \emph{not} belong to, and hence it can be considered as an extreme PLL case with $c-1$ candidate labels. Restricting each instance to associate with a single complementary label limits its potential because the labelers can easily annotate multiple candidate labels or complementary labels. Recently, \citet{feng2020learning} proposed to learn with multiple complementary labels. Nonetheless, they proposed an explicit data generation process, while our work has no special requirement on this aspect. Empirical studies also show that our proposed method can achieve promising performance in different partial types (e.g. binomial or pair in Section~\ref{sec:exp}).

\section{Learning with Partial Labels}

In this section, we define a classifier-consistent risk estimator, and theoretically establish an estimation error bound.

\subsection{Classifier-Consistent Risk Estimator}

To make Eq.~(\ref{eq:origin-pll-risk}) estimable, an intuitive way is through a surrogate loss that treats all the candidate labels equally \cite{jin2003learning}: $\ell_{\rm PLL}(\bm g(X),S)=1/|S|\sum_{i\in S}\ell(\bm g(X),\bm e^i)$. Nonethelss, the true label may be overwhelmed by the distractive outputs of the multiple false positive labels.  
Therefore, we consider that only the true label contributes to retrieving the classifier. Capturing this idea we define the PLL loss as the \emph{minimal loss} over the candidate label set:
\begin{equation}\label{eq:loss-pll}
\ell_{\rm PLL}(\bm g(X),S) = \min_{i\in S}\ell(\bm g(X), \bm{e}^i),
\end{equation}
which immediately leads to a new risk estimator:
\begin{equation}\label{eq:risk-pll}
\R_{\rm PLL}(\bm g) = \EE_{(X,S)\sim p(x,s)} \min_{i\in S}\ell(\bm g(X), \bm{e}^i).
\end{equation}
For a detailed description of Eq.~(\ref{eq:loss-pll}), we can disassemble it by defining a function in terms of $Y$ as a classifier $\bm g$'s best guess at the true label:
\begin{equation}\label{eq:risk-pll-reverse}
Y^{\bm g} = \argmin_{i\in S}\ell(\bm g(X), \bm e^i).
\end{equation} 
Then correspondingly the PLL loss can be rewritten into an equivalent expression, i.e., $\bm g$ only needs to fit the best guess $Y^{\bm g}$ among $S$: 
\begin{equation*}
\ell_{\rm PLL}(\bm g(X),S) = \min_{i\in S}\ell(\bm g(X), \bm e^i) = \ell(\bm g(X), \bm e^{Y^{\bm g}}).
\end{equation*}

From now on, we explain under which conditions Eq.~(\ref{eq:risk-pll}) is classifier-consistent. We start investigating this question from the following two lemmas.

\begin{lemma}(Summarized from Sec.3.1 of \citealp{liu2014learnability})\label{lemma:learnable}
	The ambiguity degree is defined as
	\begin{equation*}
	\gamma = sup_{(X,Y) \sim p(x,y), \bar{Y}\in\Y, S \sim p(s|x,y), \bar{Y}\neq Y}{\rm Pr}(\bar{Y} \in S).
	\end{equation*}
	If $\gamma<1$, i.e. under the small ambiguity degree condition, the PLL problem is ERM learnability.
\end{lemma}
$\gamma$ is the maximum probability of a negative label $\bar{Y}$ co-occurs with the true label $Y$. The small ambiguity degree condition implies that except for the true label, no other labels will be one hundred percent included in the candidate label set, which guarantees a classification error made on any instance will be detected with probability at least $1-\gamma$.

Then we provide a condition of the identifiability of the ordinary optimal classifier.
\begin{lemma}(\citealp{yu2018learning})\label{lemma:identify}
	If $\ell$ is the cross-entropy loss or mean squared error loss, the optimal classifier $\bm g^*$ of Eq.~(\ref{eq:risk-ordinary}) satisfies $g_i^*(X) = p(Y=i|X)$.
\end{lemma}

With the above lemmas, we can prove Eq.~(\ref{eq:risk-pll}) possess classifier-consistency under a reasonable assumption, that is, learning is conducted under the deterministic case.
\begin{theorem}\label{thm:consistent}
	Under the deterministic scenario, if the small ambiguity degree condition is satisfied, and cross-entropy or mean squared error loss is used, then, the PLL optimal classifier $\bm g_{\rm PLL}^*$ of Eq.~(\ref{eq:risk-pll}) is equivalent to the ordinary optimal classifier $\bm g^*$ of Eq.~(\ref{eq:risk-ordinary}), i.e., $\bm g^*_{\rm PLL}=\bm g^*$.
\end{theorem}
It is natural to assume a deterministic learning scenario, i.e., the true label $Y_X$ of an instance $X$ is uniquely determined by the measurable function $\bm g^{**}: p(Y=i|X)=g^{**}_i(X), Y_X=\argmax_{i\in\Y} g^{**}_i(X)$, for ensuring the basic definition for PLL, i.e., Eq.~(\ref{eq:keydef}) always holds. In the light of Lemma~\ref{lemma:identify}, $\bm g^*=\bm g^{**}$. In this case, the optimal classifier $\bm g^*$ can be recovered by minimizing Eq.~(\ref{eq:risk-pll}). The proof is provided in Appendix.

\subsection{Estimation Error Bound}
Let $\widehat{\R}_{\rm PLL}$ be the empirical counterpart of $\R_{\rm PLL}$, and $\widehat{\bm g}_{\rm PLL}=\argmin_{\bm g\in\G}\widehat{\R}_{\rm PLL}(\bm g)$ be the empirical risk classifier. 
Suppose $\G_y$ be a class of real functions, and $\G=\oplus_{y\in[c]}\G_y$ be a $c$-valued function class. Assume there is $C_{\bm g}>0$ such that $\sup_{{\bm g}\in\G}||{\bm g}||_\infty\leq C_{\bm g}$, and the loss function $\ell(\bm g(X),Y)$ is Lipschitz continuous for all $|\bm g|\leq C_{\bm g}$ with a Lipschitz constant $L_\ell$ and upper-bounded by $M$, i.e., $M=\sup_{X\in\X,|\bm g|\leq C_{\bm g},Y\in\Y}\ell(\bm g(X),Y)$. 
The \emph{Rademacher complexity} of $\G$ over $p(x)$ with sample size $n$ is defined as $\Ra_n(\G)$ \cite{bartlett2002rademacher}. Then we have the following estimation error bound.
\begin{theorem}\label{thm:errorbound}
	For any $\delta>0$, we have with probability at least $1-\delta$,
	\begin{align}\label{eq:errorbound}
	\R_{\rm PLL}(\widehat{\bm g}_{\rm PLL}) -&\R_{\rm PLL}(\bm g_{\rm PLL}^*) \leq \\[-1mm]\nonumber
	&4\sqrt{2}cL_{\ell}\sum_{y=1}^c \Ra_n(\G_y) + 2M\sqrt{\frac{\log (2/\delta)}{2n}}.
	\end{align}
\end{theorem}
The proof is given in Appendix. As $n\rightarrow\infty$, $\Ra_n(\G_y)\rightarrow 0$ for all parametric models with a bounded norm. Hence, $\R_{\rm PLL}(\widehat{\bm g}_{\rm PLL}) \rightarrow \R_{\rm PLL}(\bm g_{\rm PLL}^*)$ as the number of training data approaches infinity.

\section{Benchmark Solution}\label{sec:method}

\begin{algorithm}[t]
	\caption{PRODEN Algorithm}
	\setstretch{1.1}
	\begin{algorithmic}[1]\label{algorithm}
		\INPUT $\D$: \quad the partial-label training set $\lbrace (x_i, s_i) \rbrace_{i=1}^n$,\\ 
		\quad $T$: \quad number of epochs;\\
		\OUTPUT $\Theta$:\quad model parameter for $\bm g(x;\Theta)$	
		\STATE Let $\A$ be an stochastic optimization algorithm;
		\STATE Initialize uniform weights $\bm{{\rm w}}^0$ according to Eq.~(\ref{eq:init});
		\FOR{$t=1$ to $T$}
		\STATE Shuffle training set $\D$ into $B$ mini-batches;
		\FOR{$k=1$ to $B$}
		\STATE Compute $L$  according to Eq.~(\ref{eq:softmin});
		\STATE Set gradient $-\nabla_{\Theta}L$;
		\STATE Update $\bm{{\rm w}}$ according to Eq.~(\ref{eq:norm});
		\STATE Update $\Theta$ by $\A$;
		\ENDFOR
		\ENDFOR
	\end{algorithmic}
\end{algorithm}

In the previous section, based on a novel PLL loss Eq.~(\ref{eq:loss-pll}) we provided a risk estimator Eq.~(\ref{eq:risk-pll}) with intriguing theoretical properties. Although the estimator is without any assumption on the model used, the $\min$ operator makes optimization difficult, because if a wrong label $i$ in Eq.~(\ref{eq:risk-pll}) is selected in the beginning, the optimization will focus on the wrong label till the end. Thus we design a benchmark solution to approximately solve Eq.~(\ref{eq:risk-pll}) according to the idea behind Eq.~(\ref{eq:loss-pll}) that only one (true) label should be taken into account, and consequently we can easily implement the algorithm over arbitrary multi-class classifier and powerful stochastic optimization.

Specifically, we relax the $\min$ operator in Eq.~(\ref{eq:loss-pll}) by the dynamic weights. In the first place, we require that $\ell$ can be decomposed into each label, i.e.,
\begin{equation*}
\ell(\bm g(X), \bm{e}^Y)=\sum_{i=1}^c\ell(g_i(X), e_i^Y),
\end{equation*}
where $e_i^Y$ is the $i$-th element of $\bm e^Y$. Many commonly used multi-class loss functions are decomposable, for example, $\ell$ can be the cross-entropy loss: $\ell_{\rm CE}(\bm g(X), \bm{e}^Y)=-\sum_{i=1}^c e_i^Y\log(g_i(X))$ or mean squared error loss: $\ell_{\rm MSE}(\bm g(X), \bm{e}^Y)=\sum_{i=1}^c(g_i(X)-e_i^Y)^2$. In this case, the empirical risk estimator is rewritted as follows:
\begin{equation}\label{eq:softmin}
\widehat\R_{\rm PLL} = \frac{1}{n}\sum_{i=1}^n\sum_{j=1}^c {\rm w}_{ij}\ell(g_j(x_i), e_j^{s_i}),
\end{equation} 
where $e^{s_i}_j$ is the $j$-th coordinate of $\bm e^{s_i}$ and $\bm e^{s_i}=\sum_{k\in s_i} \bm e^k$. ${\bf w}_{i}\in\Delta^{c-1}$ is the $c$-dimensional simplex and the $j$-th element of it, i.e., ${\rm w}_{ij}$, is the corresponding label weight, i.e., the confidence of the $j$-th label being the true label of the $i$-th instance.
With appropriate weights, i.e., ${\rm w}_{ij}=1$ if $\argmin_{k\in s_i}\ell(\bm g(x_i),\bm e^k)=j$ and 0 otherwise, which is in consonance with Eq.~(\ref{eq:risk-pll-reverse}), Eq.~(\ref{eq:softmin}) is equivalent to the empirical counterpart of Eq.~(\ref{eq:risk-pll}).
Ideally, the label with weight 1 is exactly the true label, i.e, $j$ is the true label of $x_i$ if ${\rm w}_{ij}=1$, which means we have identified the true label successfully. To eventually achieve such an ideal situation, we propose an effective algorithm named \emph{PRODEN} (\textit{PRO}gressive i\textit{DEN}tification).

Because the weights are latent, the minimizer of Eq.~(\ref{eq:softmin}) cannot be solved directly. Inspiring by the EM algorithm, we start by training an initial model with uniform weights:
\begin{align}\label{eq:init}
{\rm w}_{ij} =\left\{ 
\begin{aligned}
&{1}/{|s_i|}&\quad\quad\quad\text{ if } j\in s_i,\\
&\quad 0 &\quad\quad\ \text{ otherwise}.
\end{aligned}\right.
\end{align}
Thanks to the memorization effects \cite{arpit2017a, han2018co}, the networks will remember ``frequent patterns'' in the first few iterations. If the small ambiguity degree condition is satisfied, they tend to remember the true labels in the initial epochs, which guides us towards a discriminative classifier giving relatively low losses for more possible true labels. Next, the problem becomes how to retain the informative predictions and gradually filter out the false positive labels. Instead of estimating the weights in a separate E-step, we tackle it simply using the current predictions for slightly putting more weights on more possible labels:
\begin{align}\label{eq:norm}
{\rm w}_{ij}=\left\{
\begin{aligned}
&g_j(\bm{x}_i) \Big/ \sum\nolimits_{k\in s_i} g_k(\bm{x}_i) & \text{if }~~j \in s_i,\\
&\quad\quad\quad\quad\quad 0\quad\quad\quad\quad\quad\  & \text{otherwise}.
\end{aligned}
\right.
\end{align} 
In this fashion, the true labels are \emph{progressively} identified, and the refined labels in turn help to improve the classifier. The overall procedure is summarized in Algorithm~\ref{algorithm}, whose identification step (Step 8) and optimization step (Step 9) are conducted in a \emph{seamless} manner. Next we provide more detailed arguments about the superiority and generalization of the proposed method.

Firstly, PRODEN gets rid of the overfitting issues of EM methods. The previous \emph{iterative} EM methods \cite{jin2003learning, tang2017confidence, feng2019partial} trained the model until convergence in the M-step, but overemphasizing the convergence may result in redundant computation and overfitting issues, as the model will eventually fit the initial inexact prior knowledge and make a less informative estimate in the E-step on which the subsequent learning is based. To mitigate these problems, our method advances the procedure by merging the E-step and M-step. Since the learning process has no clear separation of the E-step and M-step, the weights can be updated at any epoch such that the local convergence within each epoch is not necessary in our training process. 

Secondly, PRODEN has great flexibility for models and loss functions. In the deep learning era, loss functions with different theoretical results are one of the key factors that affect the performance of DNNs \cite{ghosh2017robust}. In this way, models are welcomed to be loss-independent that allows flexibility with any loss function \cite{ishida2019complementary}. However, existing PLL methods were restricted to some specific loss functions for optimizable, e.g., \citet{jin2003learning} limited the loss function to the KL divergence; the loss function of \citet{yao2020deep} was a combination of multiple cross-entropy losses. Instead, our proposal is flexible enough to be compatible with a large group of decomposable losses. Moreover, we will show in Appendix that the proposed method is provably a generalization of \citet{jin2003learning}.

\section{Experiments}\label{sec:exp}

\begin{figure*}[!t]
	\vskip 0.2in
	\subfigure{
		\begin{minipage}[b]{0.5\columnwidth}
			\centering
			\includegraphics[width=1.58in]{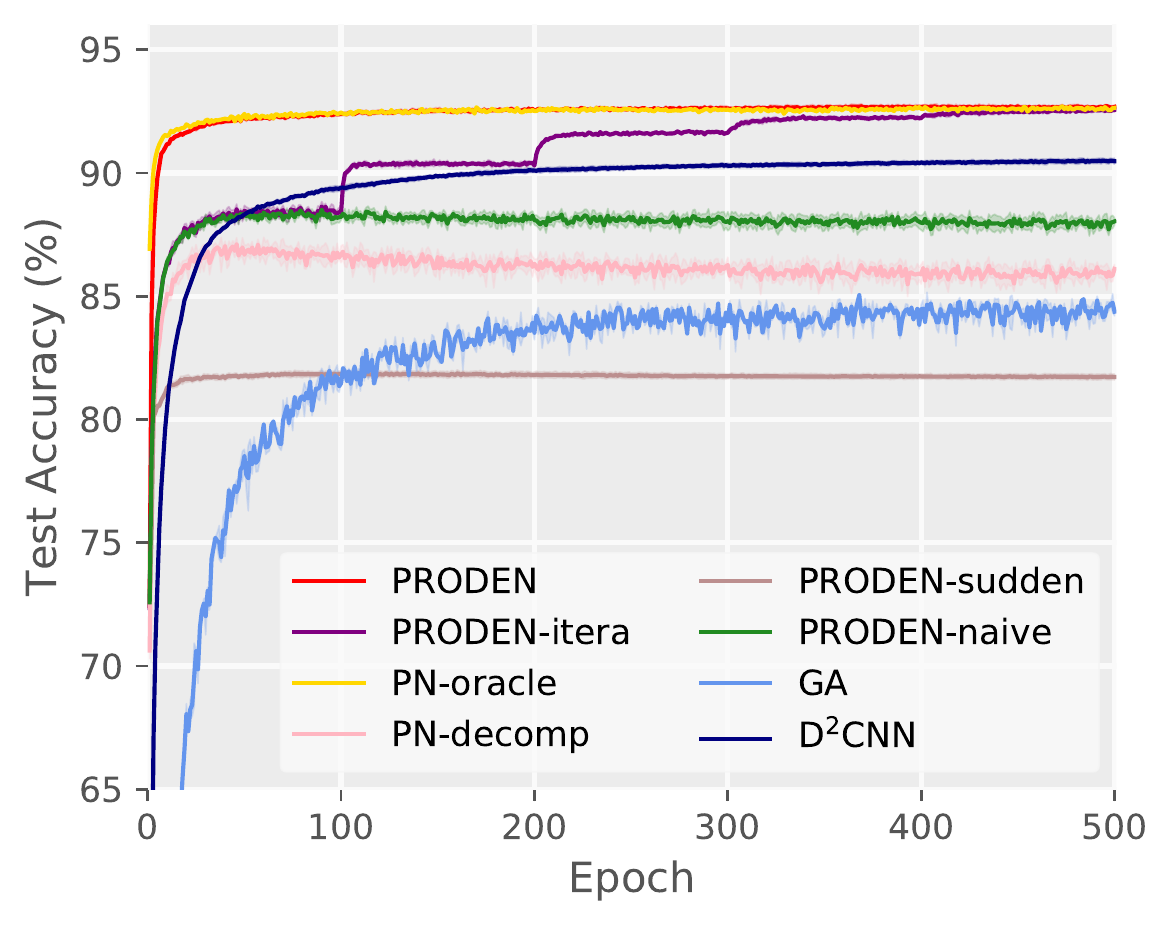}
			\centerline{MNIST, Linear, $q=0.1$}
	\end{minipage}}
	\subfigure{
		\begin{minipage}[b]{0.5\columnwidth}
			\centering
			\includegraphics[width=1.6in]{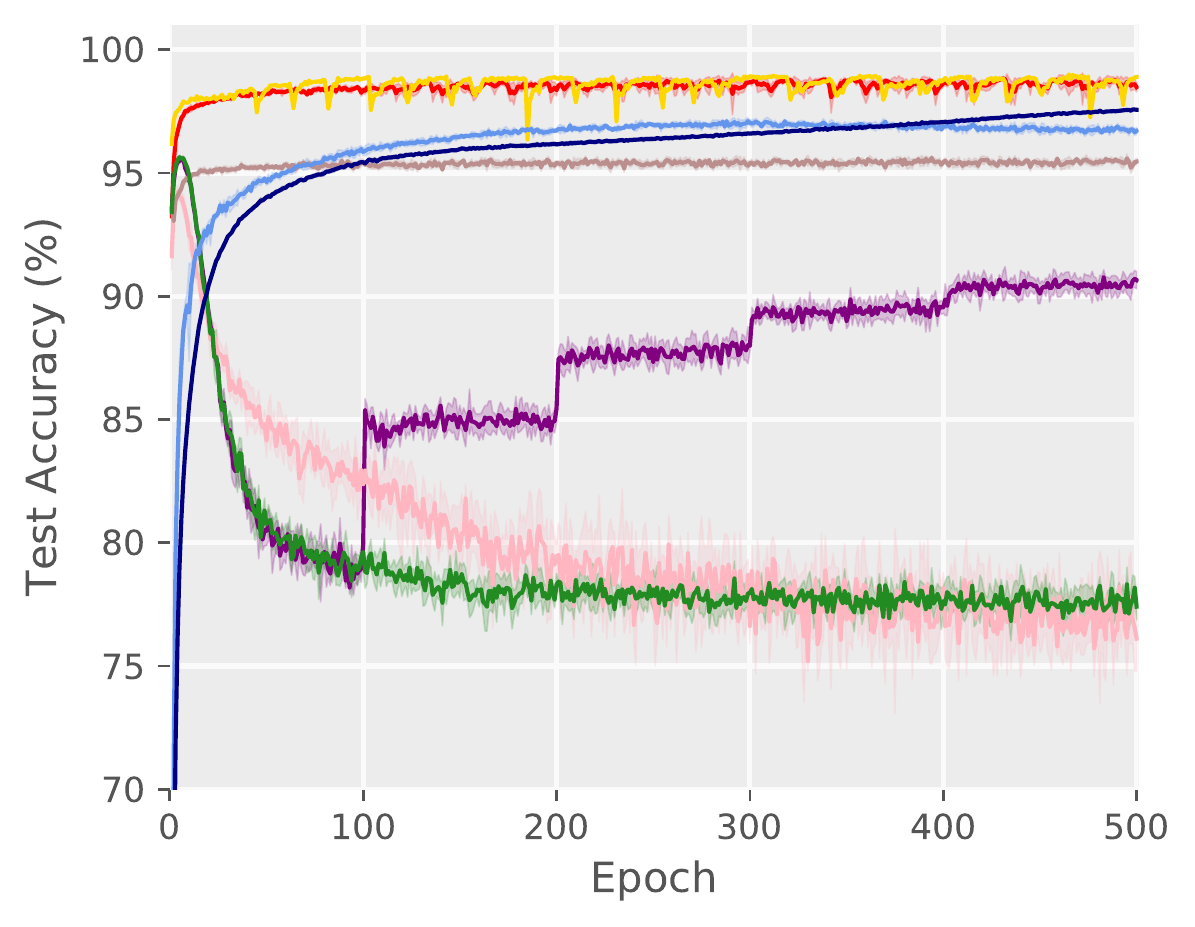}
			\centerline{\quad MNIST, MLP, $q=0.1$}
	\end{minipage}}
	\subfigure{
		\begin{minipage}[b]{0.5\columnwidth}
			\centering
			\includegraphics[width=1.6in]{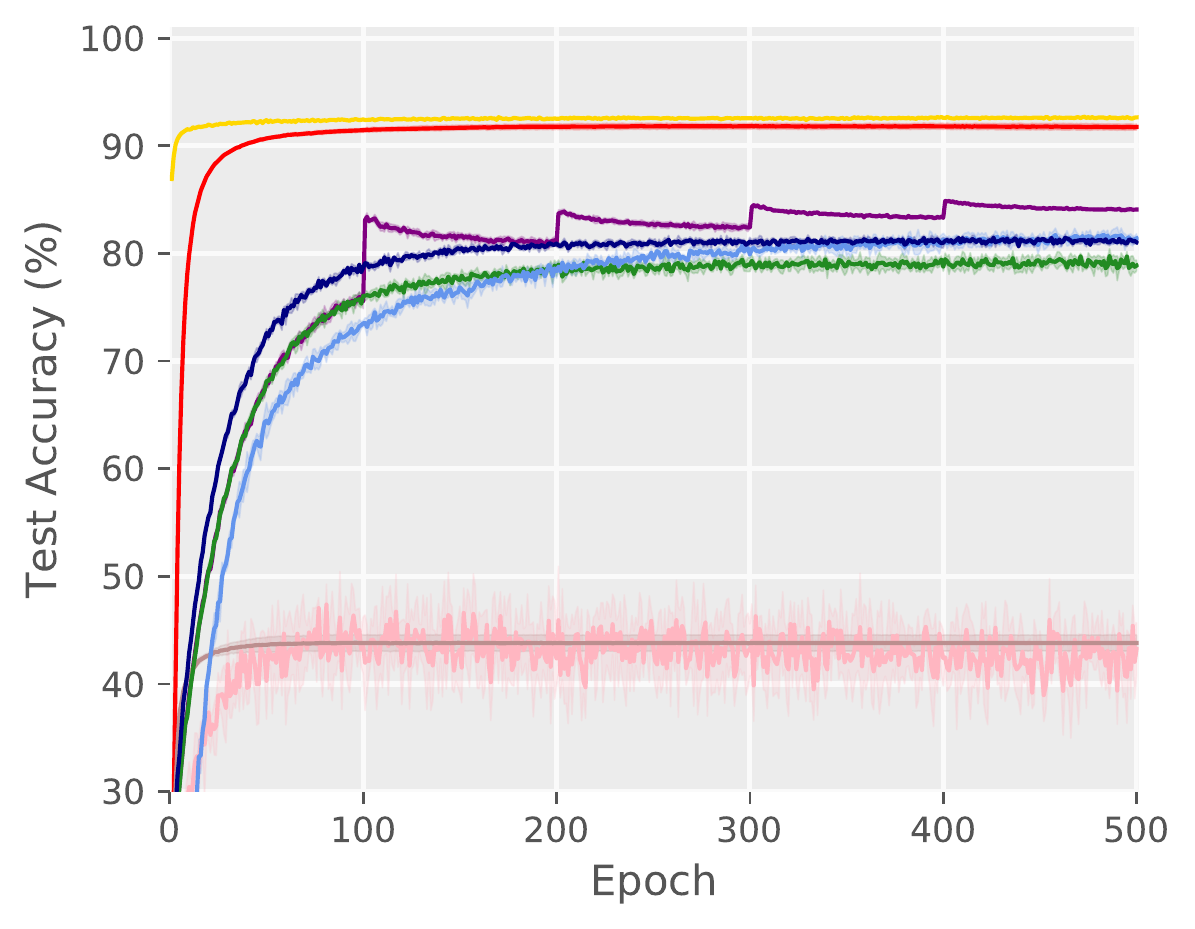}
			\centerline{\quad MNIST, Linear, $q=0.7$}
	\end{minipage}}
	\subfigure{
		\begin{minipage}[b]{0.5\columnwidth}
			\centering
			\includegraphics[width=1.6in]{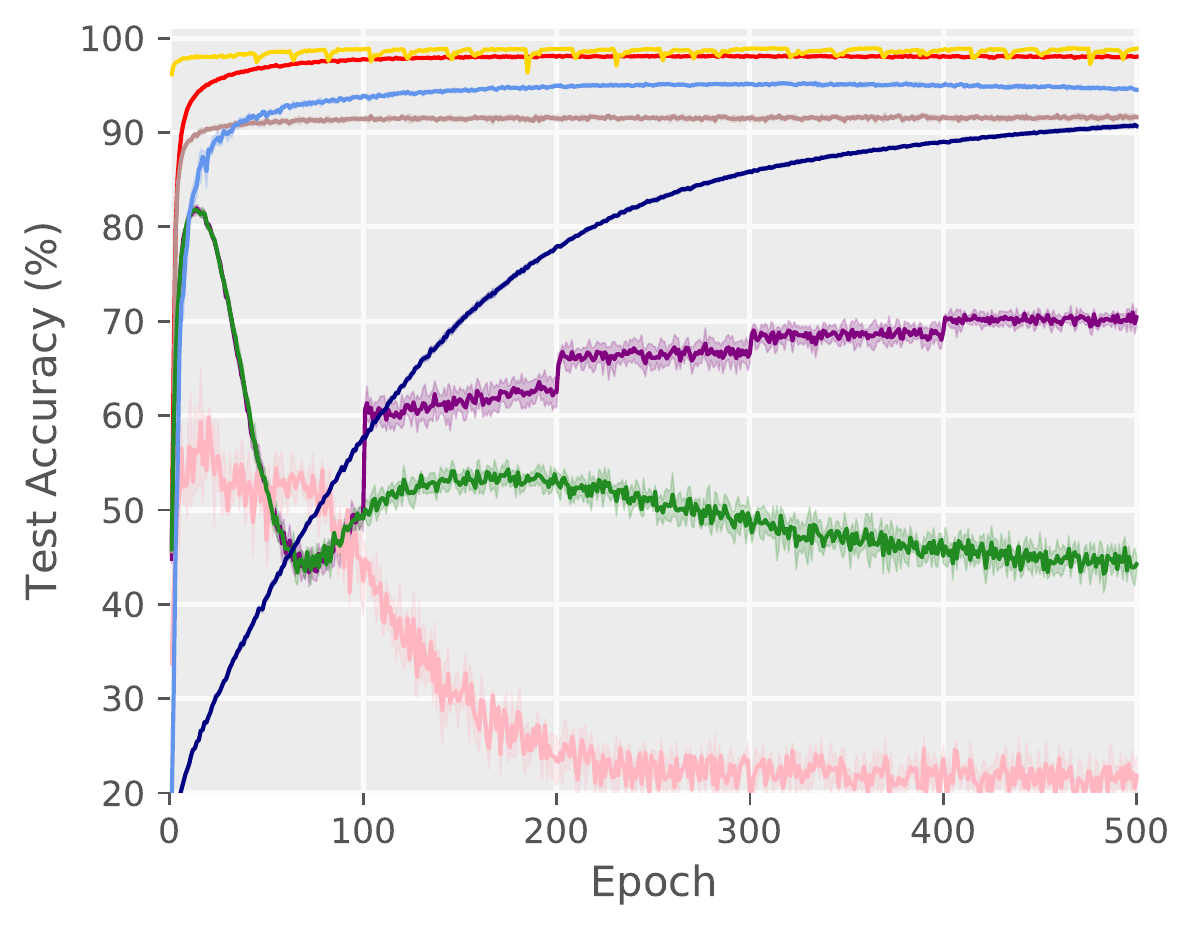}
			\centerline{\quad MNIST, MLP, $q=0.7$}
	\end{minipage}}
	
	\subfigure{
		\begin{minipage}[b]{0.5\columnwidth}
			\centering
			\includegraphics[width=1.6in]{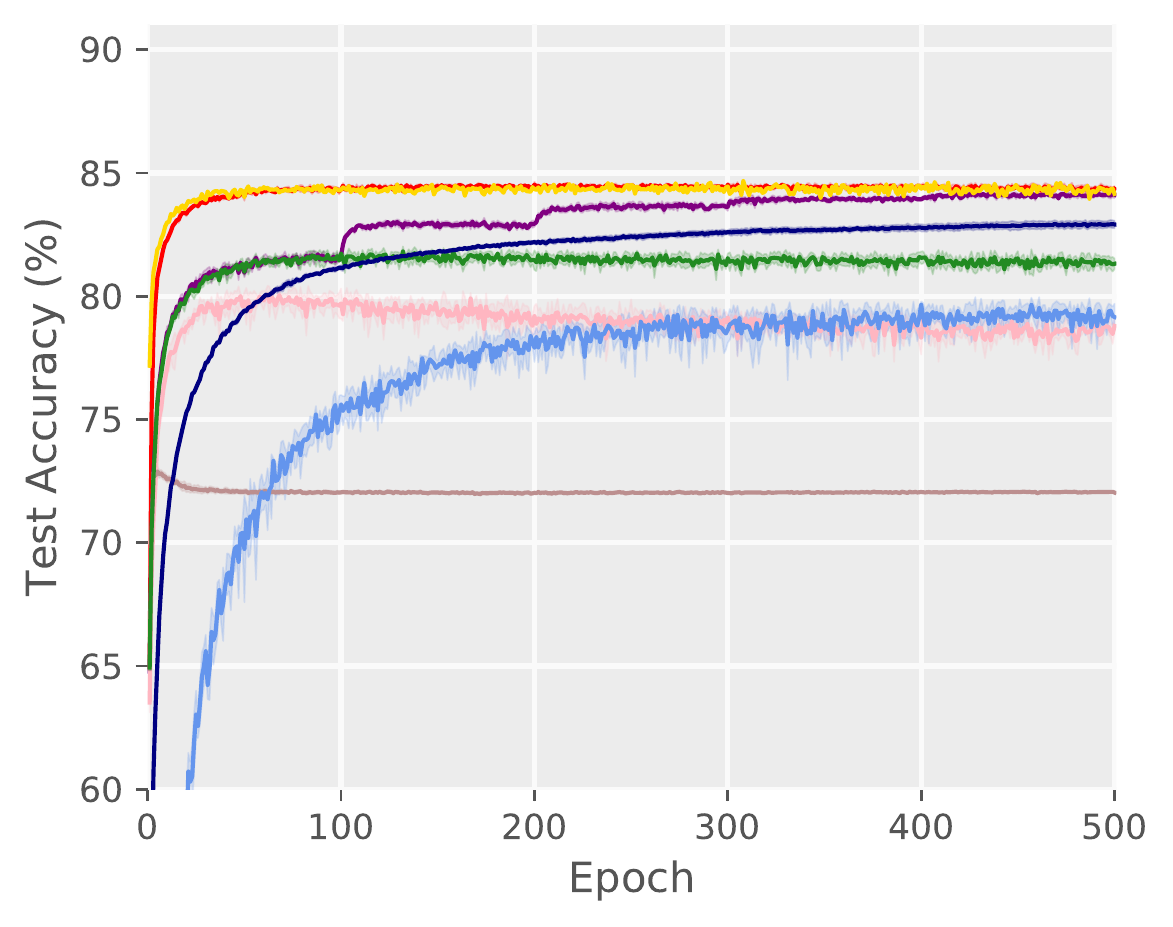}
			\centerline{\quad Fashion, Linear, $q=0.1$}
	\end{minipage}}
	\subfigure{
		\begin{minipage}[b]{0.5\columnwidth}
			\centering
			\includegraphics[width=1.6in]{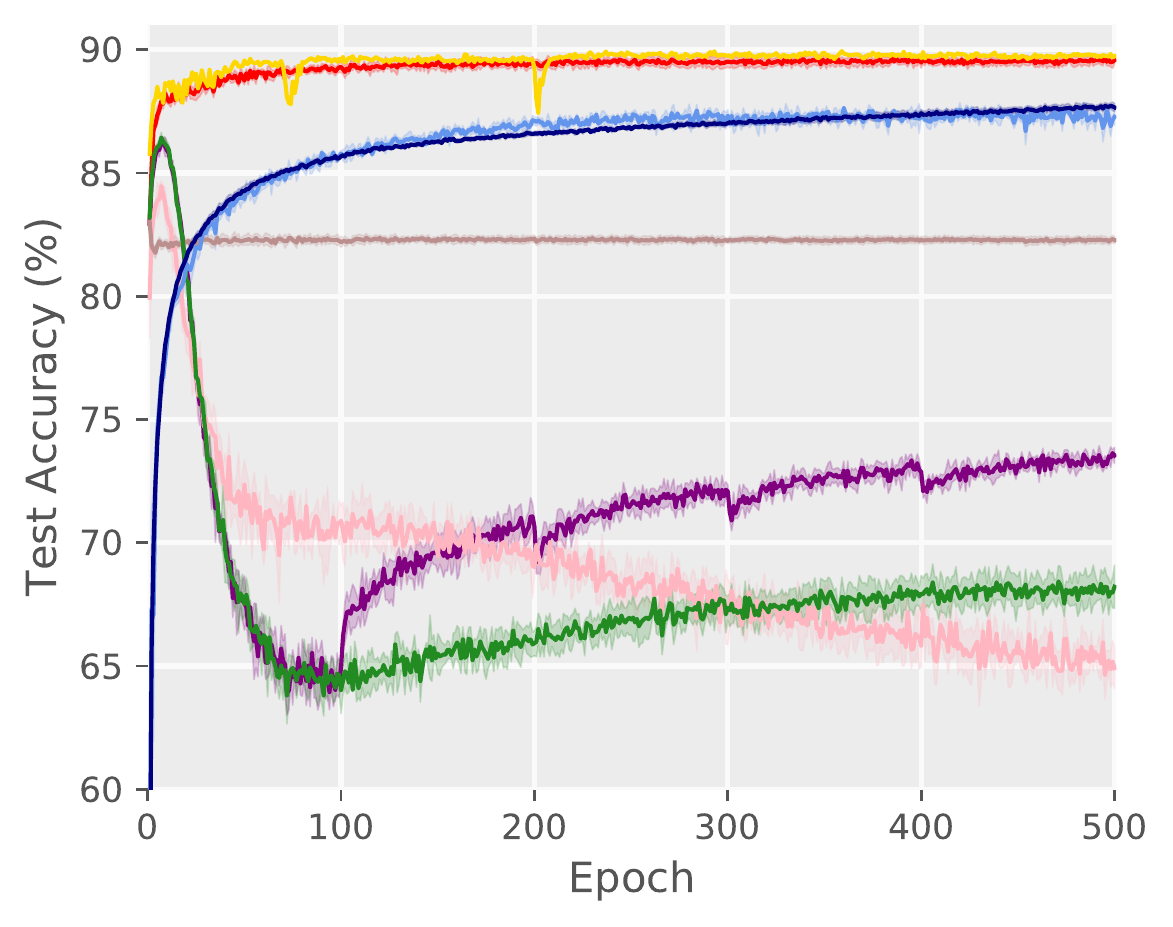}
			\centerline{\quad Fashion, MLP, $q=0.1$}
	\end{minipage}}
	\subfigure{
		\begin{minipage}[b]{0.5\columnwidth}
			\centering
			\includegraphics[width=1.6in]{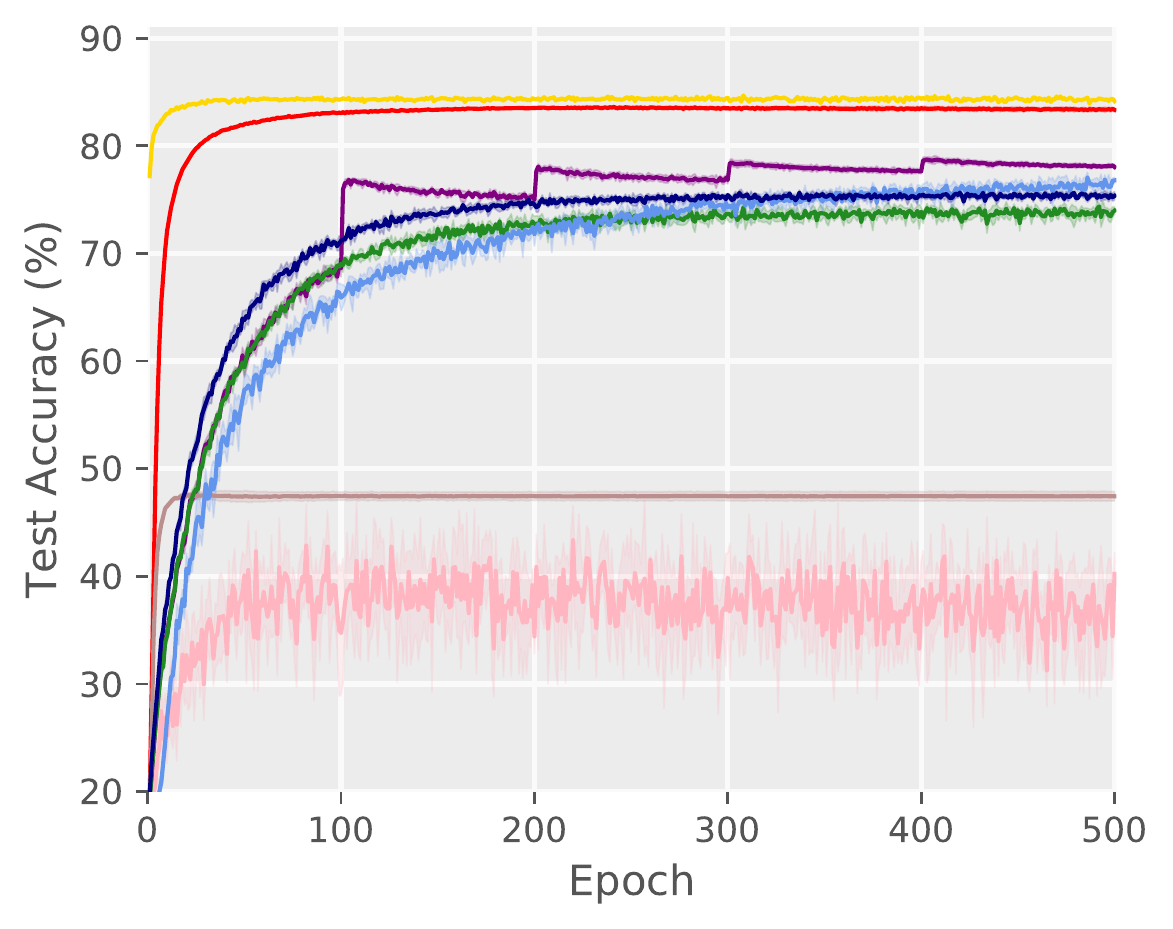}
			\centerline{\quad Fashion, Linear, $q=0.7$}
	\end{minipage}}
	\subfigure{
		\begin{minipage}[b]{0.5\columnwidth}
			\centering
			\includegraphics[width=1.6in]{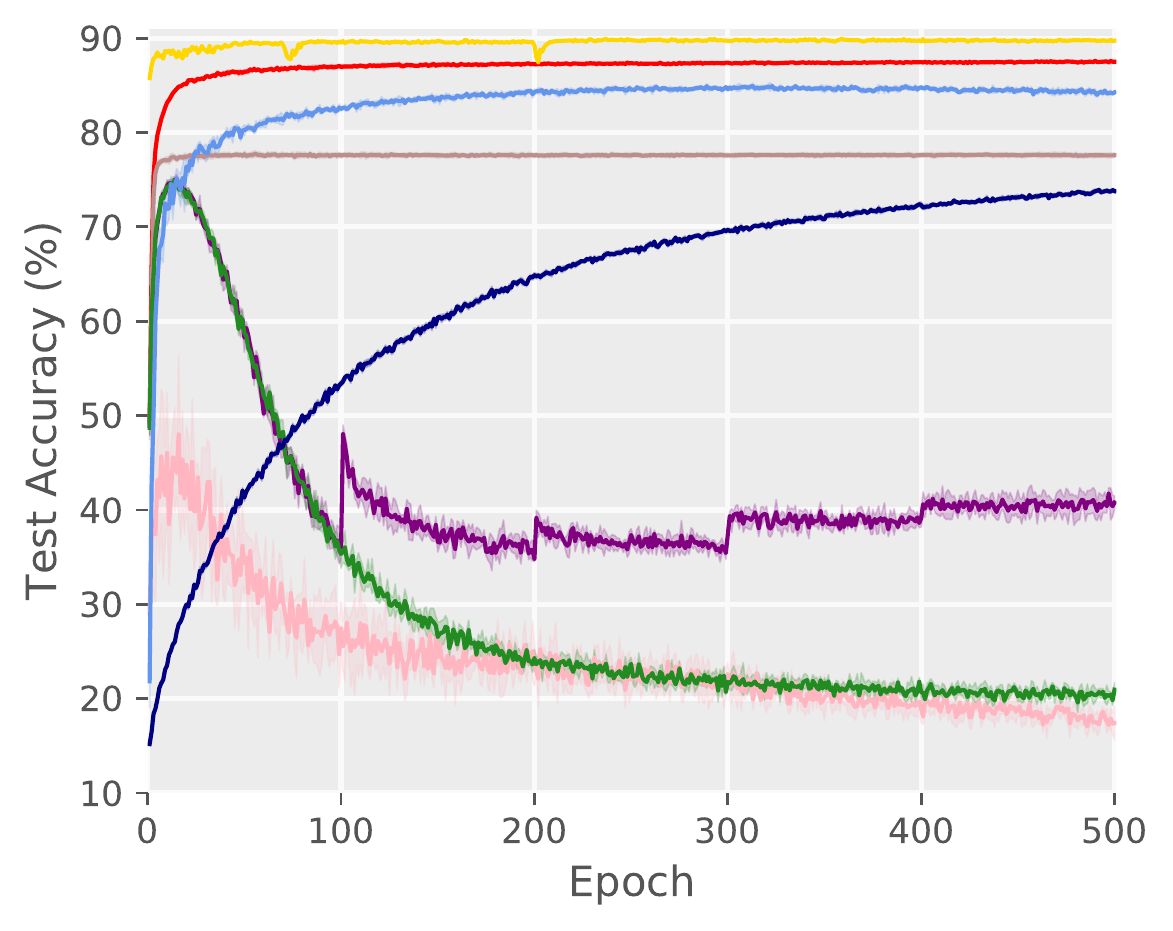}
			\centerline{\quad Fashion, MLP, $q=0.7$}
	\end{minipage}}
	
	\subfigure{
		\begin{minipage}[b]{0.5\columnwidth}
			\centering
			\includegraphics[width=1.65in]{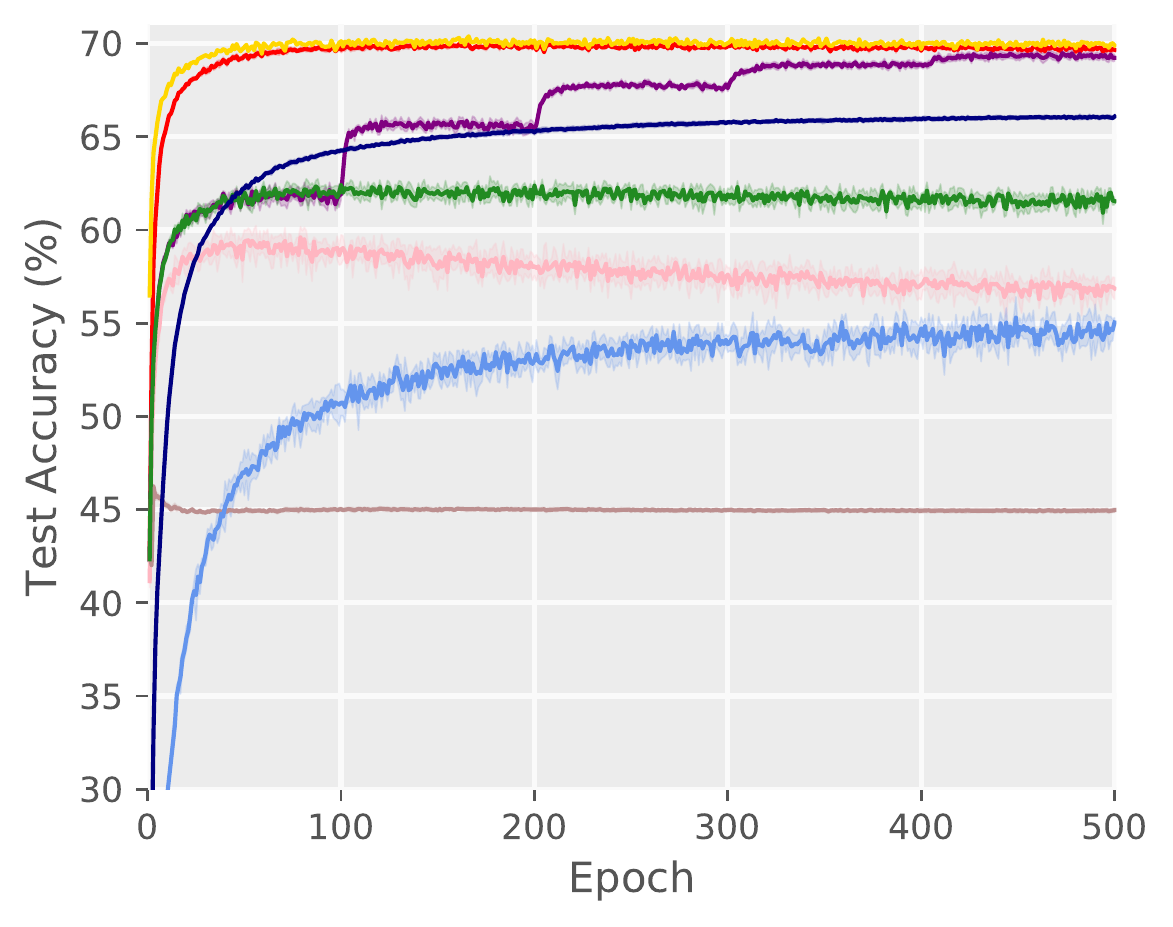}
			\centerline{Kuzushiji, Linear, $q=0.1$}
	\end{minipage}}%
	\subfigure{
		\begin{minipage}[b]{0.5\columnwidth}
			\centering
			\includegraphics[width=1.68in]{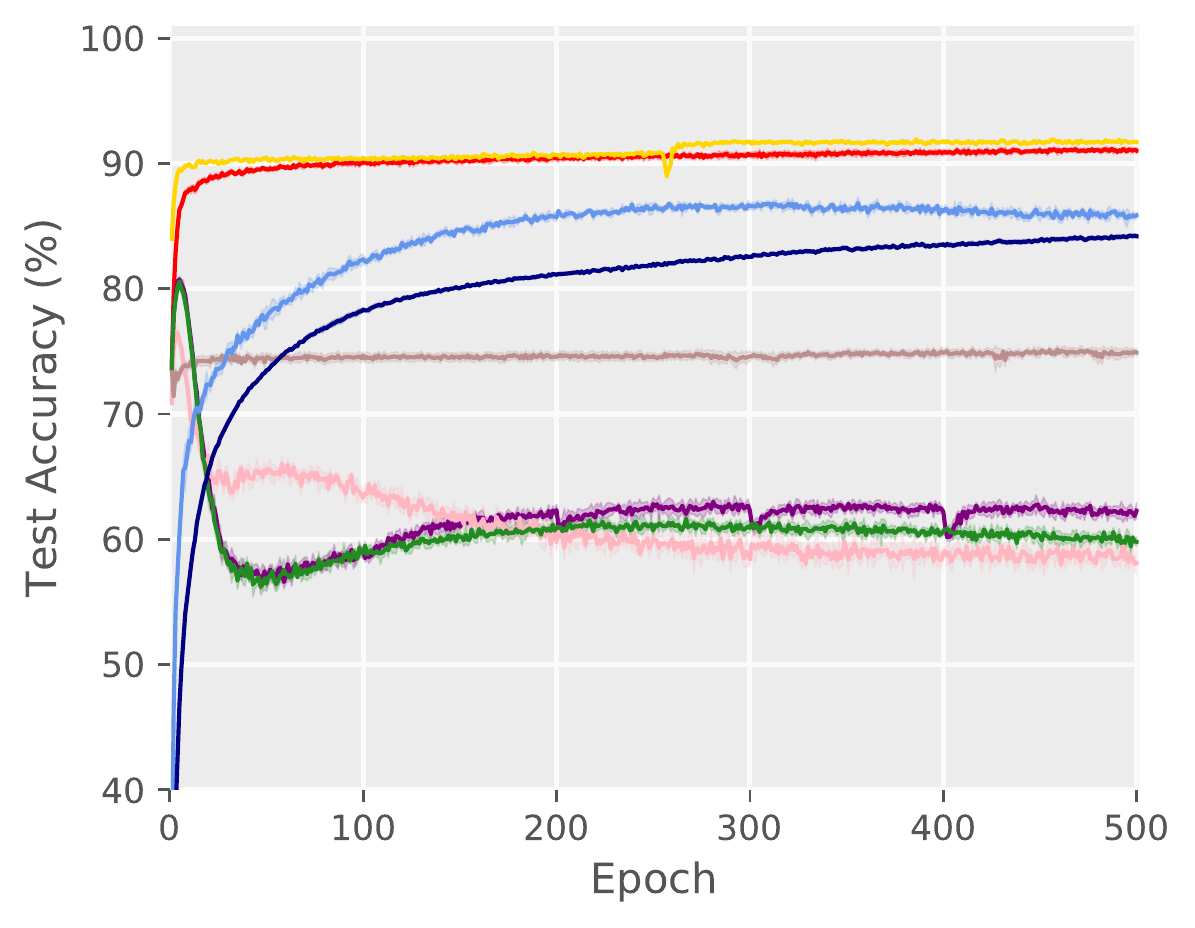}
			\centerline{\quad Kuzushiji, MLP, $q=0.1$}
	\end{minipage}}
	\subfigure{
		\begin{minipage}[b]{0.5\columnwidth}
			\centering
			\includegraphics[width=1.65in]{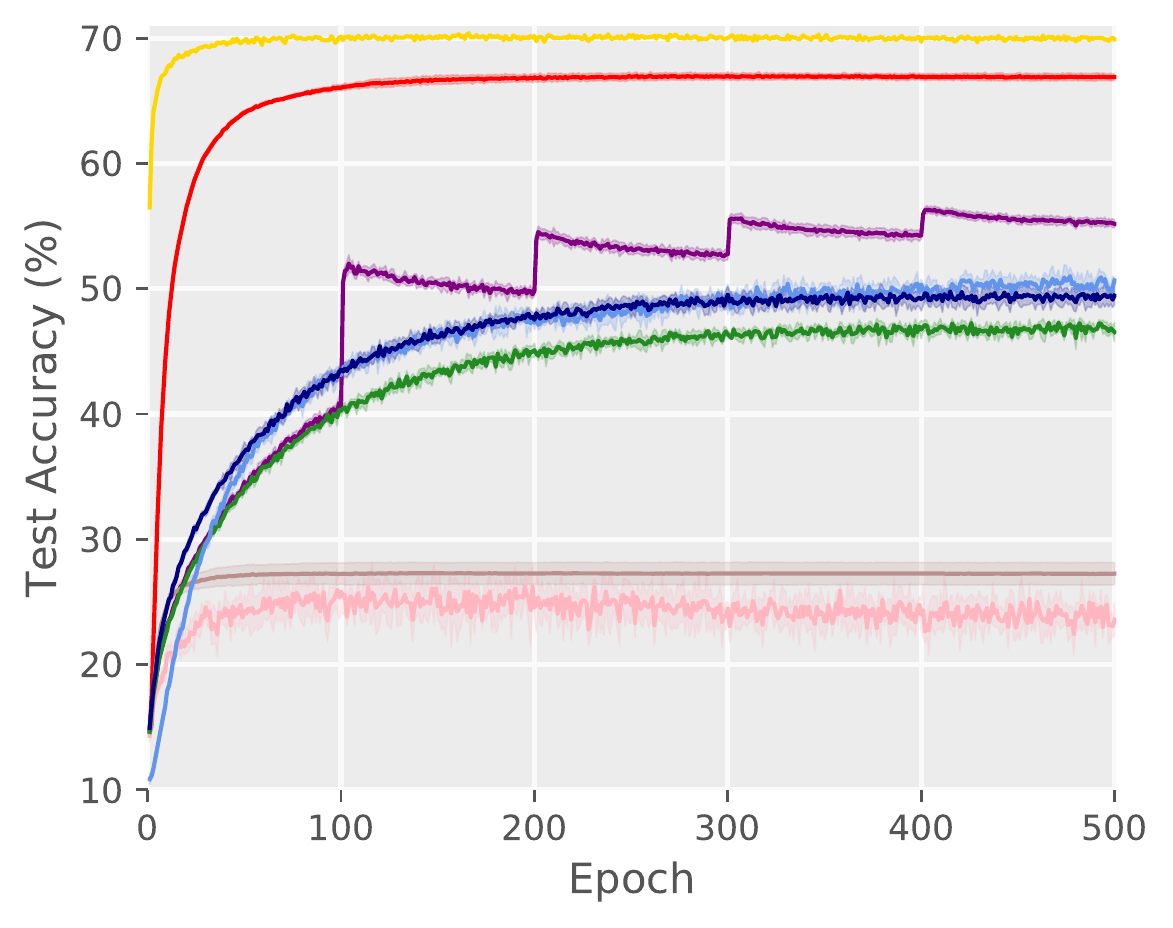}
			\centerline{\quad Kuzushiji, Linear, $q=0.7$}
	\end{minipage}}%
	\subfigure{
		\begin{minipage}[b]{0.5\columnwidth}
			\centering
			\includegraphics[width=1.68in]{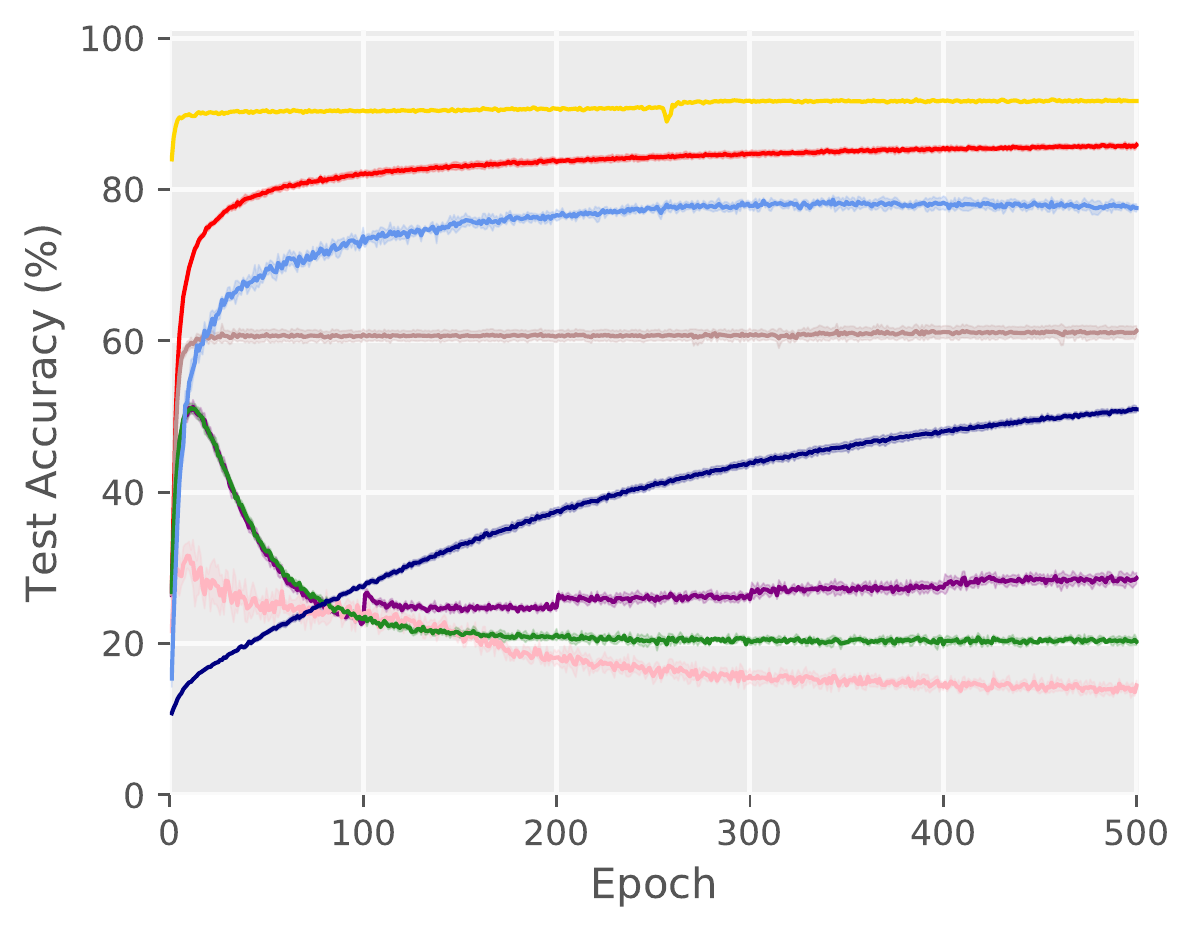}
			\centerline{\quad Kuzushiji, MLP, $q=0.7$}
	\end{minipage}}
	
	\subfigure{
		\begin{minipage}[b]{0.5\columnwidth}
			\centering
			\includegraphics[width=1.65in]{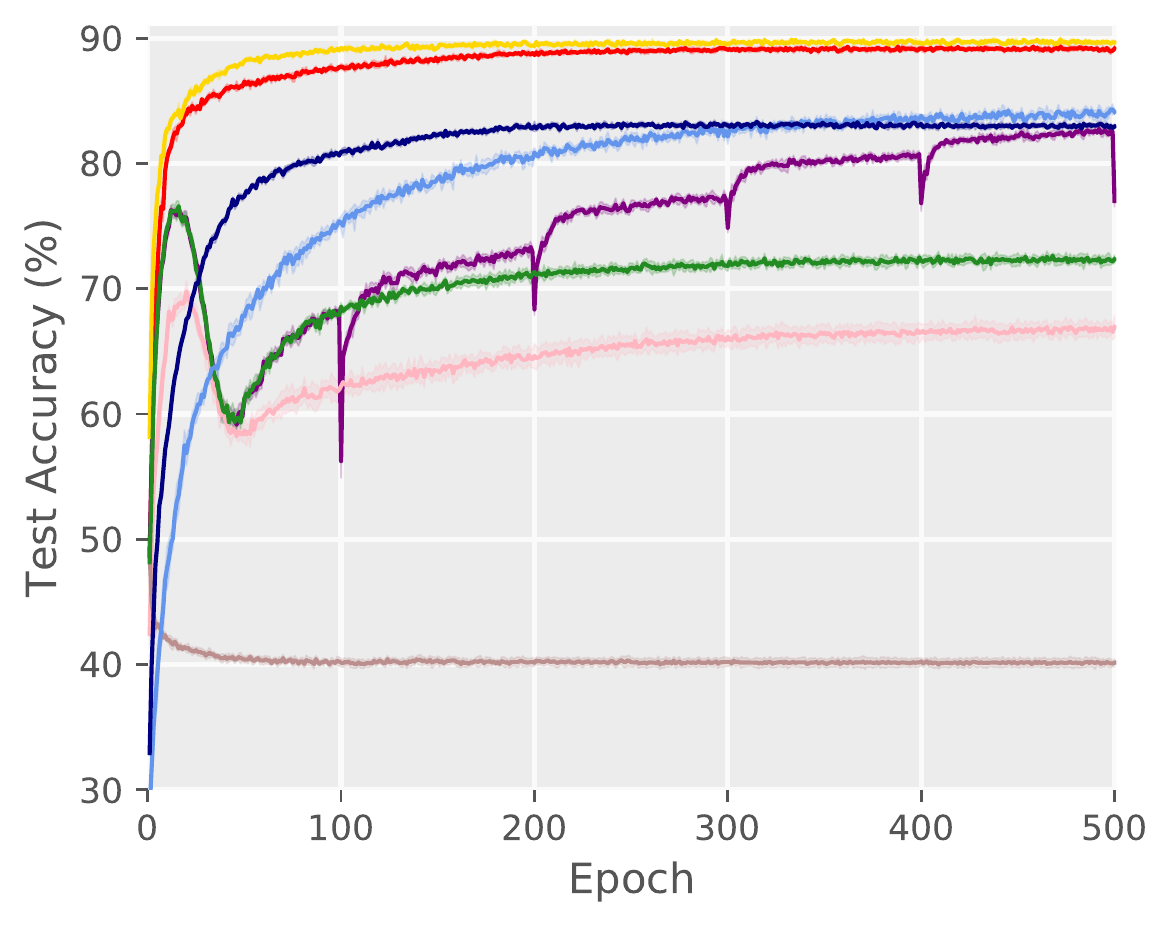}
			\centerline{CIFAR, ConvNet, $q=0.1$}
	\end{minipage}}%
	\subfigure{
		\begin{minipage}[b]{0.5\columnwidth}
			\centering
			\includegraphics[width=1.65in]{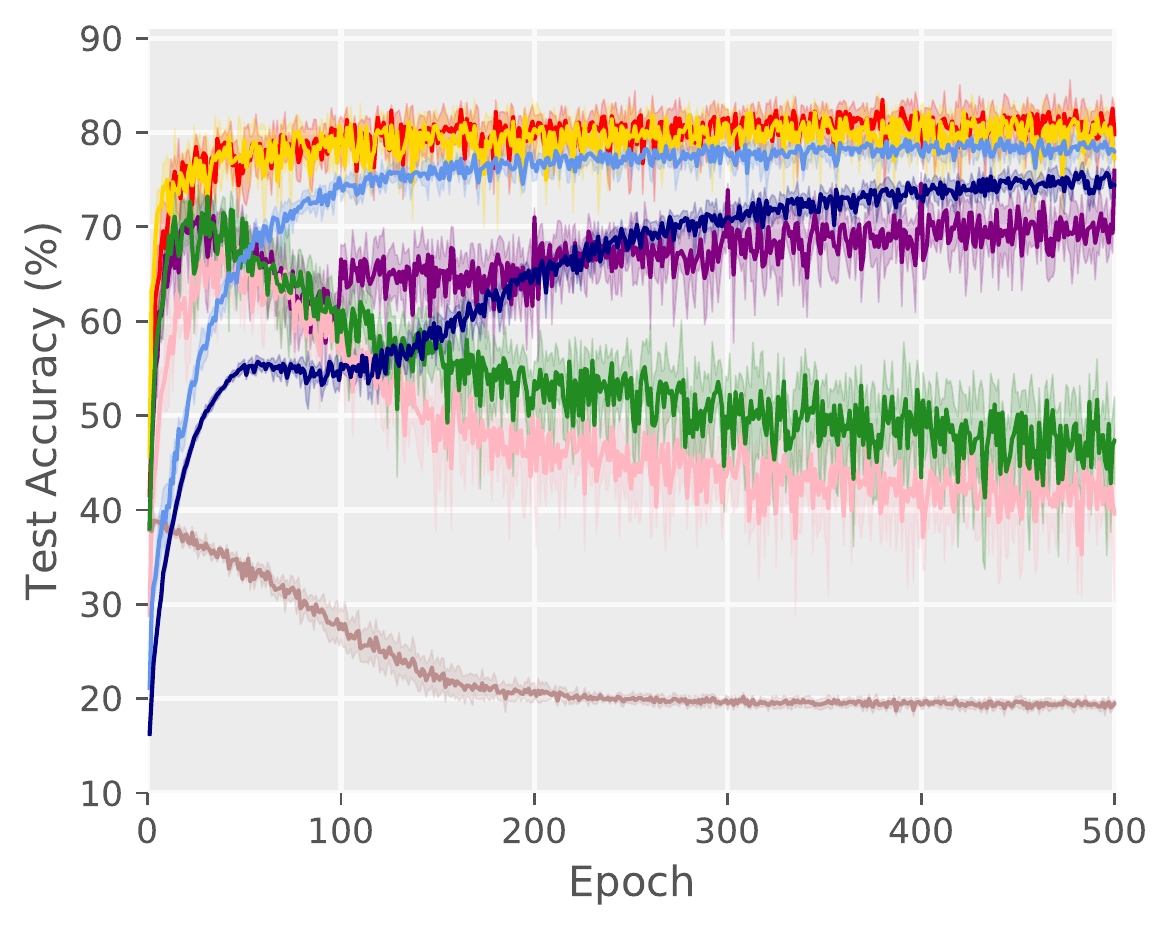}
			\centerline{\quad CIFAR, ResNet, $q=0.1$}
	\end{minipage}}
	\subfigure{
		\begin{minipage}[b]{0.5\columnwidth}
			\centering
			\includegraphics[width=1.65in]{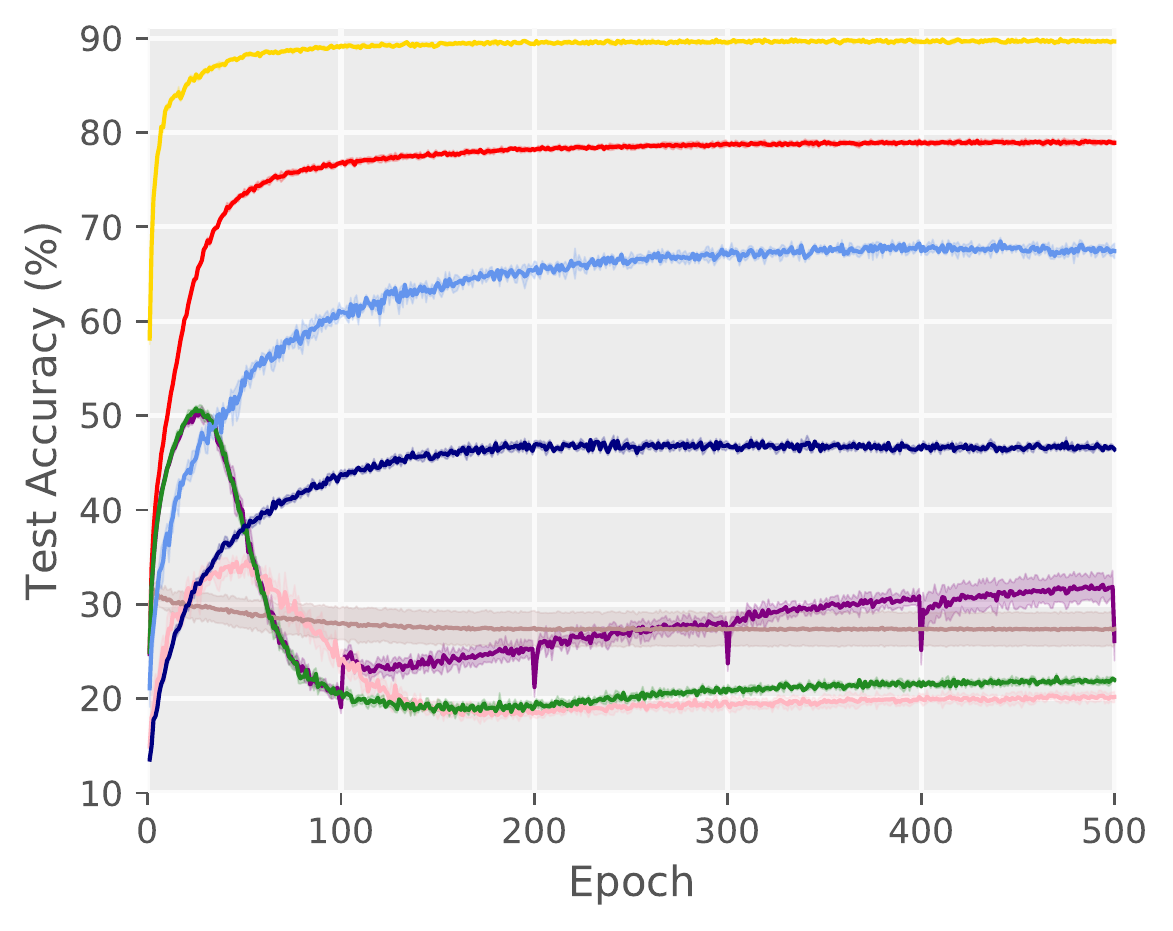}
			\centerline{\quad CIFAR, ConvNet, $q=0.7$}
	\end{minipage}}%
	\subfigure{
		\begin{minipage}[b]{0.5\columnwidth}
			\centering
			\includegraphics[width=1.68in]{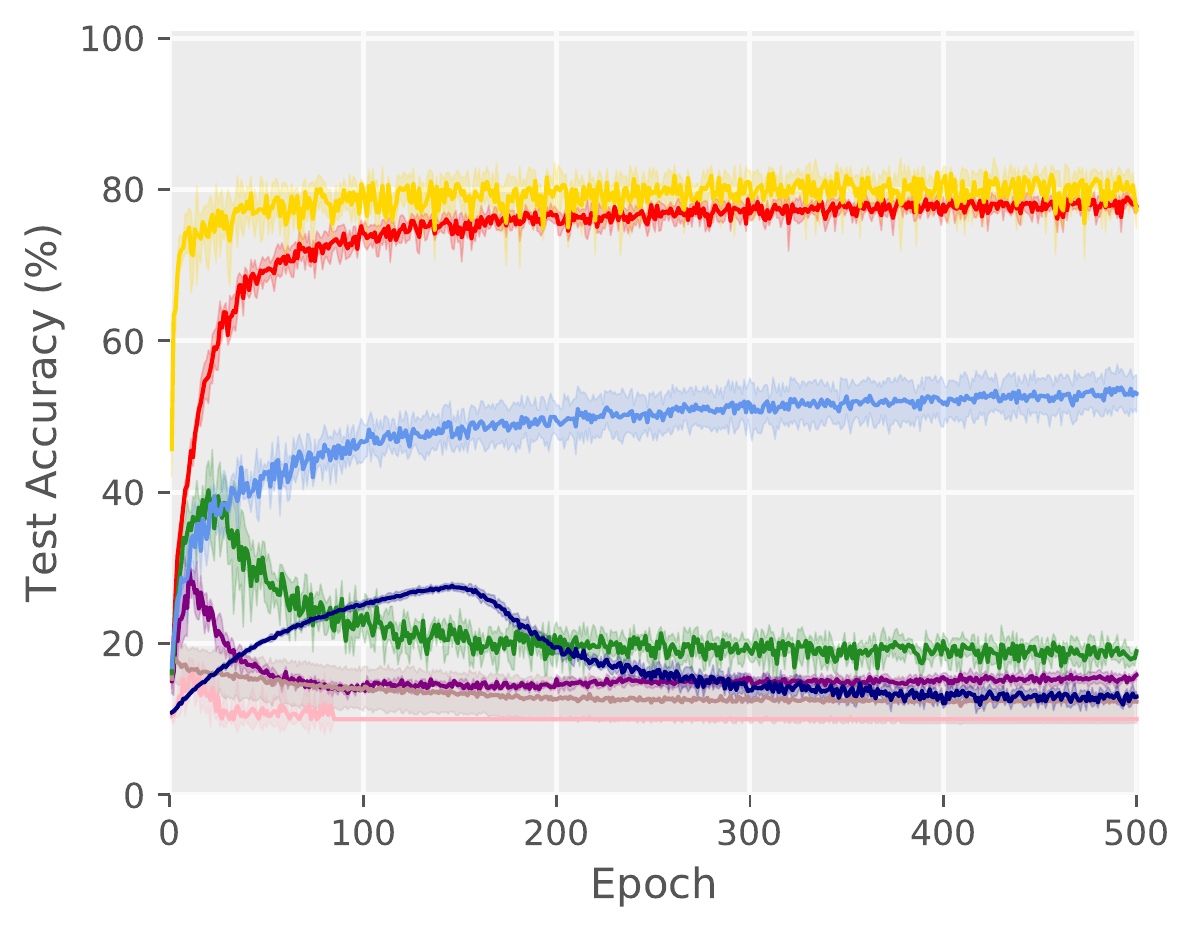}
			\centerline{\quad CIFAR, ResNet, $q=0.7$}
	\end{minipage}}
	\caption{Test accuracy for various models and datasets. Dark colors show the mean accuracy of 5 trials and light colors show standard deviation. Fashion is short for Fashion-MNIST, Kuzushiji is short of Kuzushiji-MNIST, CIFAR is short of CIFAR-10.}
	\label{fig:benchmark_binomial_test}
	\vskip -0.2in
\end{figure*}

\begin{figure*}[!t]
	\vskip 0.2in
	\subfigure{
		\begin{minipage}[b]{0.5\columnwidth}
			\includegraphics[width=1.6in]{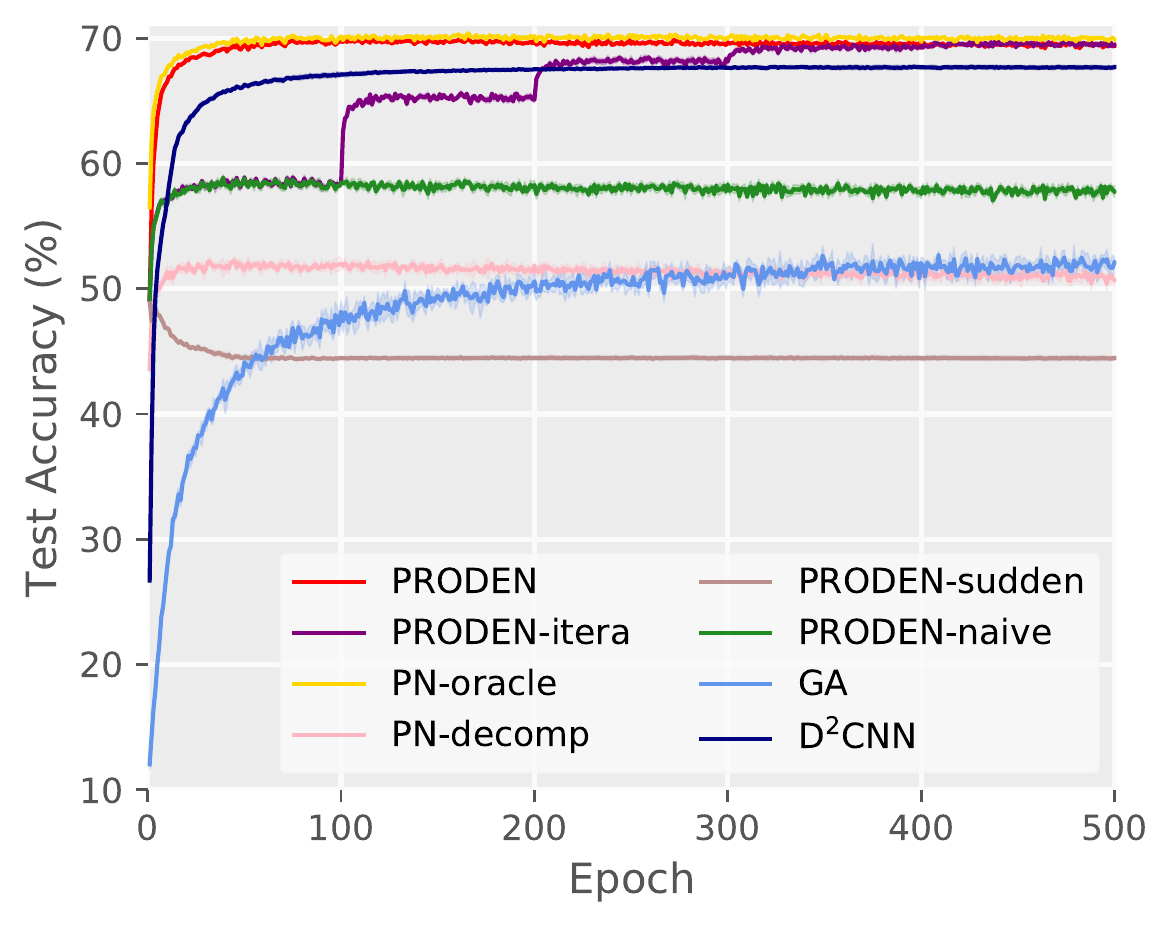}
			\centerline{\quad Linear, $q=0.5$}
	\end{minipage}}
	\subfigure{
		\begin{minipage}[b]{0.5\columnwidth}
			\includegraphics[width=1.6in]{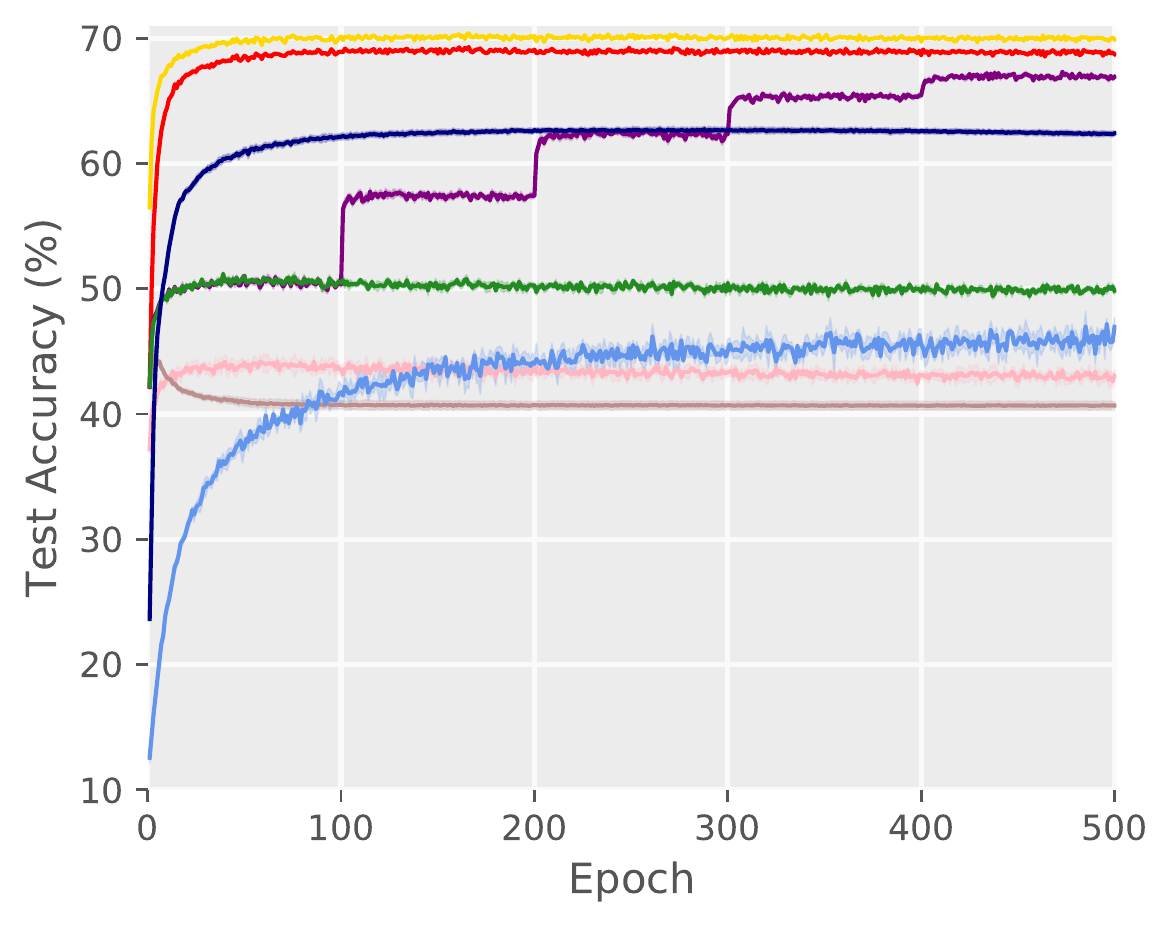}
			\centerline{\quad Linear, $q=0.7$}
	\end{minipage}}
	\subfigure{
		\begin{minipage}[b]{0.5\columnwidth}
			\includegraphics[width=1.6in]{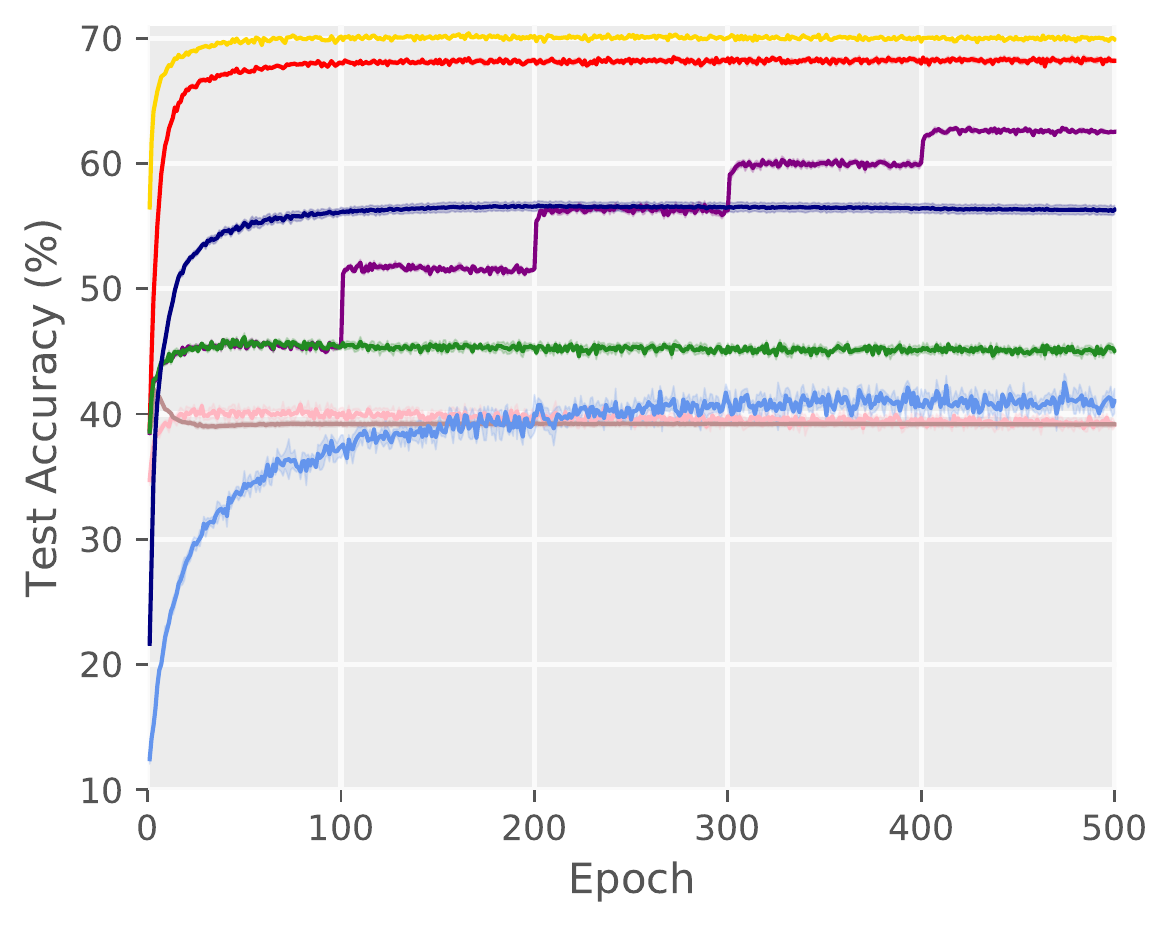}
			\centerline{\quad Linear, $q=0.8$}
	\end{minipage}}
	\subfigure{
		\begin{minipage}[b]{0.5\columnwidth}
			\includegraphics[width=1.6in]{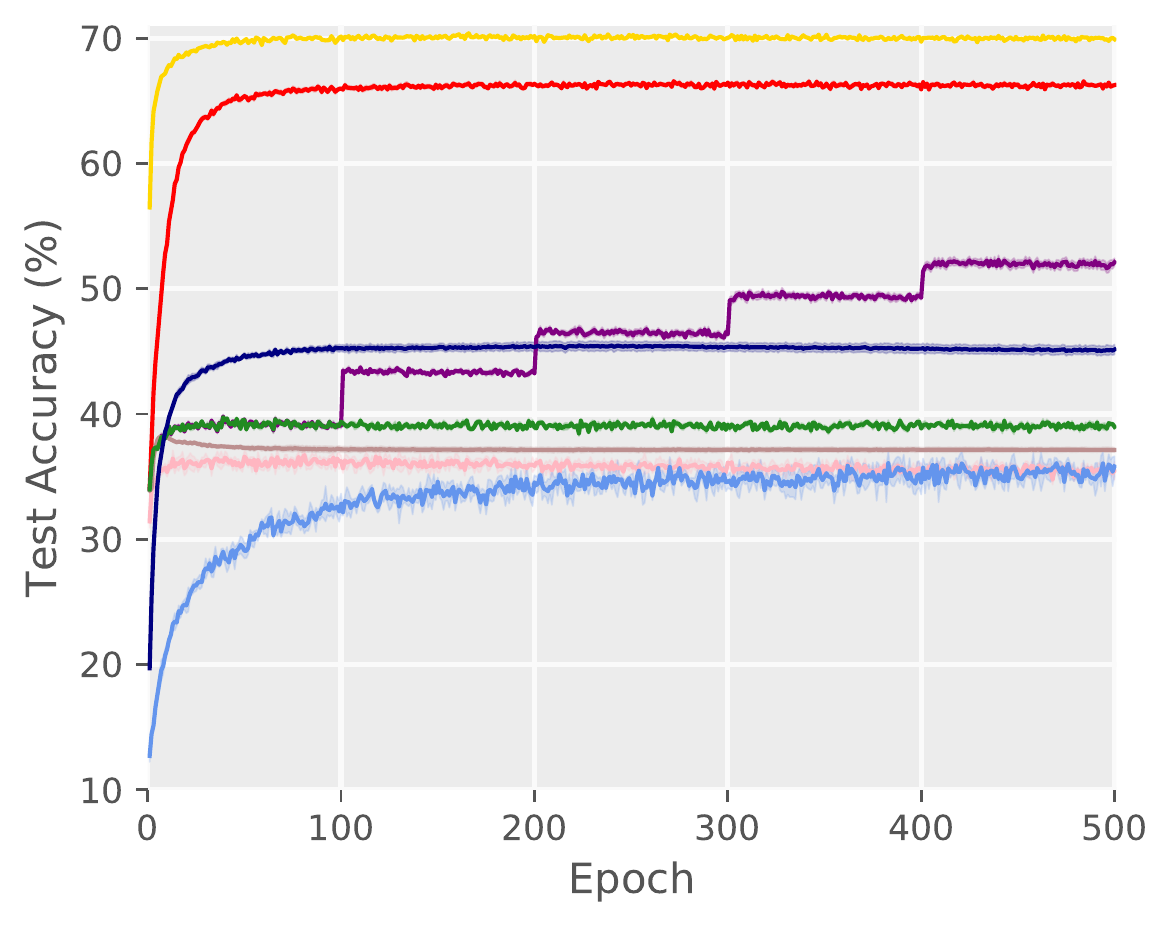}
			\centerline{\quad Linear, $q=0.9$}
	\end{minipage}}
	
	\subfigure{
		\begin{minipage}[b]{0.5\columnwidth}
			\includegraphics[width=1.6in]{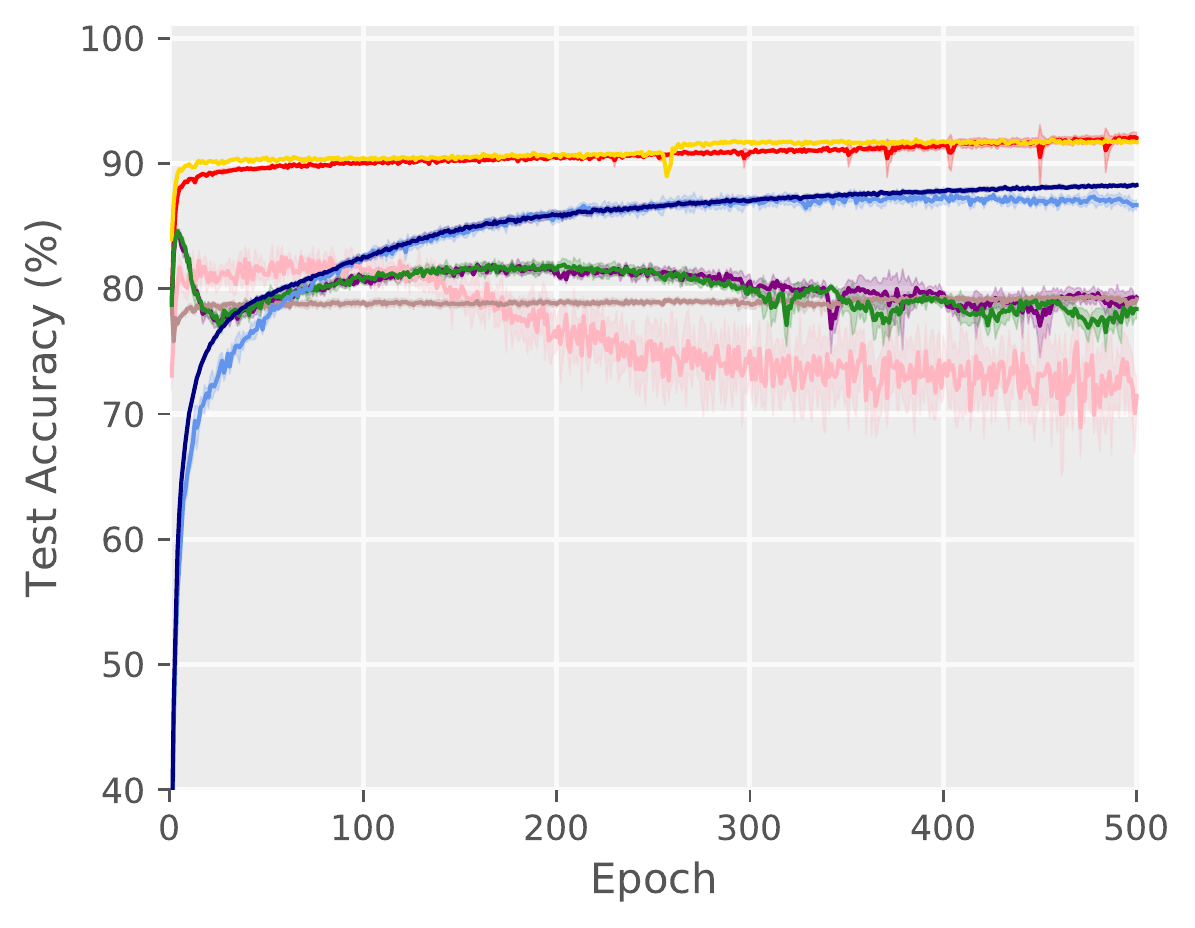}
			\centerline{\quad MLP, $q=0.5$}
	\end{minipage}}
	\subfigure{
		\begin{minipage}[b]{0.5\columnwidth}
			\includegraphics[width=1.6in]{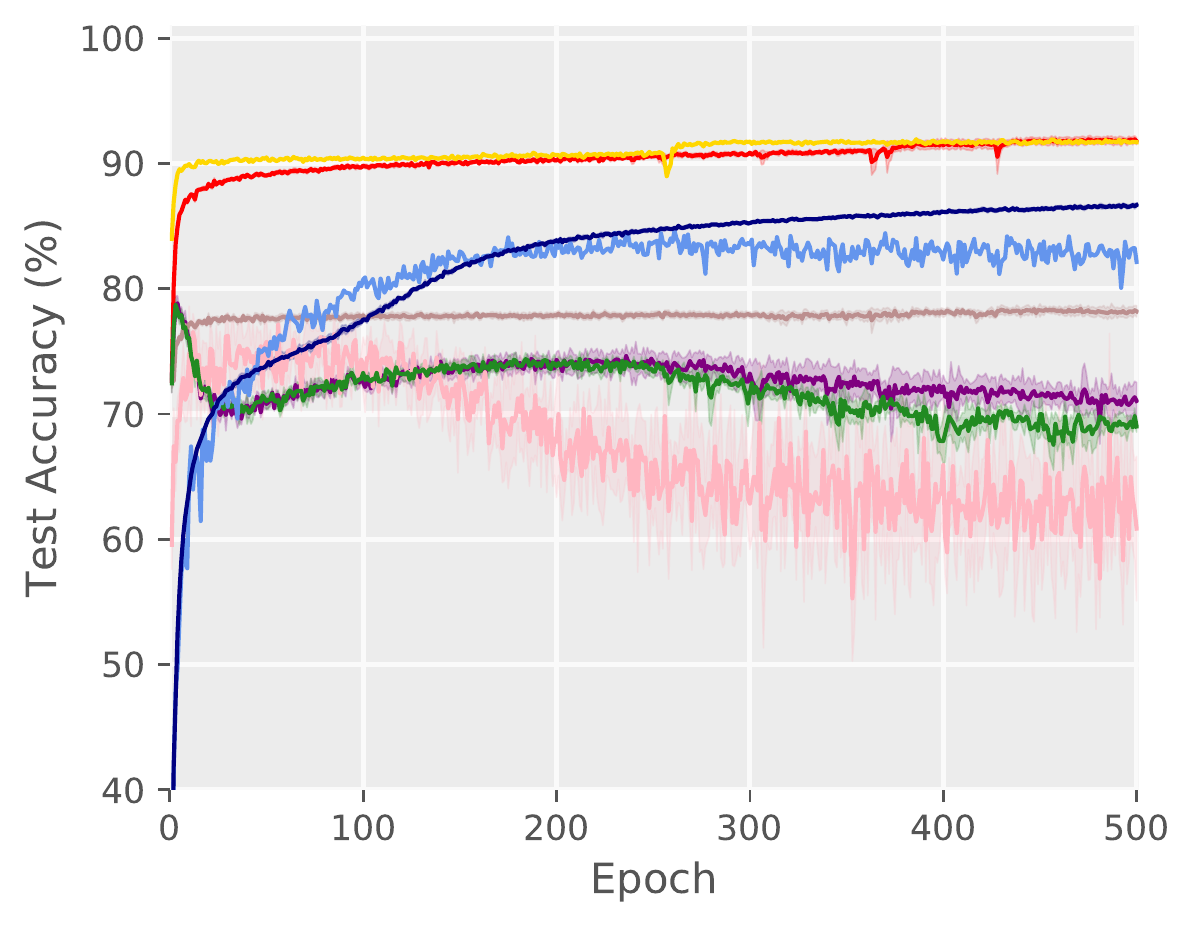}
			\centerline{\quad MLP, $q=0.7$}
	\end{minipage}}
	\subfigure{
		\begin{minipage}[b]{0.5\columnwidth}
			\includegraphics[width=1.6in]{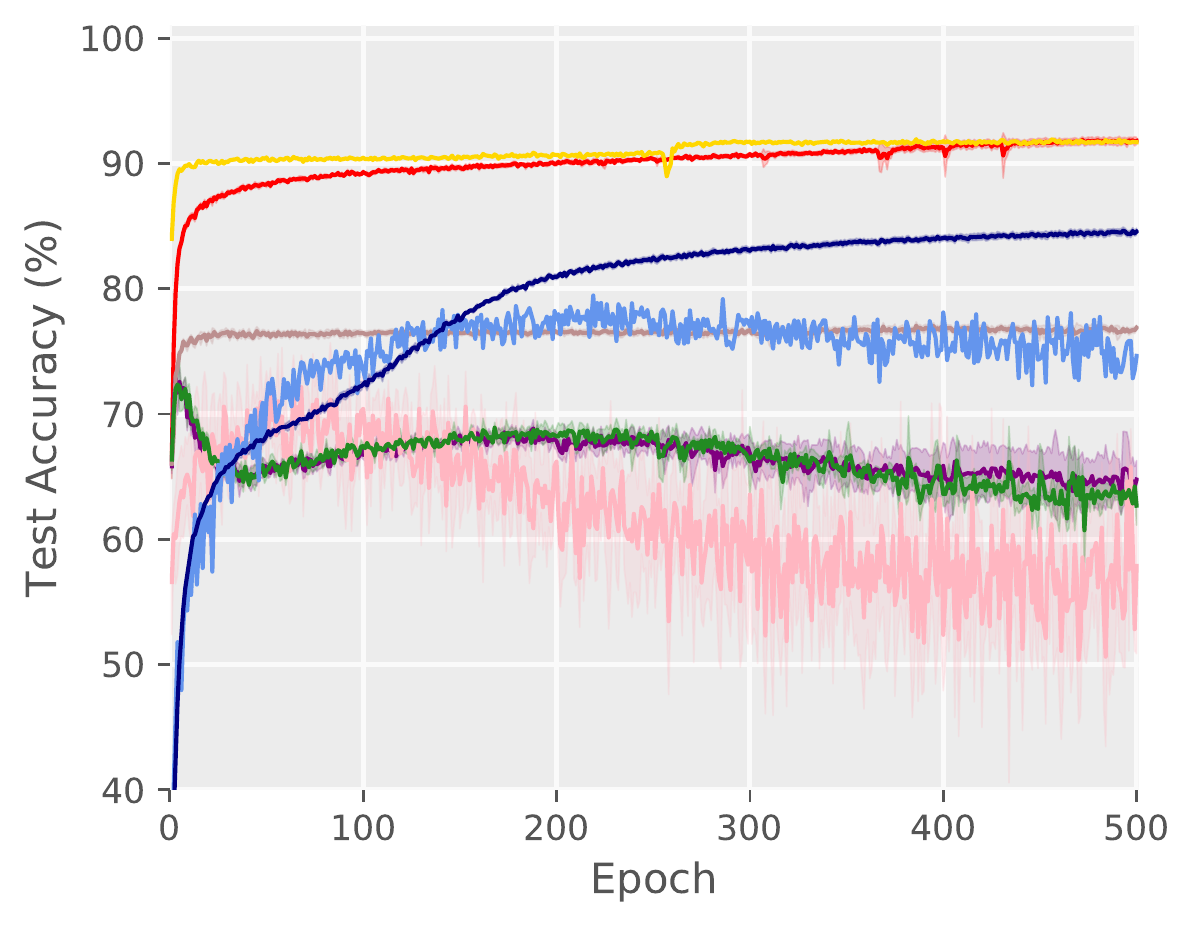}
			\centerline{\quad MLP, $q=0.8$}
	\end{minipage}}
	\subfigure{
		\begin{minipage}[b]{0.5\columnwidth}
			\includegraphics[width=1.6in]{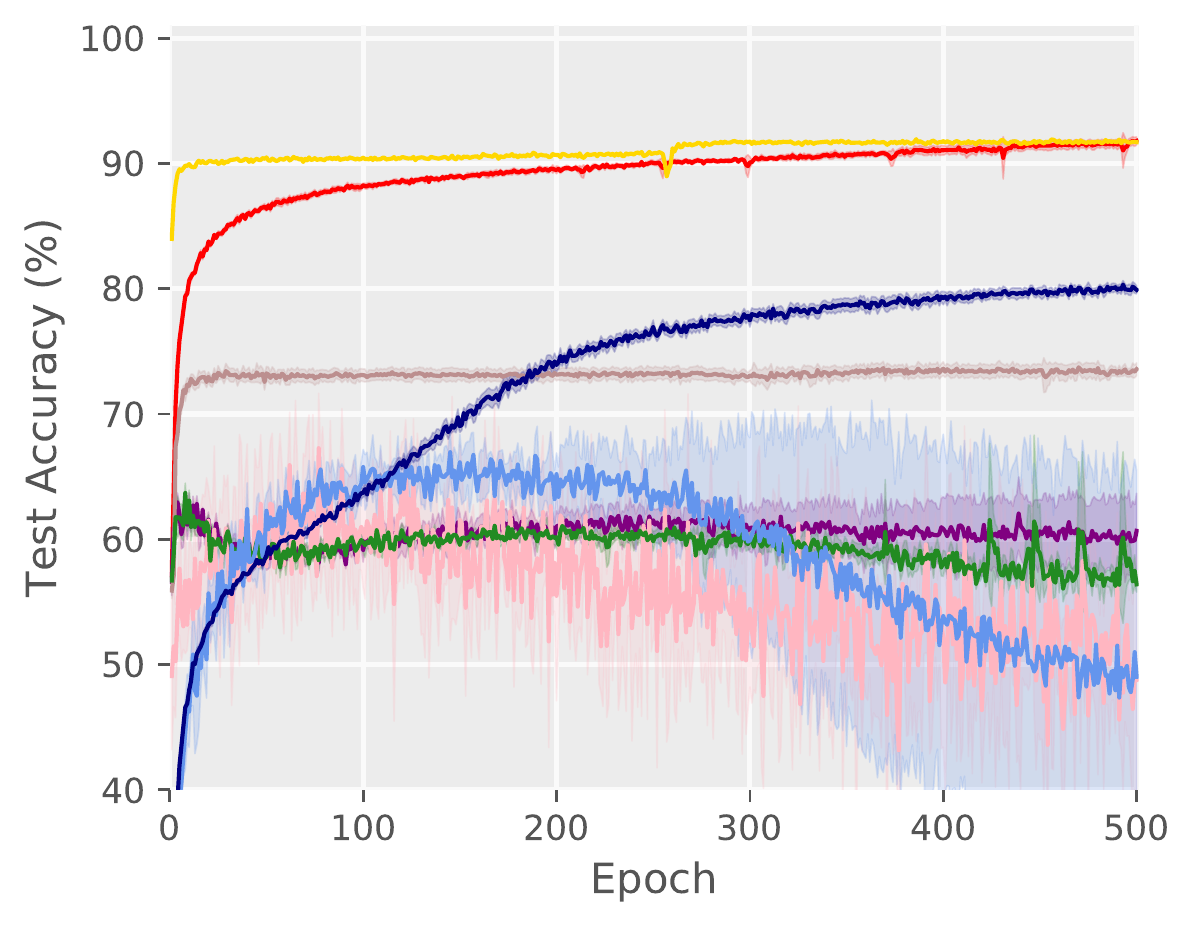}
			\centerline{\quad MLP, $q=0.9$}
	\end{minipage}}
	\caption{Test accuracy on Kuzushiji-MNIST in the pair case.}
	\label{fig:benchmark_pair_test}
	\vskip -0.2in
\end{figure*}

In this section, we experimentally analyze the proposed method PRODEN, and compare with state-of-the-art PLL methods. The implementation is based on PyTorch \cite{paszke2019pytorch} and experiments were carried out with NVIDIA Tesla V100 GPU; it is available at \url{https://github.com/Lvcrezia77/PRODEN}.

\textbf{Datasets} 
\quad We use four widely used benchmark datasets including MNIST \cite{lecun1998gradient}, Fashion-MNIST \cite{xiao2017fashion}, Kuzushiji-MNIST \cite{clanuwat2018deep}, and CIFAR-10 \cite{krizhevsky2009learning}, and five datasets from the UCI Machine Learning Repository \cite{krizhevsky2009learning}, including Yeast, Texture, Dermatology, Synthetic Control, and 20Newgroups. Similar to the corrupted strategy in label noise \cite{patrini2017making, han2018co}, we manually corrupt these datasets into partially labeled versions by a flipping probability $q$ where $q={\rm Pr}(\tilde{y}=1|y=0)$ gives the probability that false positive label $\tilde{y}$ is flipped from negative label $y$. We adopt a \emph{binomial} flipping strategy: $c-1$ independent experiments are conducted on all training examples, each determining whether a negative label is flipped with probability $q$. Then for the examples that none of the negative labels is flipped, we additionally flip a random negative label to the candidate label set for ensuring all the training examples are partially labeled. In this paper, we conduct the experiments under both less-partial circumstances $q=0.1$ and strong-partial circumstances $q=0.7$. In addition, five real-world partial-label datasets are adopted, including Lost \cite{cour2011learning}, Birdsong \cite{briggs2012rank}, MSRCv2 \cite{liu2012conditional}, Soccer Player \cite{zeng13learning}, and Yahoo! News \cite{guillaumin2010multiple}.

\textbf{Baselines}
\quad We first analyze PRODEN by comparing it with seven baseline methods based on DNNs, including three variants of PRODEN, two supervised methods, and two state-of-the-art PLL methods:
\begin{itemize}[topsep=0ex,itemsep=-1ex,leftmargin=*]
	\item \emph{PRODEN-itera} means updating the label weights in the iterative EM manner instead of a seamless manner proposed in Section~\ref{sec:method}. 
	\item \emph{PRODEN-sudden} means performing a \emph{sudden} identification, i.e., updating the weights $w_{ik}=1$ if $\argmax_{j\in s_i}g_j(x_i)=k$ and $w_{ij}=0, \forall j\neq k$ in every iteration step.
	\item \emph{PRODEN-naive} means never updating the uniform weights.
	\item \emph{PN-oracle} means supervised learning from ordinary labels. It should achieve the best performance, and is used merely for a proof of concept. 
	\item \emph{PN-decomp} means decomposing one instance with multiple candidate labels into many (same) instances each one single label, so that we could use any ordinary multi-class classification methods.
	\item \emph{D$^2$CNN} \cite{yao2020deep} means a PLL learning method based on DNNs.
	\item \emph{GA} \cite{ishida2019complementary} means a CLL method based on DNNs and gradient ascent.
\end{itemize}
We employ multiple base models, including linear model, 5-layer perceptron (MLP), 12-layer ConvNet \cite{laine17temporal} and 32-layer ResNet \cite{he2016deep} to show that our proposal is compatible with a wide family of learning models. The detailed descriptions of the datasets with their corresponding models are provided in Appendix. The optimizer is stochastic gradient descent (SGD) \cite{robbins1951stochastic} with momentum 0.9. We train each model 500 epochs with softmax function and cross-entropy loss. And PRODEN-itera updates label weights every 100 epochs.

We further compare PRODEN with six classical PLL methods that can hardly be implemented by DNNs on small-scale datasets (UCI and real-world datasets). They are four parametric methods: \emph{SURE} \cite{feng2019partial}, \emph{CLPL} \cite{cour2011learning}, \emph{ECOC} \cite{zhang2017disambiguation}, \emph{PLSVM} \cite{Nguyen2008classification}, and two non-parametric methods: \emph{PLkNN} \cite{Hullermeier2006learning} and \emph{IPAL} \cite{zhang2015solving}. Their hyper-parameters are selected according to the suggested parameter settings in original papers. Detailed information is provided in Appendix. For a fair comparison, PRODEN and all the parametric methods employ a linear model. We average the classification accuracy of PRODEN over the last 10 epochs as the results to prove that PRODEN is always stable and will not cause performance degradation due to overfitting when the number of epochs increases.

\textbf{Results on benchmark datasets}
\quad We record inductive results which indicate the classification accuracy in the test set (test accuracy). The means and standard deviations of test accuracy based on 5 random samplings in the binomial case are shown in Figure~\ref{fig:benchmark_binomial_test}. The transductive accuracy which reflects the ability in identifying the true labels in the training set can be found in Appendix.

We first observe the performance when $q=0.1$ (the left two columns). PRODEN is always the best method and comparable to PN-oracle with all the models. PRODEN-itera is comparable to PRODEN with a linear model, but its performance deteriorates drastically with complex models because of the overfitting issues. This phenomenon is consistent with the discussions in Section~\ref{sec:method}. In PRODEN-naive, the weights are always equally distributed, while in PRODEN-sudden, the sudden identification concentrates all the weights to the winners, resulting in their poorer performance.

Next, we compare these methods when $q=0.7$ (the right two columns). This is significantly harder than $q=0.1$. But when the task difficulty is relatively low (e.g. MNIST, Fashion-MNIST), PRODEN is still comparable to PN-oracle with such a large flipping probability, while the performance of the baselines is greatly reduced. The superiority always stands out for PRODEN compared with two latest deep methods GA and D$^2$CNN.

\begin{table*}[!t]
	\centering
	\caption{Test accuracy (mean$\pm$std) on the UCI datasets in the binomial case.}
	\label{tab:uci_binomial_test}
	\vskip 0.15in
	\renewcommand\arraystretch{1.2}
	\begin{small}
		\setlength{\tabcolsep}{3.3mm}{
			\begin{tabular}{c|cccccc}
				\toprule
				&q & Yeast & Texture & Dermatology & Synthetic Control & 20Newsgroups\\
				\midrule
				\multirow{2}{*}{PRODEN}     &0.1              & \bf{59.05$\pm$4.66\%}    & 99.64$\pm$0.16\%      & \bf{96.16$\pm$3.26\%}   & \bf{98.33$\pm$1.32\%}         & 77.28$\pm$0.68\%  \\ 
				&0.7 & \bf{55.15$\pm$3.87\%}     & \bf{99.33$\pm$0.32\%}      & \bf{95.34$\pm$3.15\%}   & \bf{95.00$\pm$4.75\%}         & \bf{64.74$\pm$0.90\%}  \\ 
				\midrule
				\multirow{2}{*}{PRODEN-itera}              &0.1     & 57.55$\pm$3.96\%     & \bf{99.73$\pm$0.13}\%      & 95.56$\pm$2.90\%   & 82.38$\pm$4.18\%$\bullet$         & 77.04$\pm$0.77\%  \\
				&0.7 & 53.62$\pm$5.54\%     & 94.79$\pm$1.63\%$\bullet$      & 76.22$\pm$4.48\%$\bullet$   & 48.33$\pm$7.96\%$\bullet$         & 52.43$\pm$0.75\%$\bullet$  \\ 
				\midrule
				\multirow{2}{*}{GA}  &0.1& 25.43$\pm$3.81\%$\bullet$    &  96.77$\pm$0.51\%$\bullet$  & 73.53$\pm$7.73\%$\bullet$  &   64.70$\pm$0.42\%$\bullet$   & 66.53$\pm$2.60\%$\bullet$  \\
				&0.7& 22.59$\pm$2.51\%$\bullet$    &  95.76$\pm$0.62\%$\bullet$  & 52.00$\pm$12.32\%$\bullet$  &   49.72$\pm$6.38\%$\bullet$   & 56.47$\pm$0.59\%$\bullet$  \\
				\midrule
				\multirow{2}{*}{D$^2$CNN}  &0.1& 59.05$\pm$4.66\%    &  98.80$\pm$0.31\%$\bullet$  & 95.34$\pm$2.29\% &   84.50$\pm$11.28\%$\bullet$      &   73.20$\pm$0.46\%$\bullet$  \\
				&0.7& 55.15$\pm$3.87\%   &  97.23$\pm$0.71\%$\bullet$  & 91.59$\pm$1.08\%$\bullet$  &   71.00$\pm$13.02\%$\bullet$   &  52.13$\pm$0.41\%$\bullet$  \\
				\midrule
				\multirow{2}{*}{SURE}   &0.1& 55.52$\pm$4.92\%   &  97.96$\pm$0.32\%$\bullet$   &  95.16$\pm$2.25\%   &  76.67$\pm$2.83\%$\bullet$    &  69.82$\pm$0.26\%$\bullet$\\ 
				&0.7& 49.32$\pm$3.95\%$\bullet$  &  94.75$\pm$3.68\%$\bullet$   &  90.96$\pm$2.02\%$\bullet$   &  54.67$\pm$7.67\%$\bullet$    &  62.66$\pm$1.21\%$\bullet$\\
				\midrule
				\multirow{2}{*}{CLPL} &0.1& 57.90$\pm$4.35\% & 98.78$\pm$0.22\%$\bullet$   & 94.79$\pm$3.27\%   &  78.00$\pm$3.10\%$\bullet$ &  76.26$\pm$0.77\%  \\
				&0.7& 54.88$\pm$7.78\% & 95.09$\pm$1.01\%$\bullet$   & 92.33$\pm$2.67\%$\bullet$   &  63.33$\pm$7.02\%$\bullet$ &  50.37$\pm$0.57\%$\bullet$\\ 
				\midrule
				\multirow{2}{*}{ECOC}  &0.1&  59.01$\pm$3.72\%  & 99.15$\pm$0.27\%$\bullet$   &     94.71$\pm$2.08\%    &  96.67$\pm$2.08\%  & \bf{77.67$\pm$1.11\%}   \\
				&0.7&  53.37$\pm$2.37\%$\bullet$  & 97.69$\pm$0.82\%$\bullet$   &     91.47$\pm$3.02\%$\bullet$    &  91.67$\pm$2.08\%$\bullet$ & 61.32$\pm$2.45\%$\bullet$ \\
				\midrule
				\multirow{2}{*}{PLSVM} &0.1& 54.73$\pm$2.89\% & 93.75$\pm$1.99\%$\bullet$   &  92.88$\pm$2.25\%$\bullet$    &  92.50$\pm$2.12\%$\bullet$ & 76.25$\pm$1.22\% \\
				&0.7& 42.34$\pm$1.92\%$\bullet$ & 50.69$\pm$5.62\%$\bullet$   &  89.32$\pm$4.98\%$\bullet$    &  85.50$\pm$2.61\%$\bullet$ & 59.12$\pm$0.41\%$\bullet$ \\
				\midrule
				\multirow{2}{*}{PL$k$NN}    &0.1& 54.28$\pm$3.64\% & 97.20$\pm$0.28\%$\bullet$  &  94.52$\pm$2.74\%   & 94.83$\pm$2.08\%$\bullet$ & 42.39$\pm$0.97\%$\bullet$   \\
				&0.7& 28.72$\pm$1.01\%$\bullet$ & 96.78$\pm$0.31\%$\bullet$  &  85.48$\pm$3.70\%$\bullet$   & 85.50$\pm$6.36\%$\bullet$ & 18.01$\pm$0.34\%$\bullet$   \\
				\midrule
				\multirow{2}{*}{IPAL}  &0.1& 50.56$\pm$2.83\%$\bullet$    & 99.33$\pm$0.26\%$\bullet$  & 95.34$\pm$2.29\%  &   98.33$\pm$0.83\%   & 75.01$\pm$0.69\%$\bullet$  \\
				&0.7& 41.10$\pm$3.71\%$\bullet$    &  97.84$\pm$0.57\%$\bullet$  & 95.02$\pm$2.03\%  &   94.83$\pm$6.55\%   & 56.93$\pm$0.50\%$\bullet$  \\
				\bottomrule
			\end{tabular}
		}
	\end{small}
	\vskip -0.1in
\end{table*}

\begin{table*}[!t]
	\centering
	\caption{Test accuracy (mean$\pm$std) on the real-world datasets.}
	\label{tab:realworld_test}
	\vskip 0.15in
	\renewcommand\arraystretch{1.2}
	\begin{small}
		\setlength{\tabcolsep}{4.5mm}{
			\begin{tabular}{c|ccccc}
				\toprule
				& Lost & Birdsong & MSRCv2 & Soccer Player & Yahoo! News\\
				\midrule
				PRODEN                   & \bf{76.57$\pm$1.47\%}     & \bf{72.01$\pm$0.44\%}      & 45.27$\pm$1.73\%   & \bf{55.99$\pm$0.58\%}         & \bf{67.40$\pm$0.55\%}  \\ 
				\midrule
				PRODEN-itera                   & 68.09$\pm$3.78\%$\bullet$     & 68.02$\pm$0.76\%$\bullet$      & 42.79$\pm$2.80\%   & 53.50$\pm$0.94\%$\bullet$         & 67.04$\pm$0.67\%  \\
				D$^2$CNN  & 69.61$\pm$5.48\%$\bullet$    &  66.58$\pm$1.49\%$\bullet$  & 40.17$\pm$1.99\%$\bullet$ &   49.06$\pm$0.15\%$\bullet$   & 56.39$\pm$0.89\%$\bullet$  \\
				GA  & 48.21$\pm$4.44\%$\bullet$    &  29.77$\pm$1.43\%$\bullet$  & 22.30$\pm$2.71\%$\bullet$ &   51.88$\pm$0.44\%$\bullet$   & 34.32$\pm$0.95\%$\bullet$  \\
				SURE   & 71.61$\pm$3.44\%$\bullet$   &  58.04$\pm$1.22\%$\bullet$   &  31.57$\pm$2.48\%$\bullet$  &  49.16$\pm$0.20\%$\bullet$    &  45.73$\pm$0.90\%$\bullet$\\
				CLPL & 76.17$\pm$1.81\% & 67.56$\pm$1.12\%$\bullet$   & 43.64$\pm$0.24\%   &  49.88$\pm$4.29\%$\bullet$ &  53.74$\pm$0.95\%$\bullet$\\ 
				ECOC  &  63.93$\pm$5.45\%$\bullet$  & 71.47$\pm$1.24\%$\bullet$   &     46.78$\pm$2.84\%    &  55.51$\pm$0.54\% & 64.78$\pm$0.78\%$\bullet$   \\
				PLSVM & 72.86$\pm$5.45\%$\bullet$ & 60.46$\pm$1.99\%$\bullet$   &  38.97$\pm$4.62\%$\bullet$    &  46.15$\pm$1.00\%$\bullet$ & 60.46$\pm$1.48\%$\bullet$ \\
				PL$k$NN    & 35.00$\pm$4.71\%$\bullet$ & 64.22$\pm$1.14\%$\bullet$  &  41.60$\pm$2.30\%$\bullet$   & 49.18$\pm$0.26\%$\bullet$ & 40.30$\pm$0.90\%$\bullet$   \\
				IPAL  & 71.25$\pm$1.40\%$\bullet$    &  71.19$\pm$1.54\%  & \bf{52.36$\pm$2.87\%}$\circ$  &   54.41$\pm$0.56\%$\bullet$   & 66.22$\pm$0.80\%$\bullet$  \\
				\bottomrule
			\end{tabular}
		}
	\end{small}
\end{table*}

\textbf{Analysis on the ambiguity degree}
\quad It is intuitive that the higher the ambiguity degree is, the more difficult it is to find the true labels. We investigate the influence of ambiguity degree $\gamma$ through a \emph{pair} flipping strategy: only the classes that are similar to the true class will be put into the candidate label set with probability $q$. We gradually move $q$ from 0.5 to 0.9 to simulate $\gamma$ ($\gamma \rightarrow q$ as $n \rightarrow \infty$), and the experimental results on Kuzushiji-MNIST are reported in Figure~\ref{fig:benchmark_pair_test}, and phenomena on other datasets are similar and reported in Appendix. We can see that even when the ambiguity degree is high, PRODEN is still highly competitive. PRODEN tends to be less affected with increased ambiguity, while the baselines are affected severely.

\textbf{Results on small-scale datasets}
\quad We perform five-fold cross-validation, and use paired $t$-test at 5\% significance level. Table~\ref{tab:uci_binomial_test} and Table~\ref{tab:realworld_test} report the mean test accuracy with standard deviation on UCI datasets and the real-world datasets, respectively. $\bullet/\circ$ represents whether PRODEN is significantly better/worse than the comparing methods. Clearly, PRODEN is overall the best performing, confirming the advantage of PRODEN afforded not only by the network architecture, but also the progressive identification process.

\section{Conclusion}
In this paper, we focused on proposing a novel method for PLL which is compatible with flexible multi-class classifiers and stochastic optimization. We first proposed a classifier-consistent risk estimator based on the minimal loss incurred by candidate labels. The estimator has good theoretical properties, while the $\min$ operator makes it difficult to be solved by the epoch-wise optimization of neural networks. Therefore, we proposed a progressive identification method named PRODEN for approximately minimizing the proposed risk estimator, whose idea is relaxing $\min$ to a weighted combination. Learning the weights and the classifier are then conducted in a seamless manner for mitigating the overfitting problem. At last, experiments demonstrated that the proposed method could successfully train various models, and it compared favorably with state-of-the-art methods.

\newpage
\section*{Acknowledgements}
JL and XG were supported by the National Key Research \&
Development Plan of China (No. 2017YFB1002801), the National Science
Foundation of China (61622203), the Collaborative Innovation Center of
Novel Software Technology and Industrialization, and the Collaborative
Innovation Center of Wireless Communications Technology. MX was supported by the Open Project Program of the State Key Lab. for Novel Software Technology, Nanjing University, Nanjing, 210093, P.R.China. LF was supported by MOE, NRF, and NTU, Singapore. GN and MS were supported by JST AIP Acceleration Research Grant Number JPMJCR20U3, Japan.

\bibliography{proden-icml20}

\begin{thebibliography}{61}
\providecommand{\natexlab}[1]{#1}
\providecommand{\url}[1]{\texttt{#1}}
\expandafter\ifx\csname urlstyle\endcsname\relax
  \providecommand{\doi}[1]{doi: #1}\else
  \providecommand{\doi}{doi: \begingroup \urlstyle{rm}\Url}\fi

\bibitem[Aharon et~al.(2006)Aharon, Elad, and Bruckstein]{aharon2011the}
Aharon, M., Elad, M., and Bruckstein, A.~M.
\newblock The k-svd: An algorithm for designing of overcomplete dictionaries
  for sparse representation.
\newblock \emph{IEEE Transactions on Signal Processing}, 54\penalty0
  (11):\penalty0 4311--4322, 2006.

\bibitem[Arpit et~al.(2017)Arpit, Jastrzebski, Ballas, Krueger, Bengio, Kanwal,
  Maharaj, and Fischer]{arpit2017a}
Arpit, D., Jastrzebski, S., Ballas, N., Krueger, D., Bengio, E., Kanwal, M.,
  Maharaj, T., and Fischer, A.
\newblock A closer look at memorization in deep networks.
\newblock In \emph{Proceedings of 34th International Conference on Machine
  Learning (ICML'17)}, pp.\  233--242, Sydney, Australia, 2017.

\bibitem[Bartlett \& Mendelson(2002)Bartlett and
  Mendelson]{bartlett2002rademacher}
Bartlett, P.~L. and Mendelson, S.
\newblock Rademacher and gaussian complexities: Risk bounds and structural
  results.
\newblock \emph{Journal of Machine Learning Research}, 3\penalty0
  (11):\penalty0 463--482, 2002.

\bibitem[Bertsekas(1997)]{bertsekas1997nonlinear}
Bertsekas, D.~P.
\newblock Nonlinear programming.
\newblock \emph{Journal of the Operational Research Society}, 48\penalty0
  (3):\penalty0 334--334, 1997.

\bibitem[Bottou \& Bousquet(2007)Bottou and Bousquet]{bottou2007the}
Bottou, L. and Bousquet, O.
\newblock The tradeoffs of large scale learning.
\newblock In \emph{Advances in Neural Information Processing Systems 20
  (NIPS'07)}, volume~20, pp.\  1--8, Vancouver, Canada, 2007.

\bibitem[Briggs et~al.(2012)Briggs, Fern, and Raich]{briggs2012rank}
Briggs, F., Fern, X.~Z., and Raich, R.
\newblock Rank-loss support instance machines for miml instance annotation.
\newblock In \emph{Proceedings of the 18th ACM SIGKDD International Conference
  on Knowledge Discovery and Data Mining (KDD'12)}, pp.\  534--542, Beijing,
  China, 2012.

\bibitem[Chen et~al.(2018)Chen, Patel, and Chellappa]{chen18learning}
Chen, C., Patel, V.~M., and Chellappa, R.
\newblock Learning from ambiguously labeled face images.
\newblock \emph{IEEE Transactions on Pattern Analysis and Machine
  Intelligence}, 40\penalty0 (7):\penalty0 1653--1667, 2018.

\bibitem[Chen et~al.(2013)Chen, Patel, Pillai, Chellappa, and
  Phillips]{chen13dictionary}
Chen, Y., Patel, V.~M., Pillai, J.~K., Chellappa, R., and Phillips, P.~J.
\newblock Dictionary learning from ambiguously labeled data.
\newblock In \emph{Proceedings of the 26th IEEE Conference on Computer Vision
  and Pattern Recognition (CVPR'13)}, pp.\  353--360, Portland, OR, 2013.

\bibitem[Chen et~al.(2014)Chen, Patel, Pillai, Chellappa, and
  Phillips]{chen14ambiguously}
Chen, Y., Patel, V.~M., Pillai, J.~K., Chellappa, R., and Phillips, P.~J.
\newblock Ambiguously labeled learning using dictionaries.
\newblock \emph{IEEE Transactions on Information Forensics and Security},
  9\penalty0 (12):\penalty0 2076--2088, 2014.

\bibitem[Clanuwat et~al.(2018)Clanuwat, Bober-Irizar, Kitamoto, Lamb, Yamamoto,
  and Ha]{clanuwat2018deep}
Clanuwat, T., Bober-Irizar, M., Kitamoto, A., Lamb, A., Yamamoto, K., and Ha,
  D.
\newblock Deep learning for classical japanese literature.
\newblock \emph{arXiv preprint arXiv:1812.01718}, 2018.

\bibitem[Cour et~al.(2011)Cour, Sapp, and Taskar]{cour2011learning}
Cour, T., Sapp, B., and Taskar, B.
\newblock Learning from partial labels.
\newblock \emph{Journal of Machine Learning Research}, 12\penalty0
  (5):\penalty0 1501--1536, 2011.

\bibitem[Duchi et~al.(2011)Duchi, Hazan, and Singer]{duchi2011adaptive}
Duchi, J., Hazan, E., and Singer, Y.
\newblock Adaptive subgradient methods for online learning and stochastic
  optimization.
\newblock \emph{Journal of Machine Learning Research}, 12\penalty0
  (7):\penalty0 2121--2159, 2011.

\bibitem[Fan et~al.(2008)Fan, Chang, Hsieh, Wang, and Lin]{fan2008liblinear}
Fan, R., Chang, K., Hsieh, C., Wang, X., and Lin, C.
\newblock Liblinear: A library for large linear classification.
\newblock \emph{Journal of Machine Learning Research}, 9\penalty0 (8):\penalty0
  1871--1874, 2008.

\bibitem[Feng \& An(2019{\natexlab{a}})Feng and An]{feng2019partial}
Feng, L. and An, B.
\newblock Partial label learning with self-guided retraining.
\newblock In \emph{Proceedings of 33rd AAAI Conference on Artificial
  Intelligence (AAAI'19)}, pp.\  3542--3549, Honolulu, HI, 2019{\natexlab{a}}.

\bibitem[Feng \& An(2019{\natexlab{b}})Feng and An]{feng2019partialb}
Feng, L. and An, B.
\newblock Partial label learning by semantic difference maximization.
\newblock In \emph{Proceedings of the 28th International Joint Conference on
  Artificial Intelligence (IJCAI'19)}, pp.\  2294--2300, Macao, China,
  2019{\natexlab{b}}.

\bibitem[Feng et~al.(2020)Feng, Kaneko, Han, Niu, An, and
  Sugiyama]{feng2020learning}
Feng, L., Kaneko, T., Han, B., Niu, G., An, B., and Sugiyama, M.
\newblock Learning with multiple complementary labels.
\newblock In \emph{Proceedings of 37th International Conference on Machine
  Learning (ICML'20)}, 2020.

\bibitem[Ghosh et~al.(2017)Ghosh, Kumar, and Sastry]{ghosh2017robust}
Ghosh, A., Kumar, H., and Sastry, P.
\newblock Robust loss functions under label noise for deep neural networks.
\newblock In \emph{Proceedings of 31st AAAI Conference on Artificial
  Intelligence (AAAI'17)}, pp.\  1919--1925, San Francisco, CA, 2017.

\bibitem[Gong et~al.(2017)Gong, Liu, Tang, Yang, Yang, and
  Tao]{gong2017regularization}
Gong, C., Liu, T., Tang, Y., Yang, J., Yang, J., and Tao, D.
\newblock A regularization approach for instance-based superset label learning.
\newblock \emph{IEEE Transactions on Cybernetics}, 48\penalty0 (3):\penalty0
  967--978, 2017.

\bibitem[Gong et~al.(2019)Gong, Shi, Liu, Zhang, Yang, and Tao]{gong2019loss}
Gong, C., Shi, H., Liu, T., Zhang, C., Yang, J., and Tao, D.
\newblock Loss decomposition and centroid estimation for positive and unlabeled
  learning.
\newblock \emph{IEEE Transactions on Pattern Analysis and Machine
  Intelligence}, 2019.

\bibitem[Goodfellow et~al.(2016)Goodfellow, Bengio, and
  Courville]{Goodfellow2016}
Goodfellow, I., Bengio, Y., and Courville, A.
\newblock \emph{Deep Learning}.
\newblock MIT Press, 2016.

\bibitem[Guillaumin et~al.(2010)Guillaumin, Verbeek, and
  Schmid]{guillaumin2010multiple}
Guillaumin, M., Verbeek, J., and Schmid, C.
\newblock Multiple instance metric learning from automatically labeled bags of
  faces.
\newblock \emph{Lecture Notes in Computer Science}, 63\penalty0 (11):\penalty0
  634--647, 2010.

\bibitem[H.~Xiao \& Vollgraf(2017)H.~Xiao and Vollgraf]{xiao2017fashion}
H.~Xiao, K.~R. and Vollgraf, R.
\newblock Fashion-mnist: a novel image dataset for benchmarking machine
  learning algorithms.
\newblock \emph{arXiv preprint arXiv:1708.07747}, 2017.

\bibitem[Han et~al.(2018)Han, Yao, Yu, Niu, Xu, Hu, Tsang, and
  Sugiyama]{han2018co}
Han, B., Yao, Q., Yu, X., Niu, G., Xu, M., Hu, W., Tsang, I., and Sugiyama, M.
\newblock Co-teaching: Robust training of deep neural networks with extremely
  noisy labels.
\newblock In \emph{Advances in Neural Information Processing Systems 31
  (NeurIPS'18)}, pp.\  8527--8537, Montreal, Canada, 2018.

\bibitem[He et~al.(2016)He, Zhang, Ren, and Sun]{he2016deep}
He, K., Zhang, X., Ren, S., and Sun, J.
\newblock Deep residual learning for image recognition.
\newblock In \emph{Proceedings of the 29th IEEE conference on Computer Vision
  and Pattern Recognition (CVPR'16)}, pp.\  770--778, Las Vegas, NV, 2016.

\bibitem[Hullermeier \& Beringer(2006)Hullermeier and
  Beringer]{Hullermeier2006learning}
Hullermeier, E. and Beringer, J.
\newblock Learning from ambiguously labeled examples.
\newblock \emph{Intelligent Data Analysis}, 10\penalty0 (5):\penalty0 419--439,
  2006.

\bibitem[Ioffe \& Szegedy(2015)Ioffe and Szegedy]{ioffe2015batch}
Ioffe, S. and Szegedy, C.
\newblock Batch normalization: Accelerating deep network training by reducing
  internal covariate shift.
\newblock In \emph{Proceedings of the 32nd International Conference on Machine
  Learning (ICML'15)}, pp.\  448--456, Lille, France, 2015.

\bibitem[Ishida et~al.(2017)Ishida, Niu, Hu, and Sugiyama]{ishida2017learning}
Ishida, T., Niu, G., Hu, W., and Sugiyama, M.
\newblock Learning from complementary labels.
\newblock In \emph{Advances in Neural Information Processing Systems 30
  (NIPS'17)}, pp.\  5639--5649, Long Beach, CA, 2017.

\bibitem[Ishida et~al.(2019)Ishida, Niu, Menon, and
  Sugiyama]{ishida2019complementary}
Ishida, T., Niu, G., Menon, A.~K., and Sugiyama, M.
\newblock Complementary-label learning for arbitrary losses and models.
\newblock In \emph{Proceedings of 36th International Conference on Machine
  Learning (ICML'19)}, pp.\  2971--2980, Long Beach, CA, 2019.

\bibitem[Jin \& Ghahramani(2003)Jin and Ghahramani]{jin2003learning}
Jin, R. and Ghahramani, Z.
\newblock Learning with multiple labels.
\newblock In \emph{Advances in Neural Information Processing Systems 16
  (NIPS'03)}, pp.\  921--928, Vancouver, Canada, 2003.

\bibitem[Kamnitsas et~al.(2018)Kamnitsas, Castro, Folgoc, Walker, Tanno,
  Rueckert, Glocker, Criminisi, and Nori]{kamnitsas2018semi}
Kamnitsas, K., Castro, D.~C., Folgoc, L.~L., Walker, I., Tanno, R., Rueckert,
  D., Glocker, B., Criminisi, A., and Nori, A.
\newblock Semi-supervised learning via compact latent space clustering.
\newblock In \emph{Proceedings of 35th International Conference on Machine
  Learning (ICML'18)}, pp.\  2464--2473, Stockholm, Sweden, 2018.

\bibitem[Krizhevsky \& Hinton(2009)Krizhevsky and
  Hinton]{krizhevsky2009learning}
Krizhevsky, A. and Hinton, G.
\newblock Learning multiple layers of features from tiny images.
\newblock 2009.

\bibitem[Laine \& Aila(2017)Laine and Aila]{laine17temporal}
Laine, S. and Aila, T.
\newblock Temporal ensembling for semi-supervised learning.
\newblock In \emph{Proceedings of 5th International Conference on Learning
  Representations (ICLR'17)}, Toulon, France, 2017.

\bibitem[LeCun et~al.(1998)LeCun, Bottou, Bengio, and
  Haffner]{lecun1998gradient}
LeCun, Y., Bottou, L., Bengio, Y., and Haffner, P.
\newblock Gradient-based learning applied to document recognition.
\newblock \emph{Proceedings of the IEEE}, 86\penalty0 (11):\penalty0
  2278--2324, 1998.

\bibitem[Ledoux \& Talagrand(2013)Ledoux and Talagrand]{ledoux2013probability}
Ledoux, M. and Talagrand, M.
\newblock \emph{Probability in Banach Spaces: Isoperimetry and Processes}.
\newblock Springer Science \& Business Media, 2013.

\bibitem[Liu \& Dietterich(2012)Liu and Dietterich]{liu2012conditional}
Liu, L. and Dietterich, T.~G.
\newblock A conditional multinomial mixture model for superset label learning.
\newblock In \emph{Advances in Neural Information Processing Systems 25
  (NIPS'12)}, pp.\  548--556, Lake Tahoe, NV, 2012.

\bibitem[Liu \& Dietterich(2014)Liu and Dietterich]{liu2014learnability}
Liu, L. and Dietterich, T.~G.
\newblock Learnability of the superset label learning problem.
\newblock In \emph{Proceedings of 31st International Conference on Machine
  Learning (ICML'14)}, pp.\  1629--1637, Beijing, China, 2014.

\bibitem[Lu et~al.(2019)Lu, Niu, Menon, and Sugiyama]{lu2018minimal}
Lu, N., Niu, G., Menon, A.~K., and Sugiyama, M.
\newblock On the minimal supervision for training any binary classifier from
  only unlabeled data.
\newblock In \emph{Proceedings of 7th International Conference on Learning
  Representations (ICLR'19)}, New Orleans, LA, 2019.

\bibitem[Luo \& Orabona(2010)Luo and Orabona]{luo2010learning}
Luo, J. and Orabona, F.
\newblock Learning from candidate labeling sets.
\newblock In \emph{Advances in Neural Information Processing Systems 23
  (NIPS'10)}, pp.\  1504--1512, Vancouver, Canada, 2010.

\bibitem[Lyu et~al.(2019)Lyu, Feng, Wang, Lang, and Li]{lyu2019gm}
Lyu, G., Feng, S., Wang, T., Lang, C., and Li, Y.
\newblock Gm-pll: Graph matching based partial label learning.
\newblock \emph{IEEE Transactions on Knowledge and Data Engineering}, 2019.

\bibitem[Maas et~al.(2013)Maas, Hannun, and Ng]{maas2013rectifier}
Maas, A.~L., Hannun, A.~Y., and Ng, A.~Y.
\newblock Rectifier nonlinearities improve neural network acoustic models.
\newblock In \emph{Proceedings of 30th International Conference on Machine
  Learning (ICML'13)}, volume~30, pp.\ ~3, Atlanta, GA, 2013.

\bibitem[Masnadi-Shirazi \& Vasconcelos(2009)Masnadi-Shirazi and
  Vasconcelos]{masnadi2009design}
Masnadi-Shirazi, H. and Vasconcelos, N.
\newblock On the design of loss functions for classification: theory,
  robustness to outliers, and savageboost.
\newblock In \emph{Advances in Neural Information Processing Systems 22
  (NIPS'09)}, pp.\  1049--1056, Vancouver, Canada, 2009.

\bibitem[Maurer(2016)]{maurer2016vector}
Maurer, A.
\newblock A vector-contraction inequality for rademacher complexities.
\newblock In \emph{International Conference on Algorithmic Learning Theory
  (ALT'16)}, pp.\  3--17, 2016.

\bibitem[McDiarmid(1989)]{mcdiarmid1989method}
McDiarmid, C.
\newblock On the method of bounded differences.
\newblock \emph{Surveys in combinatorics}, 141\penalty0 (1):\penalty0 148--188,
  1989.

\bibitem[Nguyen \& Caruana(2008)Nguyen and Caruana]{Nguyen2008classification}
Nguyen, N. and Caruana, R.
\newblock Classification with partial labels.
\newblock In \emph{Proceedings of the 14th ACM SIGKDD International Conference
  on Knowledge Discovery and Data Mining (KDD'08)}, pp.\  381--389, Las Vegas,
  NV, 2008.

\bibitem[Panis \& Lanitis(2014)Panis and Lanitis]{panis2014overview}
Panis, G. and Lanitis, A.
\newblock An overview of research activities in facial age estimation using the
  fg-net aging database.
\newblock In \emph{Proceedings of the 13th European Conference on Computer
  Vision (ECCV'14)}, pp.\  737--750, Zurich, Switzerland, 2014.

\bibitem[Paszke et~al.(2019)Paszke, Gross, Massa, Lerer, Bradbury, Chanan,
  Killeen, Lin, Gimelshein, Antiga, et~al.]{paszke2019pytorch}
Paszke, A., Gross, S., Massa, F., Lerer, A., Bradbury, J., Chanan, G., Killeen,
  T., Lin, Z., Gimelshein, N., Antiga, L., et~al.
\newblock Pytorch: An imperative style, high-performance deep learning library.
\newblock In \emph{Advances in Neural Information Processing Systems 32
  (NeurIPS'19)}, pp.\  8024--8035, Vancouver, Canada, 2019.

\bibitem[Patrini et~al.(2017)Patrini, Rozza, Menon, Nock, and
  Qu]{patrini2017making}
Patrini, G., Rozza, A., Menon, A.~K., Nock, R., and Qu, L.
\newblock Making deep neural networks robust to label noise: A loss correction
  approach.
\newblock In \emph{Proceedings of the 30th IEEE Conference on Computer Vision
  and Pattern Recognition (CVPR'17)}, pp.\  1944--1952, Honolulu, HI, 2017.

\bibitem[Robbins \& Monro(1951)Robbins and Monro]{robbins1951stochastic}
Robbins, H. and Monro, S.
\newblock A stochastic approximation method.
\newblock \emph{The Annals of Mathematical Statistics}, 22\penalty0
  (3):\penalty0 400--407, 1951.

\bibitem[Shrivastava et~al.(2015)Shrivastava, Patel, and
  Chellappa]{shrivastavav15non}
Shrivastava, A., Patel, V.~M., and Chellappa, R.
\newblock Non-linear dictionary learning with partially labeled data.
\newblock \emph{Pattern Recognition}, 48\penalty0 (11):\penalty0 3283--3292,
  2015.

\bibitem[Tang \& Zhang(2017)Tang and Zhang]{tang2017confidence}
Tang, C. and Zhang, M.
\newblock Confidence-rated discriminative partial label learning.
\newblock In \emph{Proceedings of 31st AAAI Conference on Artificial
  Intelligence (AAAI'17)}, pp.\  2611--2617, San Francisco, CA, 2017.

\bibitem[Vapnik(1998)]{vapnik1998statistical}
Vapnik, V.~N.
\newblock Statistical learning theory.
\newblock \emph{John Wiley \& Sons}, 1998.

\bibitem[Wang et~al.(2019)Wang, Liu, Zhao, Zhang, Hu, and
  Chen]{wang2019discriminative}
Wang, H., Liu, W., Zhao, Y., Zhang, C., Hu, T., and Chen, G.
\newblock Discriminative and correlative partial multi-label learning.
\newblock In \emph{Proceedings of 28th International Joint Conference on
  Artificial Intelligence (IJCAI'19)}, pp.\  3691--3697, Macao, China, 2019.

\bibitem[Xia et~al.(2019)Xia, Liu, Wang, Han, Gong, Niu, and
  Sugiyama]{xia2019anchor}
Xia, X., Liu, T., Wang, N., Han, B., Gong, C., Niu, G., and Sugiyama, M.
\newblock Are anchor points really indispensable in label-noise learning?
\newblock In \emph{Advances in Neural Information Processing Systems 32
  (NeurIPS'19)}, pp.\  6838--6849, Vancouver, Canada, 2019.

\bibitem[Yao et~al.(2020{\natexlab{a}})Yao, Chen, Deng, and
  Yang]{yao2020network}
Yao, Y., Chen, G., Deng, J., and Yang, J.
\newblock Network cooperation with progressive disambiguation for partial label
  learning.
\newblock \emph{arXiv preprint arXiv:2002.11919}, 2020{\natexlab{a}}.

\bibitem[Yao et~al.(2020{\natexlab{b}})Yao, Gong, Deng, Chen, Wu, and
  Yang]{yao2020deep}
Yao, Y., Gong, C., Deng, J., Chen, X., Wu, J., and Yang, J.
\newblock Deep discriminative cnn with temporal ensembling for
  ambiguously-labeled image classification.
\newblock In \emph{Proceedings of 34th AAAI Conference on Artificial
  Intelligence (AAAI'20)}, New York, NY, in press, 2020{\natexlab{b}}.

\bibitem[Yu \& Zhang(2017)Yu and Zhang]{yu2017maximum}
Yu, F. and Zhang, M.
\newblock Maximum margin partial label learning.
\newblock \emph{Machine Learning}, 106\penalty0 (4):\penalty0 573--593, 2017.

\bibitem[Yu et~al.(2018)Yu, Liu, Gong, and Tao]{yu2018learning}
Yu, X., Liu, T., Gong, M., and Tao, D.
\newblock Learning with biased complementary labels.
\newblock In \emph{Proceedings of the 15th European Conference on Computer
  Vision (ECCV'18)}, pp.\  68--83, Munich, Germany, 2018.

\bibitem[Zeng et~al.(2013)Zeng, Xiao, Jia, Chan, Gao, Xu, and
  Ma]{zeng13learning}
Zeng, Z., Xiao, S., Jia, K., Chan, T., Gao, S., Xu, D., and Ma, Y.
\newblock Learning by associating ambiguously labeled images.
\newblock In \emph{Proceedings of the 26th IEEE Conference on Computer Vision
  and Pattern Recognition (CVPR'13)}, pp.\  708--715, Portland, OR, 2013.

\bibitem[Zhang \& Yu(2015)Zhang and Yu]{zhang2015solving}
Zhang, M. and Yu, F.
\newblock Solving the partial label learning problem: An instance-based
  approach.
\newblock In \emph{Proceedings of the 24th International Joint Conference on
  Artificial Intelligence (IJCAI'15)}, pp.\  4048--4054, Buenos Aires,
  Argentina, 2015.

\bibitem[Zhang et~al.(2017)Zhang, Yu, and Tang]{zhang2017disambiguation}
Zhang, M., Yu, F., and Tang, C.
\newblock Disambiguation-free partial label learning.
\newblock \emph{IEEE Transactions on Knowledge and Data Engineering},
  29\penalty0 (10):\penalty0 2155--2167, 2017.

\bibitem[Zhou(2017)]{zhou2017brief}
Zhou, Z.
\newblock A brief introduction to weakly supervised learning.
\newblock \emph{National Science Review}, 5\penalty0 (1):\penalty0 44--53,
  2017.

\end{thebibliography}
\bibliographystyle{icml2020}

\newpage

\renewcommand\thesection{\Alph{section}}
\setcounter{section}{0}
\onecolumn
\icmltitle{Supplementary Material}

\section{Proof of Lemma~\ref{lemma:identify}}
\textbf{Cross-Entropy loss} 
\quad According to \cite{masnadi2009design}, since the $\ell_{\rm CE}$ is non-negative, minimizing the conditional risk $\EE_{p(y|x)}[\ell_{\rm CE}(\bm g(X),Y)|X], \forall X\in\X$ is an alternative of minimizing $\R(\bm g)$. The conditional risk can be written as
\begin{equation*}
\C(\bm g) = -\sum_{i=1}^{c}p(Y=i|X)\log(g_i(X)), \quad {\rm s.t.}\ \sum_{i=1}^{c}g_i(X)=1.
\end{equation*}
By the Lagrange Multiplier method \cite{bertsekas1997nonlinear}, we have
\begin{equation*}
\L = -\sum_{i=1}^{c}p(Y=i|X)\log(g_i(X)) + \lambda(\sum_{i=1}^{c}g_i(X)-1).
\end{equation*}
To minimize $\L$, we take the partial derivative of $\L$ with respect to $g_i$ and set it be $0$:
\begin{equation*}
g^*_i(X) = \frac{1}{\lambda} p(Y=i|X).
\end{equation*}
Because $\sum_{i=1}^{c}g^*_i(X)=1$ and $\sum_{i=1}^{c}g^*_i(X)=1$, we have
\begin{equation*}
\sum_{i=1}^{c}g^*_i(X)=\frac{1}{\lambda}\sum_{i=1}^{c}p(Y=i|X)=1.
\end{equation*}
Therefore, we can obtain $\lambda=1$ that ensures $g^*_i(X) = p(Y=i|X), \forall i\in [c], \forall X\in\X$, which concludes the proof. 

\textbf{Mean squared error loss} 
\quad Analogously, if the mean squared error loss is used, we can write the optimization problem as
\begin{equation*}
\C(\bm g) = \sum_{i=1}^{c}(p(Y=i|X)-g_i(X))^2, \quad {\rm s.t.}\ \sum_{i=1}^{c}g_i(X)=1.
\end{equation*}
By the Lagrange Multiplier method, we have
\begin{equation*}
\L = \sum_{i=1}^{c}(p(Y=i|X)-g_i(X))^2 - \lambda'(\sum_{i=1}^{c}g_i(X)-1).
\end{equation*}
By setting the derivative to 0, we obtain
\begin{equation*}
g^*_i(X) = \frac{\lambda'}{2}+p(Y=i|X).
\end{equation*}
Because $\sum_{i=1}^{c}g^*_i(X)=1$ and $\sum_{i=1}^{c}g^*_i(X)=1$, we have
\begin{equation*}
\sum_{i=1}^{c}g^*_i(X)=\frac{\lambda'c}{2}+\sum_{i=1}^{c}p(Y=i|X).
\end{equation*}
Since $c\neq0$, we can obtain $\lambda'=0$. In this way, $g^*_i(X)=p(Y=i|X), \forall i\in [c], \forall X\in\X$, which concludes the proof. \qed

\section{Proof of Theorem~\ref{thm:consistent}}

First we prove \bm $\bm g^*$ is the optimal classifier for PLL by substituting the $\bm g^*$ into the PLL risk estimator Eq.~(\ref{eq:risk-pll}):
\begin{align*}
\R_{\rm PLL}(\bm g^*) &= \EE_{(X,S)\sim p(x,s)}\min_{i\in S}\ell(\bm g^*(X), \bm{e}^i) = \int \sum_{S\in\S}\min_{i\in S}\ell(\bm g^*(X), \bm{e}^i) p(s|x)p(x)dX\\
&= \int \sum_{S\in\S}\min_{i\in S}\ell(\bm g^*(X), \bm{e}^i)\sum_{Y \in\Y}p(s,y|x)p(x)dX\\
&= \int \sum_{Y \in\Y}\sum_{S\in\S}\min_{i\in S}\ell(\bm g^*(X), \bm{e}^i)p(s|x,y)p(y|x)p(x)dX\\
&= \int \sum_{Y\in\Y}\sum_{S\in\S}\ell(\bm g^*(X), \bm{e}^{Y_X})p(s|x,y)p(y|x)p(x)dX \\
&= \int \sum_{Y\in\Y}\ell(\bm g^*(X), \bm{e}^{Y_X}) \sum_{S\in\S}p(s|x,y)p(y|x)p(x)dX\\
&=  \int \sum_{Y\in\Y}\ell(\bm g^*(X), \bm{e}^{Y_X})p(x,y)dX = \R(\bm g^*) = 0.
\end{align*}
where we have used $\min_{i\in S}\ell(\bm g^*(X), \bm{e}^i)=\ell(\bm g^*(X), \bm{e}^{Y_X})$ because $\ell$ is a proper loss and the derministic assumption is made. This indicates that the PLL risk has been minimized by $\bm g^*$.

On the other hand, we prove $\bm g^*$ is the only solution to Eq.~(\ref{eq:risk-pll}) by contradiction, namely, there is at least one other solution $\bm h$ enables $\R_{\rm PLL}(\bm h)=0$, and predicts different label $Y^{\bm h} \neq Y_X$ for at least one instance $X$. Hence for any $S\ni Y_X$ we have
\begin{equation*}
\min_{i\in S}\ell(\bm h(X), \bm{e}^i)=\ell(\bm h(X), \bm{e}^{Y^{\bm h}})=0.
\end{equation*} 
Nevertheless, the above equality is always true unless $Y^{\bm h}$ is invariably included in the candidate label set of $X$, i.e., ${\rm Pr}_{S \sim p(s|x,y)}(Y^{\bm h}\in S)=1$. 
Obviously, this contradicts the small ambiguity degree condition. Therefore, there is one, and only one minimizer of the PLL risk estimator, which is the same as the minimizer learned from ordinarily labeled data. The proof is complete. \qed

\section{Proof of Theorem~\ref{thm:errorbound}}

First, we show the uniform deviation bound, which is useful to derive the estimation error bound.
\begin{lemma}\label{lem3}
	For any $\delta>0$, we have with probability at least $1-\delta$,
	\begin{equation*}
	\sup\nolimits_{\bm g\in\G}\Big|\R_{\rm PLL}(\bm g) - \widehat\R_{\rm PLL}(\bm g)\Big| \le 2\Ra_n(\ell_{\rm PLL}\circ\G) + M\sqrt{\frac{\log (2/\delta)}{2n}}
	\end{equation*}
\end{lemma}
\begin{proof}
	Consider the one-side uniform deviation $\sup\nolimits_{\bm g\in\G}\R_{\rm PLL}(\bm g) - \widehat\R_{\rm PLL}(\bm g)$. Since the loss function $\ell$ is upper-bounded by $M$, the
	change of it will be no more than $M/n$ after replacing some $x$. Then, by \emph{McDiarmid's inequality} \cite{mcdiarmid1989method}, for any $\delta>0$, with probability at least $1-\delta/2$, the following holds:
	\begin{equation*}
	\sup\nolimits_{\bm g\in\G}\R_{\rm PLL}(\bm g)-\widehat\R_{\rm PLL}(\bm g) \le \EE\Big[\sup\nolimits_{\bm g\in\G}\R_{\rm PLL}(\bm g)-\widehat\R_{\rm PLL}(\bm g)\Big] + M\sqrt{\frac{\log(2/\delta)}{2n}}.
	\end{equation*}
	By \emph{symmetrization} \cite{vapnik1998statistical}, it is a routine work to show that
	\begin{equation*}
	\EE\Big[\sup\nolimits_{\bm g\in\G}\R_{\rm PLL}(\bm g)-\widehat\R_{\rm PLL}(\bm g)\Big] \le 2\Ra_n(\ell_{\rm PLL}\circ\G).
	\end{equation*}
	The one-side uniform deviation $\sup\nolimits_{\bm g\in\G}\widehat\R_{\rm PLL}(\bm g)-\R_{\rm PLL}(\bm g)$ can be bounded similarly.
\end{proof}

Then we upper bound $\Ra_n(\ell_{\rm PLL}\circ\G)$.
\begin{lemma}\label{lem4}
	Suppose $\ell_{\rm PLL}$ is defined as Eq.~(\ref{eq:loss-pll}), it holds that
	\begin{equation*}
	\Ra_n (\ell_{\rm PLL}  \circ \G) \le c\Ra_n (\ell \circ \G) \le \sqrt{2}cL_{\ell}\sum_{y=1}^c \Ra_n(\G_y).
	\end{equation*}	
\end{lemma}
\begin{proof}
	By definition of $\ell_{\rm PLL}$, $\ell_{\rm PLL} \circ \G(x_i,S_i) = \min_{y\in S} \ell\circ \G(x_i, y) = \min_{y\in [c]} \ell\circ \G(x_i, y)$. Given sample sized $n$, we first prove the result in the case $c=2$. The $\min$ operator can be written as
	\begin{equation*}
	\min\{z_1,z_2\} = \frac{1}{2} \big[z_1+z_2-|z_1-z_2|\big].
	\end{equation*}
	In this way, we can write
	\begin{align}\label{eq:rademacher1}
	\Ra_n (\ell_{\rm PLL} \circ \G) &= \EE_{\bm\sigma}  \left[\sup_{\bm g\in\G} \frac{1}{n}\sum_{i=1}^n \sigma_i \ell(\bm g(x_i),s_i))\right]\\\nonumber
	&= \EE_{\bm\sigma} \left[\sup_{\bm g\in\G} \frac{1}{n}\sum_{i=1}^n \sigma_i \min\{\ell(\bm g(x_i),y_1),\ell(\bm g(x_i),y_2)\}\right]\\\nonumber
	&= \EE_{\bm\sigma} \left[\sup_{\bm g\in\G} \frac{1}{2n}\sum_{i=1}^n \sigma_i \Big[\ell(\bm g(x_i),y_1)+\ell(\bm g(x_i),y_2)-|\ell(\bm g(x_i),y_1)-\ell(\bm g(x_i),y_2)|\Big]\right]\\\nonumber
	&\le \EE_{\bm\sigma} \left[\sup_{\bm g\in\G} \frac{1}{2n}\sum_{i=1}^n \sigma_i \ell(\bm g(x_i),y_1)\right]+\EE_{\bm\sigma} \left[\sup_{\bm g\in\G} \frac{1}{2n}\sum_{i=1}^n \sigma_i \ell(\bm g(x_i),y_2)\right]\\\nonumber
	&+ \EE_{\bm\sigma} \left[\sup_{\bm g\in\G} \frac{1}{2n}\sum_{i=1}^n \sigma_i\Big|\ell(\bm g(x_i),y_1)-\ell(\bm g(x_i),y_2)\Big|\right]\\\nonumber
	&=\frac{1}{2}\Big(\Ra_n(\ell \circ \G)+\Ra_n(\ell \circ \G)\Big)+\EE_{\bm\sigma} \left[\sup_{\bm g\in\G} \frac{1}{2n}\sum_{i=1}^n \sigma_i\Big|\ell(\bm g(x_i),y_1)-\ell(\bm g(x_i),y_2))\Big|\right].
	\end{align}
	Since $x \mapsto |x|$ is a 1-Lipschitz function, by \emph{Talagrand's contraction lemma} \cite{ledoux2013probability}, the last term can be bounded:
	\begin{align}\label{eq:rademacher2}
	&\EE_{\bm\sigma} \left[\sup_{\bm g\in\G} \frac{1}{2n}\sum_{i=1}^n \sigma_i\Big|\ell(\bm g(x_i),y_1)-\ell(\bm g(x_i),y_2))\Big|\right]\\\nonumber
	\le& \EE_{\bm\sigma} \left[\sup_{\bm g\in\G} \frac{1}{2n}\sum_{i=1}^n \sigma_i \Big(\ell(\bm g(x_i),y_1)-\ell(\bm g(x_i),y_2))\Big)\right] \le \frac{1}{2}\Big(\Ra_n(\ell \circ \G)+\Ra_n(\ell\circ\G)\Big).
	\end{align}
	
	Combining Eq.~(\ref{eq:rademacher1}) and Eq.~(\ref{eq:rademacher2}) yields $\Ra_n(\ell_{\rm PLL} \circ\G) \le \Ra_n(\ell \circ \G)+\Ra_n(\ell \circ\G)$. The general case can be derived from the case $c=2$ using $\min\{z_1,\ldots,z_c\}=\min\{z_1,\min\{z_2,\ldots,z_c\}\}$ and an immediate recurrence. 
	
	Then we apply the Rademacher vector contraction inequality \cite{maurer2016vector},
	\begin{equation*}
	\Ra_n (\ell\circ\G) \leq \sqrt{2}L_{\ell}\sum_{y=1}^c \Ra_n(\G_y).
	\end{equation*}
	The proof is completed.
\end{proof}

Based on Lemma~\ref{lem3} and~\ref{lem4}, the estimation error bound Eq.~(\ref{eq:errorbound}) is proven through
\begin{align*}
\R_{\rm PLL}(\widehat{\bm g}_{\rm PLL})-\R_{\rm PLL}(\bm g_{\rm PLL}^*) &=\left (\R_{\rm PLL}(\widehat{\bm g}_{\rm PLL})-\widehat\R_{\rm PLL}(\widehat{\bm g}_{\rm PLL})\right)+\left(\widehat\R_{\rm PLL}(\widehat{\bm g}_{\rm PLL})-\widehat\R_{\rm PLL}(\bm g_{\rm PLL}^*)\right)\\
&+\left(\widehat\R_{\rm PLL}(\bm g_{\rm PLL}^*)-\R_{\rm PLL}(\bm g^*)\right)\\
&\le \left (\R_{\rm PLL}(\widehat{\bm g}_{\rm PLL})-\widehat\R_{\rm PLL}(\widehat{\bm g}_{\rm PLL})\right)+\left(\widehat\R_{\rm PLL}(\bm g_{\rm PLL}^*)-\R_{\rm PLL}(\bm g_{\rm PLL}^*)\right)\\
&\le 2\sup_{\bm g\in \G}\Big|\R_{\rm PLL}(\bm g)- \widehat\R_{\rm PLL}(\bm g)\Big|\\
&\le 4\sqrt{2}cL_{\ell}\sum_{y=1}^c \Ra_n(\G_y) + 2M\sqrt{\frac{\log (2/\delta)}{2n}}.
\end{align*}\qed

\section{Supplementary Theorem on Section~\ref{sec:method}}

\begin{theorem}
	The learning objective in \citet{jin2003learning} is a special case of Eq.~(\ref{eq:softmin}).
\end{theorem}

\begin{proof}
	Recall the learning objective in \citet{jin2003learning} is formulated as:
	\begin{equation}\label{eq:jin-obj}
	\widehat\R_{\rm PLL} = \frac{1}{n} \sum_{i=1}^n {\rm KL}\big[\bm z_{i} ||\bm g(x_i) \big] = \frac{1}{n} \sum_{i=1}^n \bm z_{i} \log\frac{\bm z_{i}}{\bm g(x_i)},
	\end{equation}
	where KL divergence is used and $\bm z_i$ represents the prior probability of $x_i$.
	
	Then in Eq.~(\ref{eq:softmin}), the loss function can be specified as the cross-entropy loss: $\ell_{\rm CE}(g_j(x_i), e_j^{s_i})=-e_j^{s_i}\log(g_j(x_i))$, which is linear in the second term, i.e.,
	\begin{equation*}
	\frac{1}{n}\sum_{i=1}^n\sum_{j=1}^c {\rm w}_{ij}\ell_{\rm CE}(g_j(x_i), e_j^{s_i}) = \frac{1}{n}\sum_{i=1}^n\sum_{j=1}^c \ell_{\rm CE}(g_j(x_i), {\rm w}_{ij}e_j^{s_i}).
	\end{equation*} 
	Thus, the weights ${\bf w}$ can be moved into the loss function and yields:
	\begin{equation}\label{eq:softmin-ce}
	\widehat\R_{\rm PLL} = \frac{1}{n} \sum_{i=1}^n\sum_{j=1}^c \ell_{\rm CE}( g_j(x_i), z_{ij}) = -\frac{1}{n} \sum_{i=1}^n \bm z_{i} \log\bm g(x_i),
	\end{equation} 
	where $z_{ij} = {\rm w}_{ij}\II(j \in s_i)$. The optimal $\bm g^*$ of Eq.~(\ref{eq:softmin-ce}) is essentially equivalent to $\bm g^*$ learned from Eq.~(\ref{eq:jin-obj}). Therefore, our method is a strict extension of \cite{jin2003learning}.
\end{proof}

\section{Benchmark Datasets}
\subsection{Setup}

\begin{table*}[t]
	\caption{Summary of benchmark datasets and models.}
	\label{tab:benchmarks}
	\vskip 0.15in
	\begin{center}	
		\renewcommand\arraystretch{1.2}
		\begin{small}
			\setlength{\tabcolsep}{3.5mm}{
				\begin{tabular}{c|cccc|c}
					\toprule
					Dataset & \# Train & \# Test & \# Feature & \# Class & Model $\bm g(x;\Theta)$ \\
					\midrule
					MNIST & 60,000 & 10,000 & 784 & 10 & Linear model, MLP (depth 5)\\
					Fashion-MNIST & 60,000 & 10,000 & 784 & 10 & Linear model, MLP (depth 5) \\
					Kuzushiji-MNIST & 60,000 & 10,000 & 784 & 10 & Linear model, MLP (depth 5) \\
					CIFAR-10 & 50,000 & 10,000 & 3,072 & 10 & ConvNet \cite{laine17temporal}, ResNet \cite{he2016deep} \\
					\bottomrule
			\end{tabular}}
		\end{small}
	\end{center}
	\vskip -0.1in
\end{table*}

Tabel~\ref{tab:benchmarks} describes the benchmark datasets and the corresponding models of them.

\textbf{MNIST}
\quad This is a grayscale image dataset of handwritten digits from 0 to 9 where the size of the images is $28 \times 28$. 

The linear model is a linear-in-input model: $d$-10, and MLP refers to a 5-layer FC with ReLU as the activation function: $d$-300-300-300-300-10. Batch normalization \cite{ioffe2015batch} was applied before hidden layers. For both models, the softmax function was applied to the output layer, and $\ell_2$-regularization was added. The two models were trained by SGD with the default momentum parameter ($\beta = 0.9$), and the batch size was set to 256.

\textbf{Fashion-MNIST}
\quad This is a grayscale image dataset similarly to MNIST. In Fashion-MNIST, each instance is a $28 \times 28$ grayscale image and associated with a label from 10 fashion item classes.
The models and optimizer were the same as MNIST.

\textbf{Kuzushiji-MNIST}
\quad
This is another grayscale image dataset similarly to MNIST. In Kuzushiji-MNIST, each instance is a $28 \times 28$ grayscale image and associated with a label from 10 cursive Japanese (Kuzushiji) characters. 
The models and optimizer were the same as MNIST.

\textbf{CIFAR-10}
\quad
This dataset consists of 60,000 $32 \times 32 \times 3$ colored image in RGB format in 10 classes.

The detailed architecture of ConvNet \cite{laine17temporal} is as follows.

\quad\ \ \ \ \   0th (input) layer: \ \ (32*32*3)-
{\setlength{\parskip}{0em}
	
	\quad\ \ \ \ \  1st to 4th layers: \ \ [C(3*3, 128)]*3-Max Pooling-
	
	\quad\ \ \ \ \  5th to 8th layers: \ \ [C(3*3, 256)]*3-Max Pooling-
	
	\quad\ \ \   9th to 11th layers: \ \ C(3*3, 512)-C(3*3, 256)-C(3*3, 128)-
	
	\quad\quad\quad\quad\  12th layers: \ \ Average Pooling-10
}

where C(3*3, 128) means 128 channels of 3*3 convolutions followed by Leaky-ReLU (LReLU) active function \cite{maas2013rectifier}, [ · ]*3 means 3 such layers, etc.

The detailed architecture of ResNet \cite{he2016deep} was as follows.

\quad\ \ \ \ \   0th (input) layer: \ \ (32*32*3)-
{\setlength{\parskip}{0em}
	
	\quad\ \ \ \  1st to 11th layers: \ \ C(3*3, 16)-[C(3*3, 16), C(3*3, 16)]*5-
	
	\quad\ \  12th to 21st layers: \ \  [C(3*3, 32), C(3*3, 32)]*5-
	
	\quad\  22nd to 31st layers: \ \ [C(3*3, 64), C(3*3, 64)]*5-
	
	\quad\quad\quad\quad \ \  32nd layer: \ \ Average Pooling-10
}

where [ ·, · ] means a building block \cite{he2016deep}. These two models were trained by SGD with the default momentum parameter and the batch size was 256.

An example of a binomial flipping with $q=0.1$ and of a pair flipping with $q=0.5$ used on MNIST are below, respectively:
\begin{equation*}
\begin{scriptsize}\begin{bmatrix}
1 & 0.1 & 0.1 & 0.1 & 0.1 & 0.1 & 0.1 & 0.1 & 0.1 & 0.1 \\ 
0.1 & 1 & 0.1 & 0.1 & 0.1 & 0.1 & 0.1 & 0.1 & 0.1 & 0.1 \\ 
0.1 & 0.1 & 1 & 0.1 & 0.1 & 0.1 & 0.1 & 0.1 & 0.1 & 0.1 \\ 
0.1 & 0.1 & 0.1 & 1 & 0.1 & 0.1 & 0.1 & 0.1 & 0.1 & 0.1 \\ 
0.1 & 0.1 & 0.1 & 0.1 & 1 & 0.1 & 0.1 & 0.1 & 0.1 & 0.1 \\ 
0.1 & 0.1 & 0.1 & 0.1 & 0.1 & 1 & 0.1 & 0.1 & 0.1 & 0.1 \\ 
0.1 & 0.1 & 0.1 & 0.1 & 0.1 & 0.1 & 1 & 0.1 & 0.1 & 0.1\\ 
0.1 & 0.1 & 0.1 & 0.1 & 0.1 & 0.1 & 0.1 & 1 & 0.1 & 0.1 \\ 
0.1 & 0.1 & 0.1 & 0.1 & 0.1 & 0.1 & 0.1 & 0.1 & 1 & 0.1 \\ 
0.1 & 0.1 & 0.1 & 0.1 & 0.1 & 0.1 & 0.1 & 0.1 & 0.1 & 1 
\end{bmatrix}\end{scriptsize} \quad\quad
\begin{scriptsize}\begin{bmatrix}
1 & 0.5 & 0 & 0 & 0 & 0 & 0 & 0 & 0 & 0 \\ 
0 & 1 & 0.5 & 0 & 0 & 0 & 0 & 0 & 0 & 0 \\ 
0 & 0 & 1 & 0.5 & 0 & 0 & 0 & 0 & 0 & 0 \\ 
0 & 0 & 0 & 1 & 0.5 & 0 & 0 & 0 & 0 & 0 \\ 
0 & 0 & 0 & 0 & 1 & 0.5 & 0 & 0 & 0 & 0 \\ 
0 & 0 & 0 & 0 & 0 & 1 & 0.5 & 0 & 0 & 0 \\ 
0 & 0 & 0 & 0 & 0 & 0 & 1 & 0.5 & 0 & 0 \\ 
0 & 0 & 0 & 0 & 0 & 0 & 0 & 1 & 0.5 & 0 \\ 
0 & 0 & 0 & 0 & 0 & 0 & 0 & 0 & 1 & 0.5 \\ 
0.5 & 0 & 0 & 0 & 0 & 0 & 0 & 0 & 0 & 1 
\end{bmatrix}\end{scriptsize}
\end{equation*}

\subsection{Transductive Results}

Figure~\ref{fig:benchmark_binomial_transductive} illustrates the transductive results on the benchmark datasets, i.e., the ability in identifying the true labels in the training set. We can see that PRODEN has a strong ability to find the true labels.

\begin{figure*}[!t]
	\vskip 0.2in
	\subfigure{
		\begin{minipage}[b]{0.24\columnwidth}
			\centering
			\includegraphics[width=1.6in]{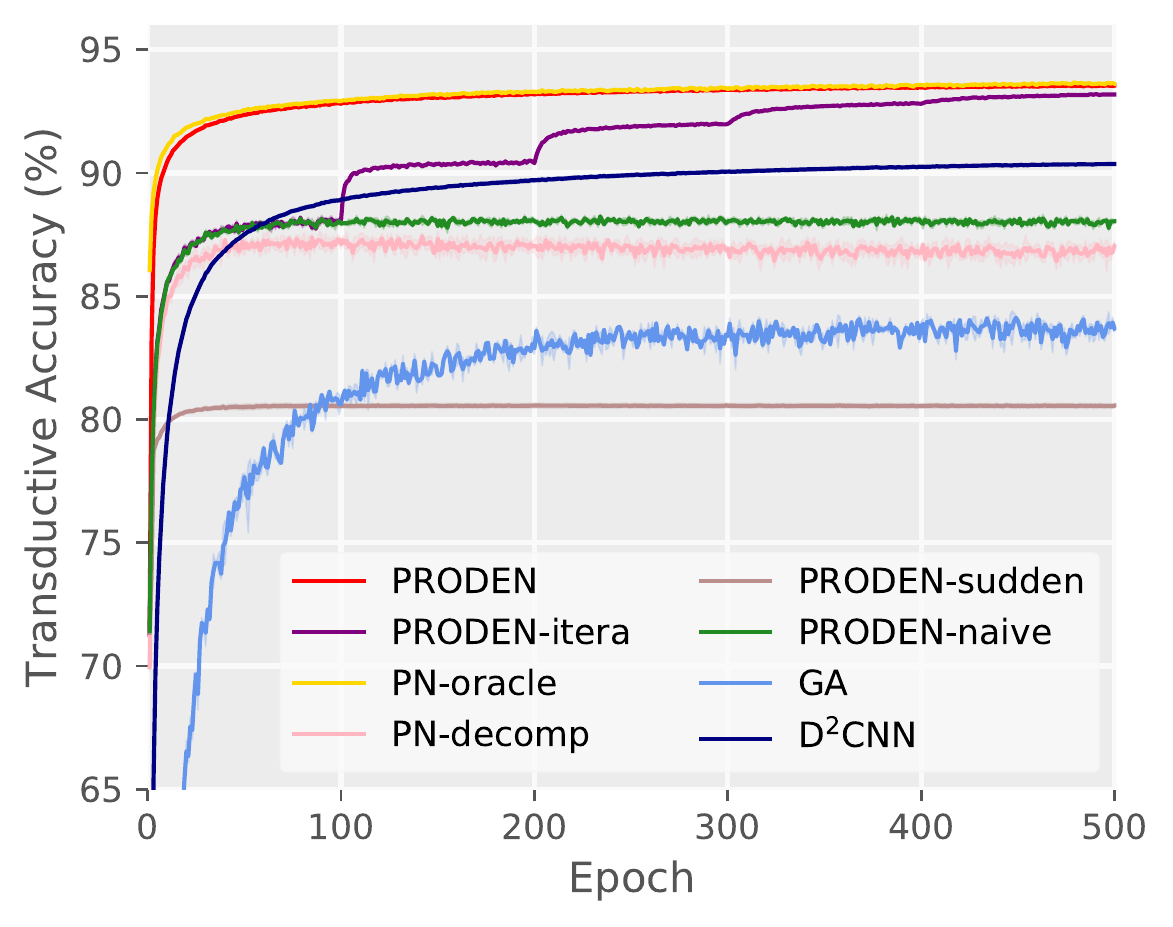}
			\centerline{MNIST, Linear, $q=0.1$}
	\end{minipage}}
	\subfigure{
		\begin{minipage}[b]{0.24\columnwidth}
			\centering
			\includegraphics[width=1.6in]{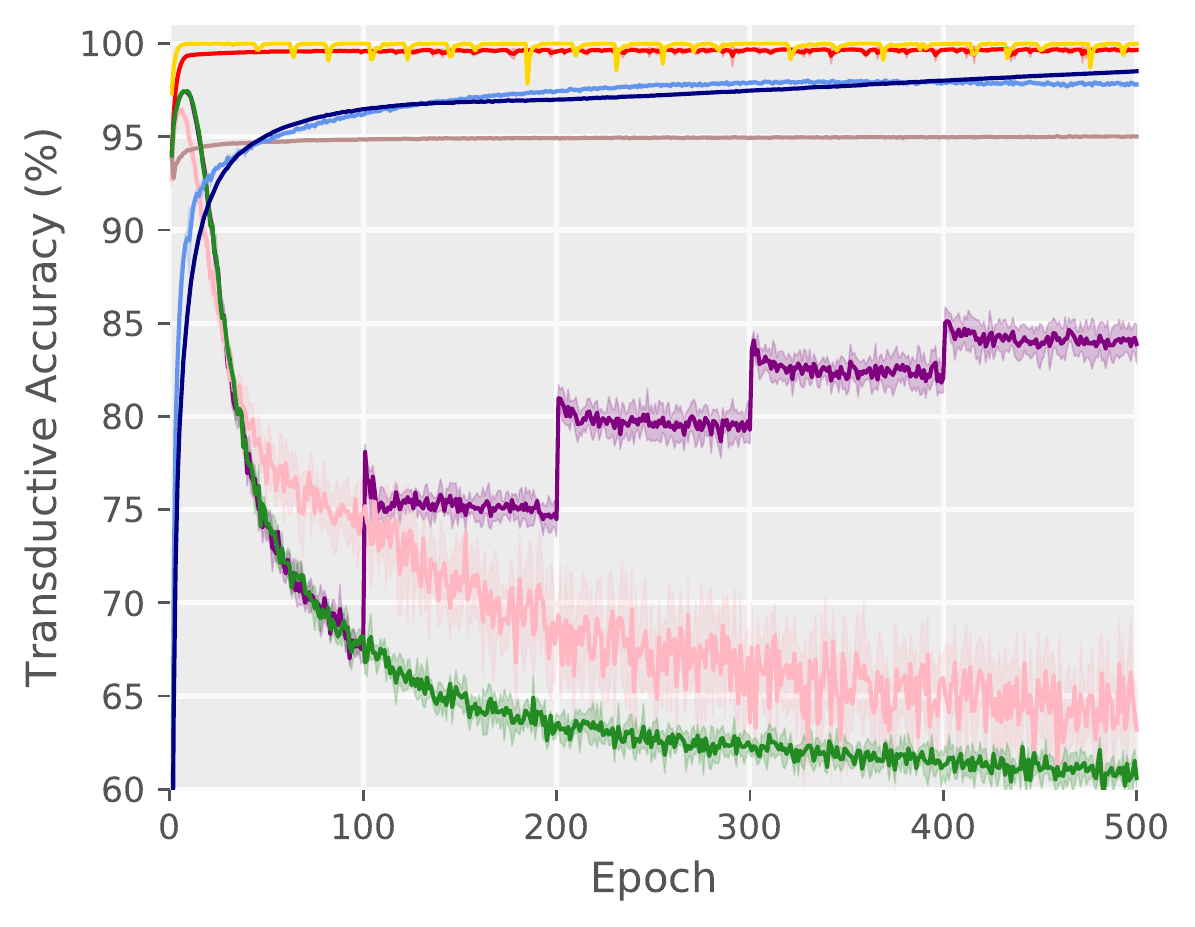}
			\centerline{\quad MNIST, MLP, $q=0.1$}
	\end{minipage}}
	\subfigure{
		\begin{minipage}[b]{0.24\columnwidth}
			\centering
			\includegraphics[width=1.6in]{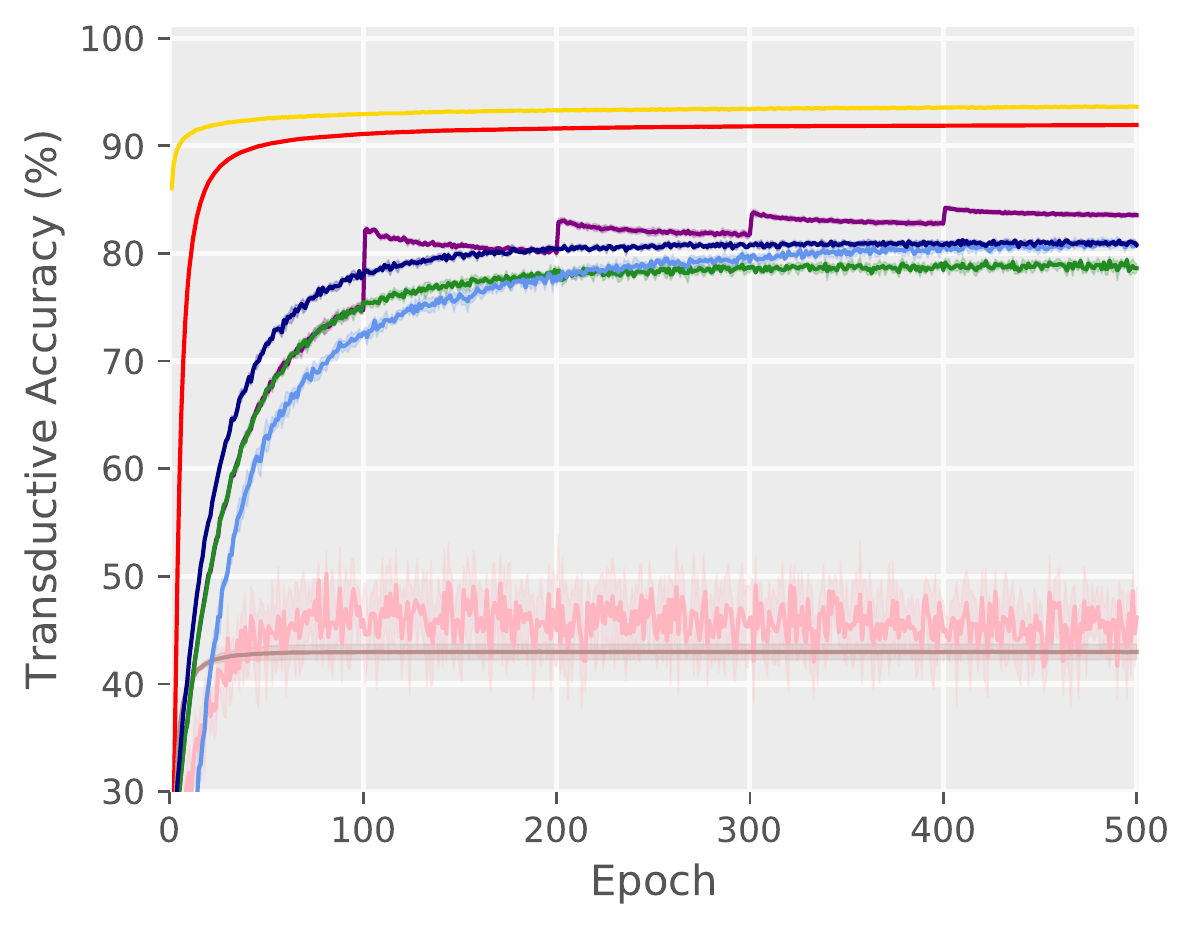}
			\centerline{\quad MNIST, Linear, $q=0.7$}
	\end{minipage}}
	\subfigure{
		\begin{minipage}[b]{0.24\columnwidth}
			\centering
			\includegraphics[width=1.6in]{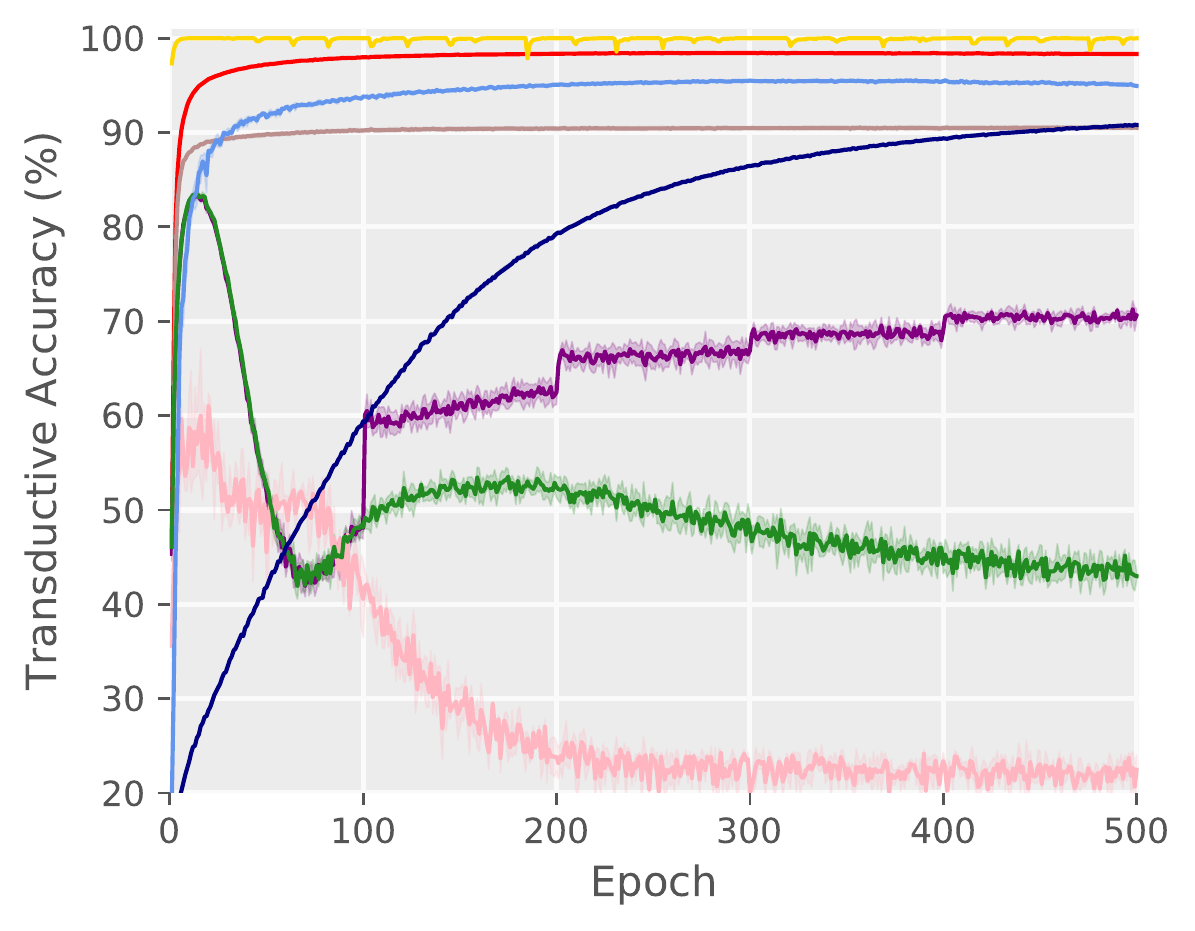}
			\centerline{\quad MNIST, MLP, $q=0.7$}
	\end{minipage}}
	
	\subfigure{
		\begin{minipage}[b]{0.24\columnwidth}
			\centering
			\includegraphics[width=1.6in]{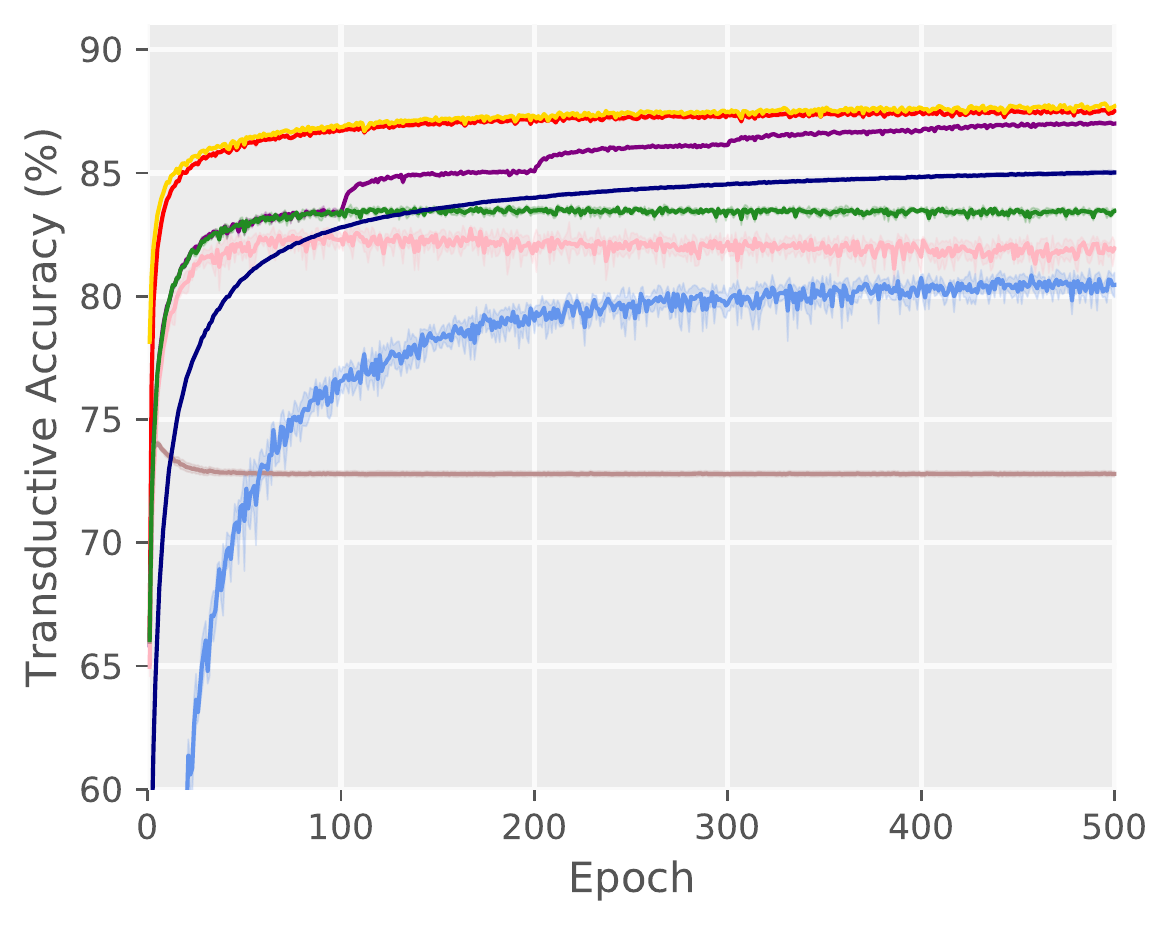}
			\centerline{Fashion, Linear, $q=0.1$}
	\end{minipage}}
	\subfigure{
		\begin{minipage}[b]{0.24\columnwidth}
			\centering
			\includegraphics[width=1.6in]{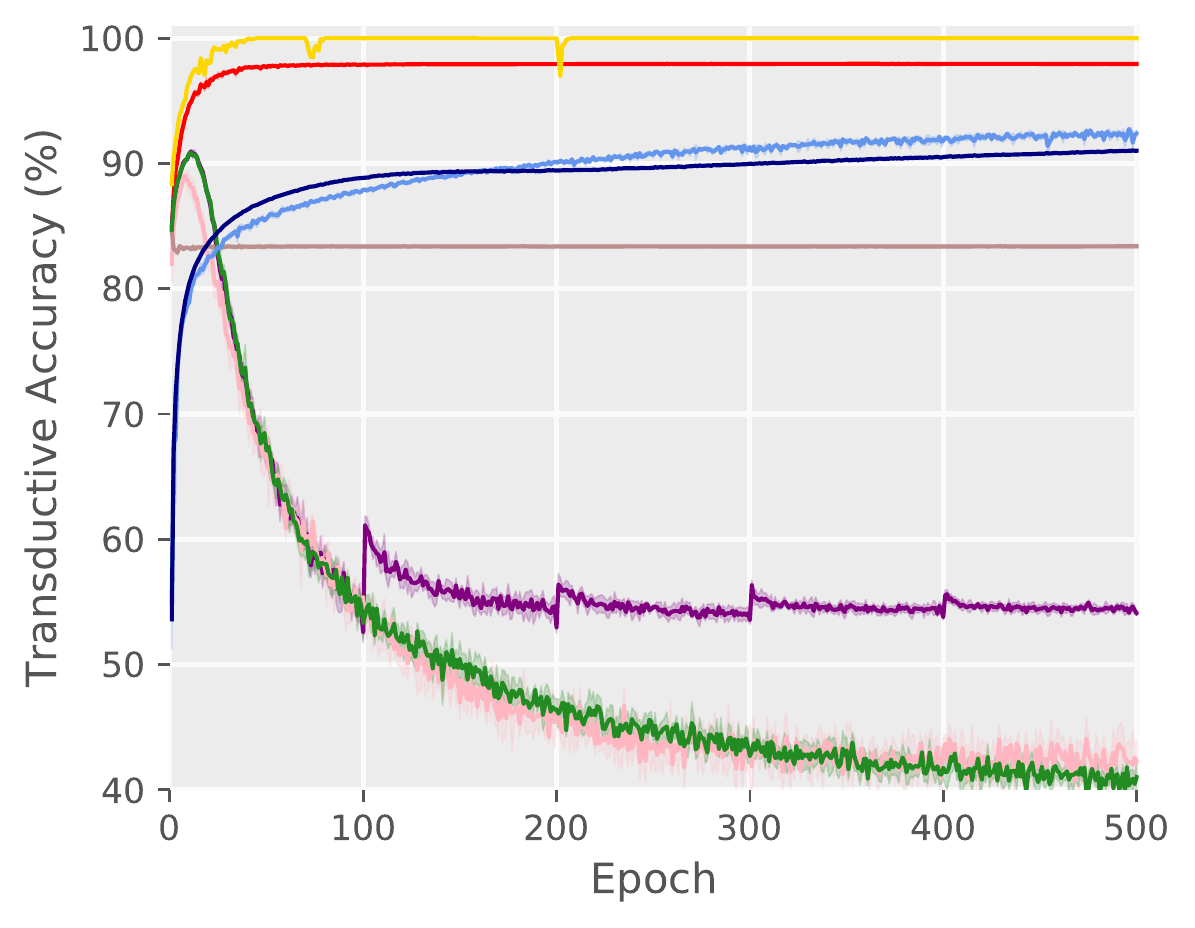}
			\centerline{\quad Fashion, MLP, $q=0.1$}
	\end{minipage}}
	\subfigure{
		\begin{minipage}[b]{0.24\columnwidth}
			\centering
			\includegraphics[width=1.6in]{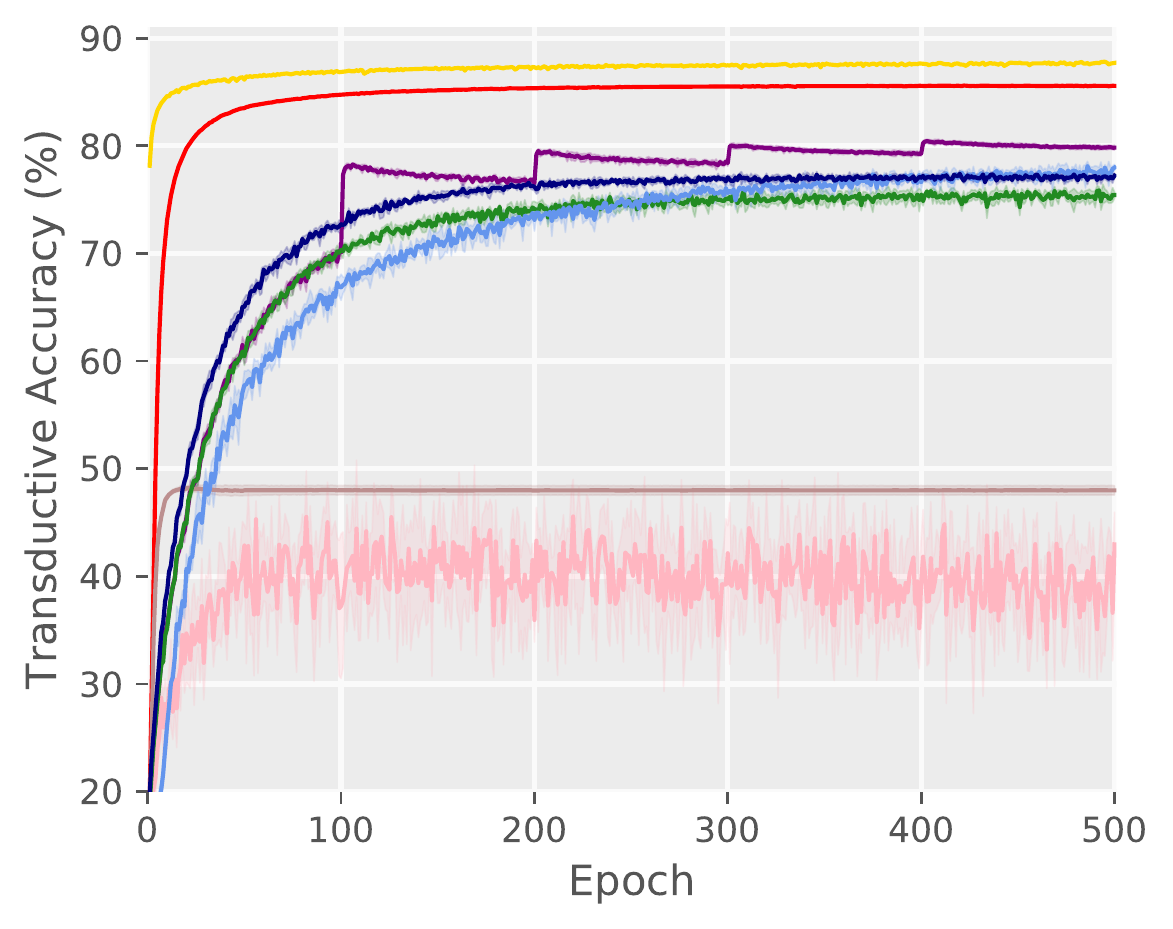}
			\centerline{\quad Fashion, Linear, $q=0.7$}
	\end{minipage}}
	\subfigure{
		\begin{minipage}[b]{0.24\columnwidth}
			\centering
			\includegraphics[width=1.6in]{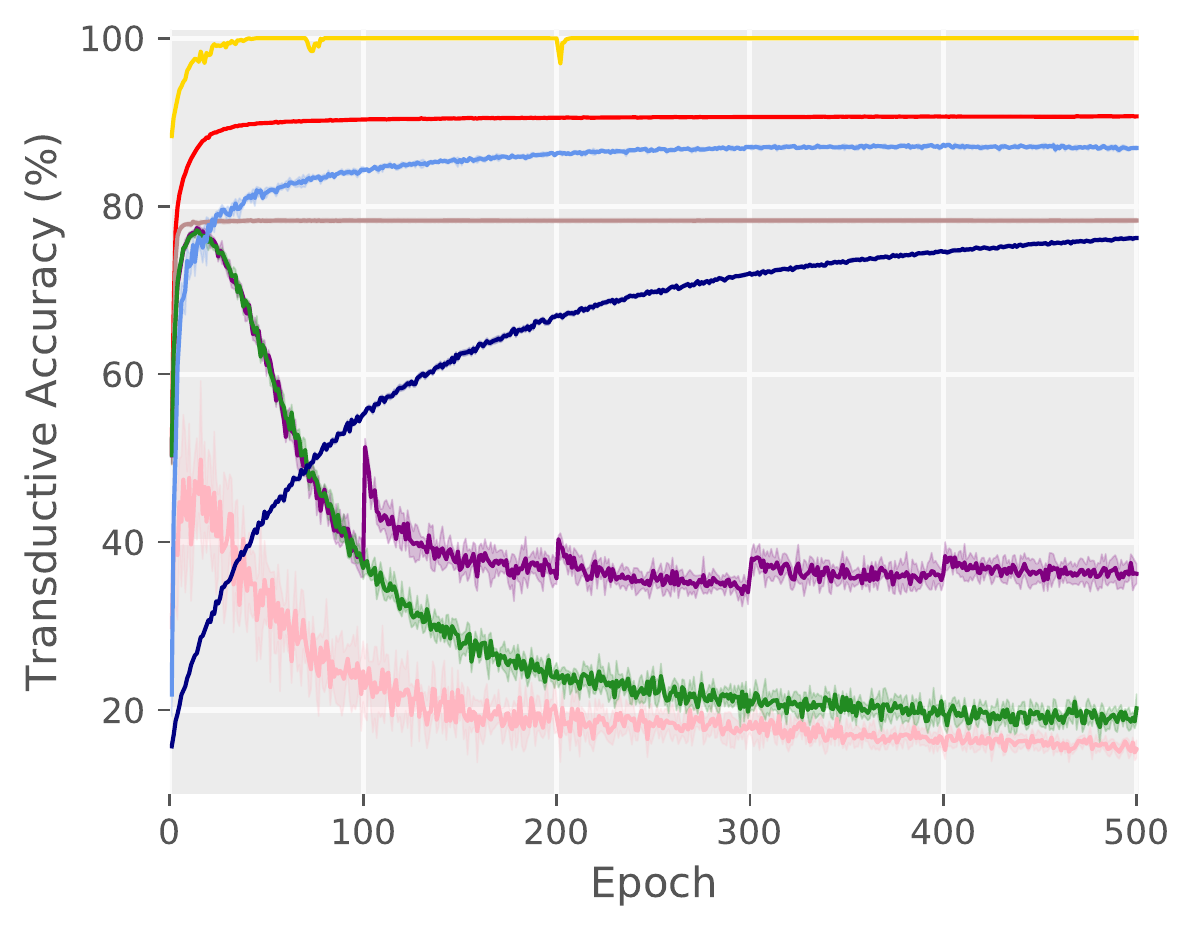}
			\centerline{\quad Fashion, MLP, $q=0.7$}
	\end{minipage}}
	
	\subfigure{
		\begin{minipage}[b]{0.24\columnwidth}
			\centering
			\includegraphics[width=1.65in]{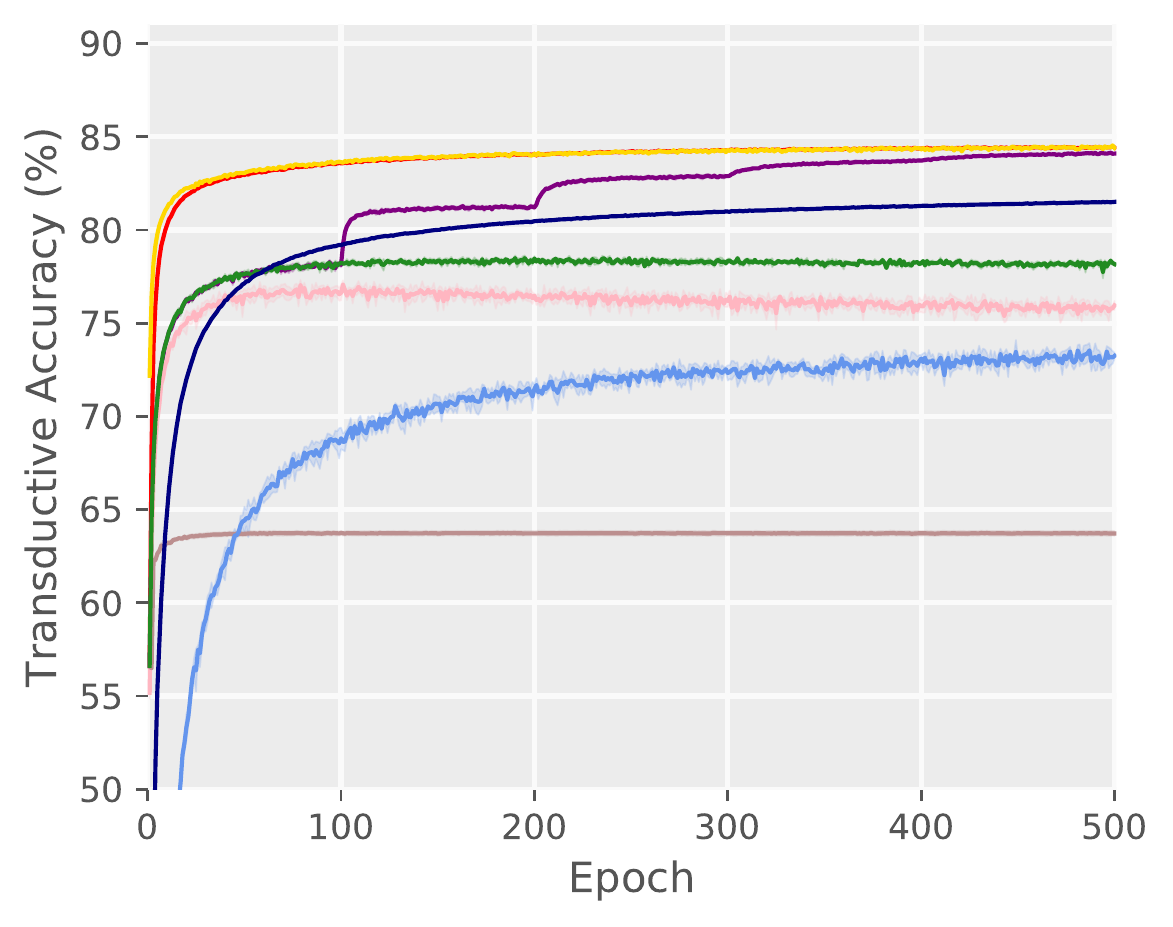}
			\centerline{Kuzushiji, Linear, $q=0.1$}
	\end{minipage}}%
	\subfigure{
		\begin{minipage}[b]{0.24\columnwidth}
			\centering
			\includegraphics[width=1.65in]{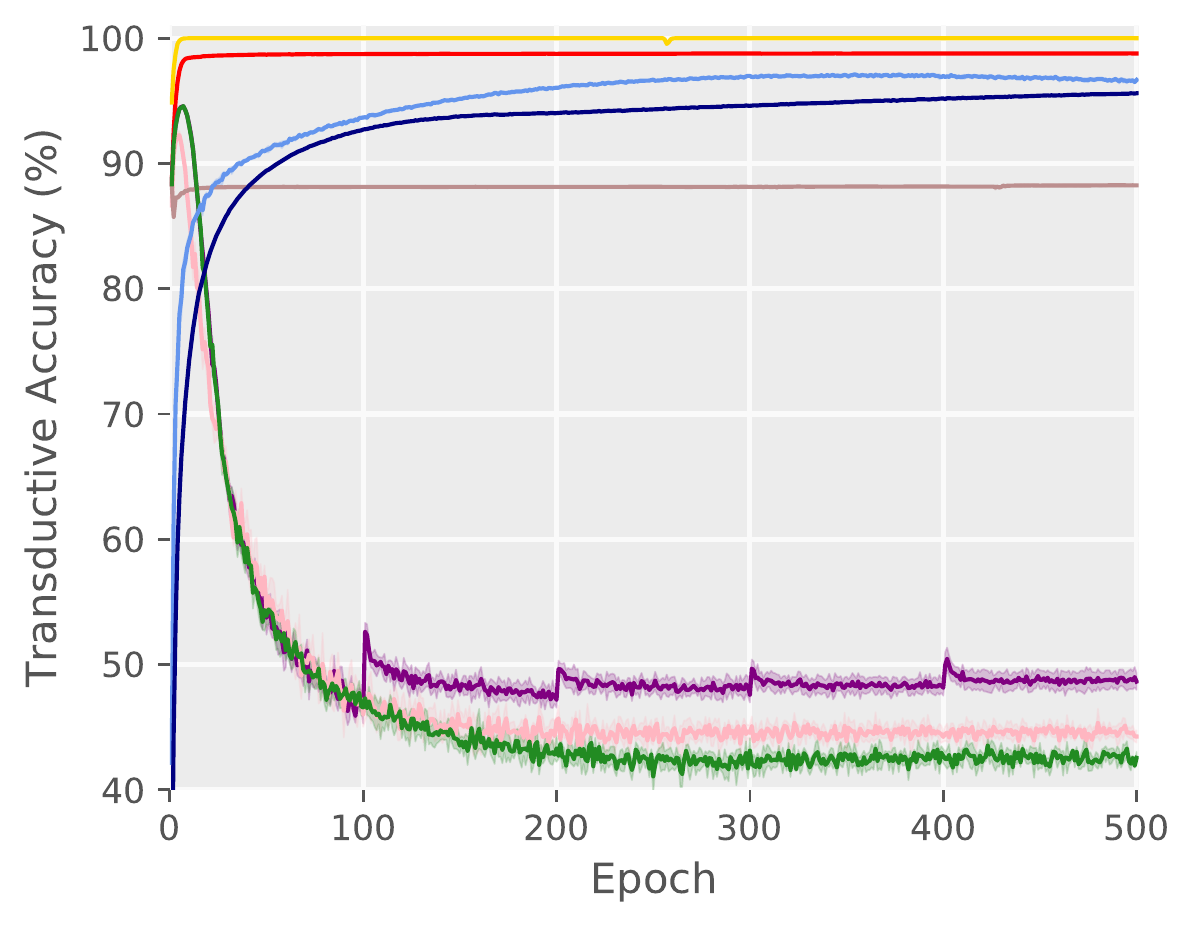}
			\centerline{\quad Kuzushiji, MLP, $q=0.1$}
	\end{minipage}}
	\subfigure{
		\begin{minipage}[b]{0.24\columnwidth}
			\centering
			\includegraphics[width=1.65in]{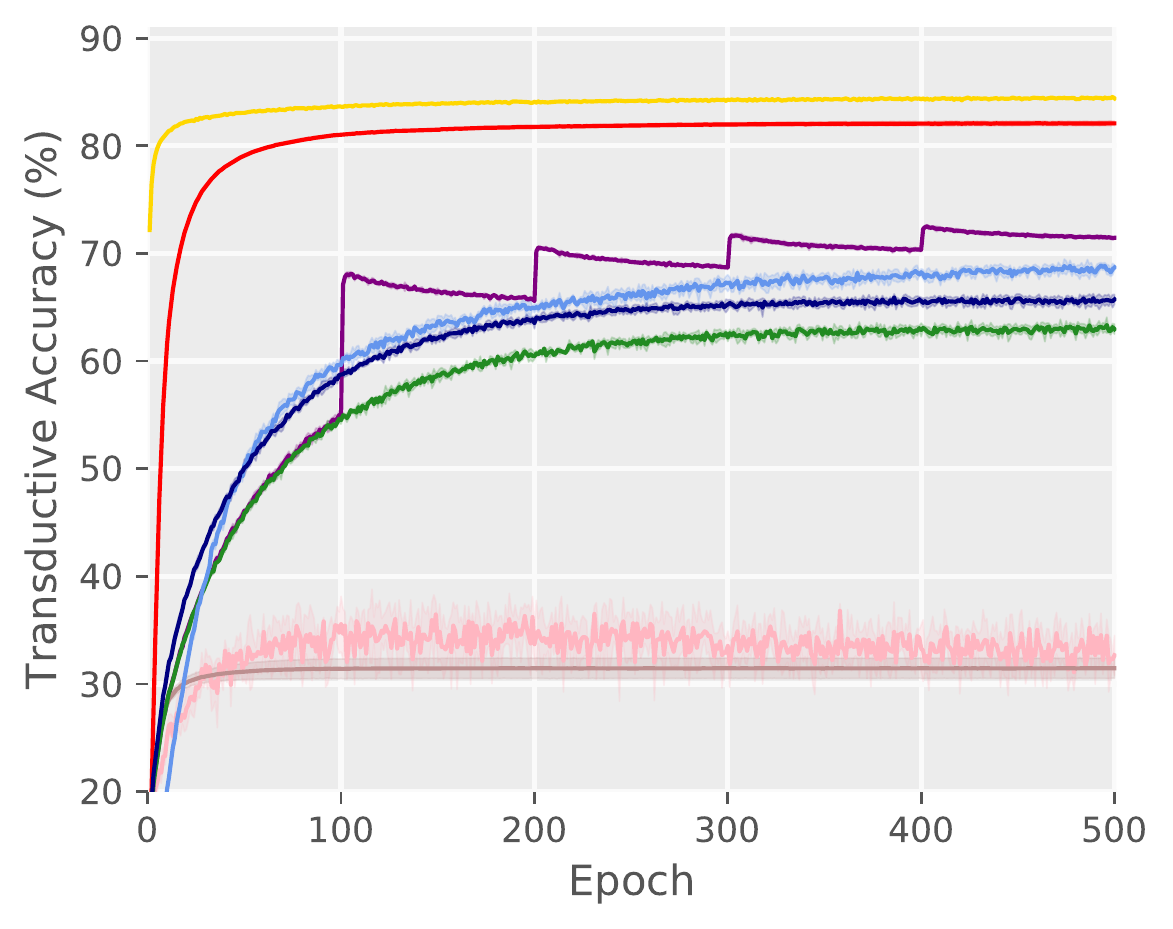}
			\centerline{\quad Kuzushiji, Linear, $q=0.7$}
	\end{minipage}}%
	\subfigure{
		\begin{minipage}[b]{0.24\columnwidth}
			\centering
			\includegraphics[width=1.65in]{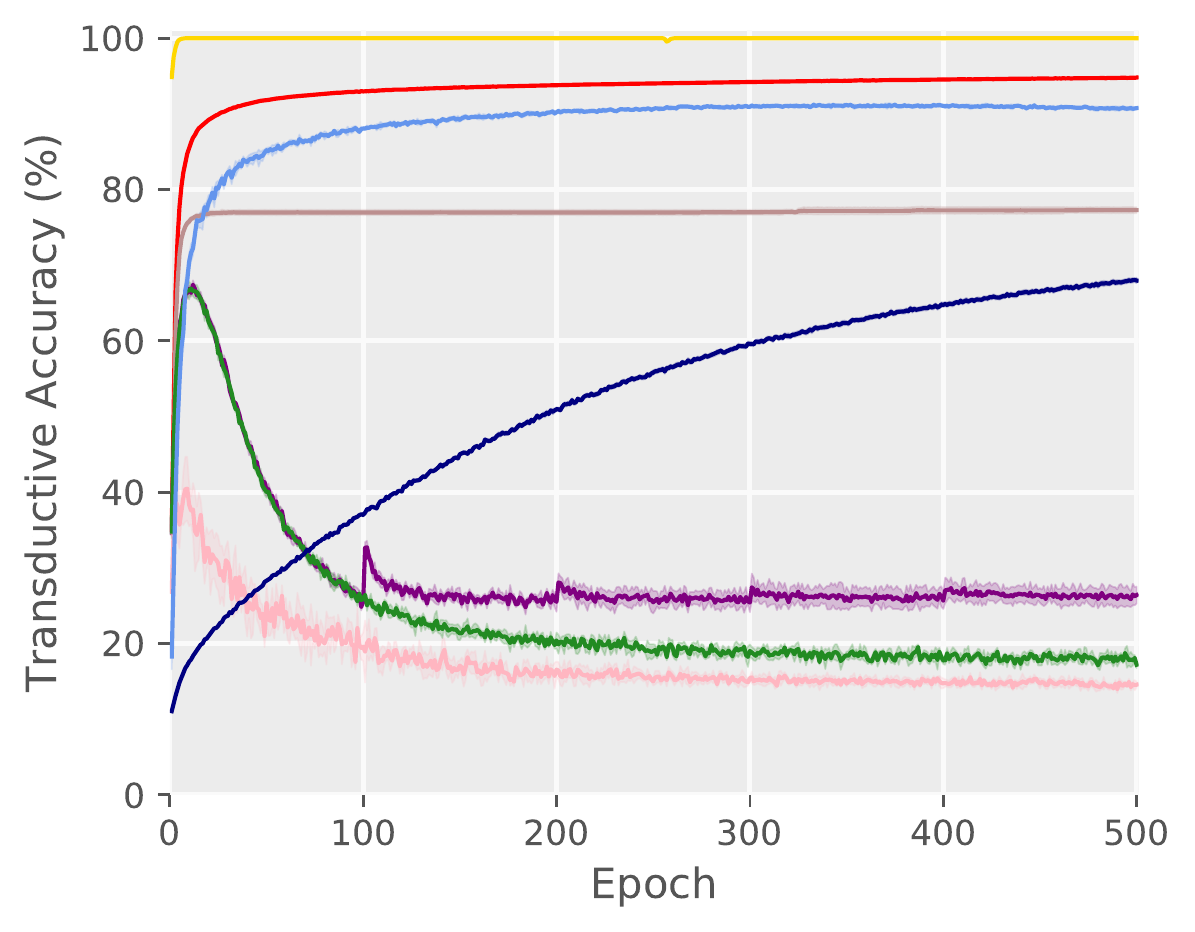}
			\centerline{\quad Kuzushiji, MLP, $q=0.7$}
	\end{minipage}}
	
	\subfigure{
		\begin{minipage}[b]{0.24\columnwidth}
			\centering
			\includegraphics[width=1.65in]{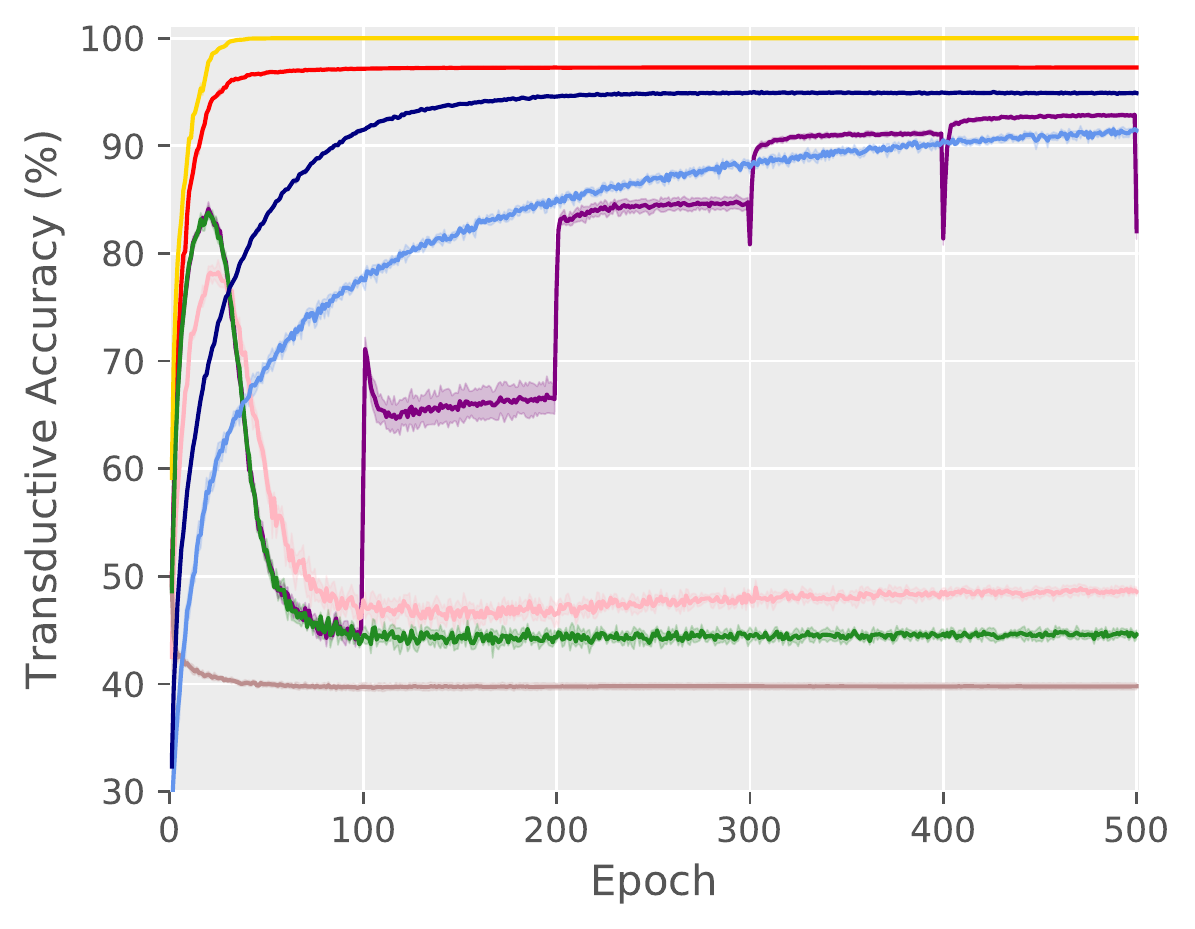}
			\centerline{CIFAR, ConvNet, $q=0.1$}
	\end{minipage}}%
	\subfigure{
		\begin{minipage}[b]{0.24\columnwidth}
			\centering
			\includegraphics[width=1.65in]{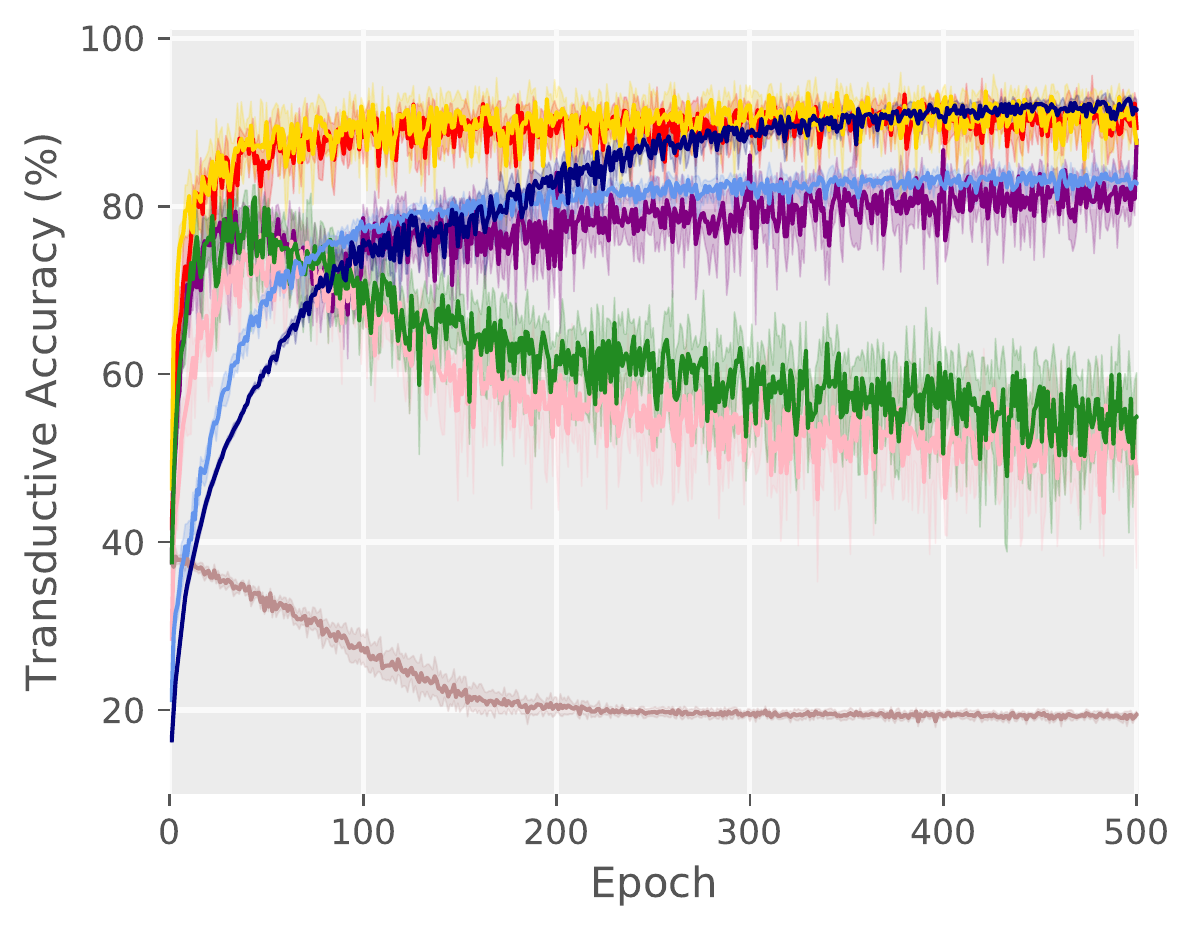}
			\centerline{\quad CIFAR, ResNet, $q=0.1$}
	\end{minipage}}
	\subfigure{
		\begin{minipage}[b]{0.24\columnwidth}
			\centering
			\includegraphics[width=1.65in]{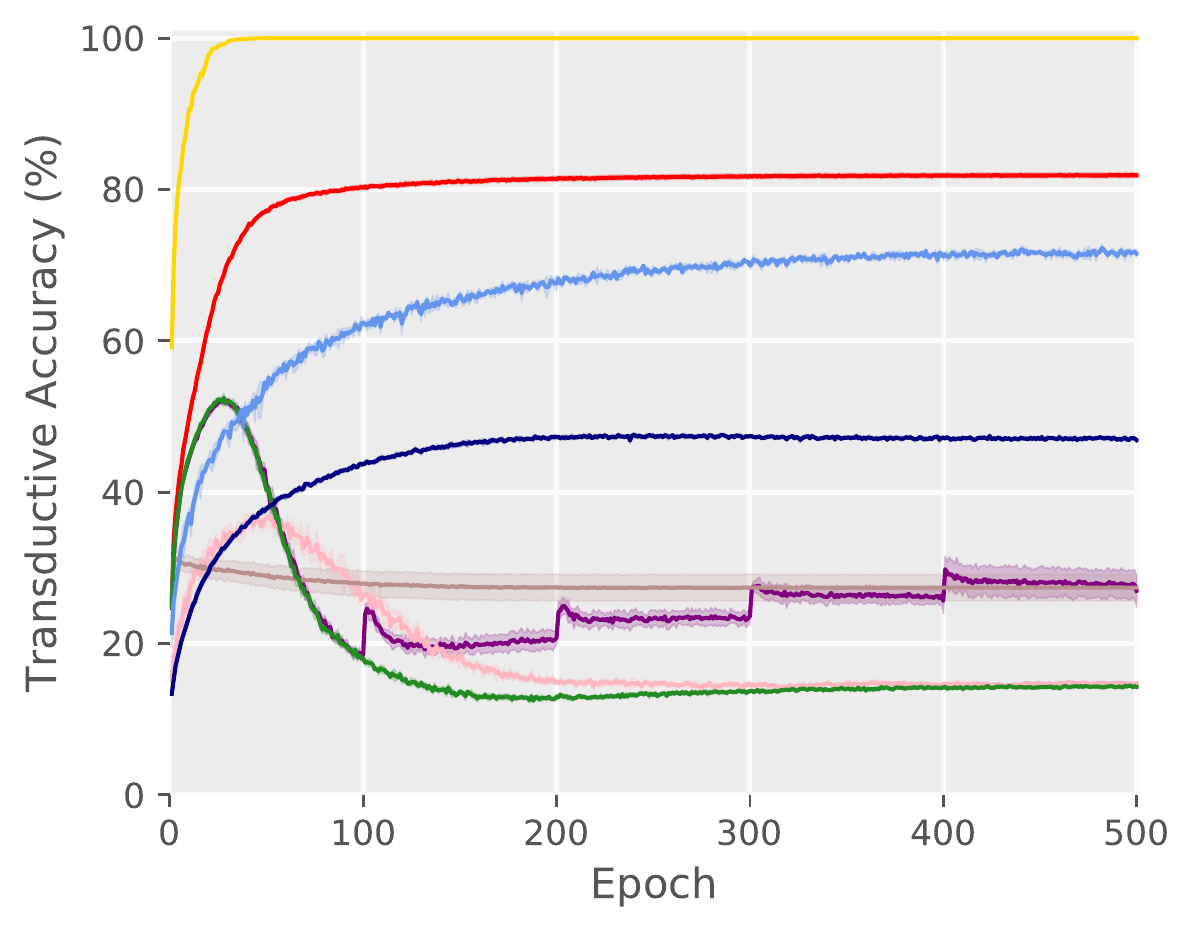}
			\centerline{\quad CIFAR, ConvNet, $q=0.7$}
	\end{minipage}}%
	\subfigure{
		\begin{minipage}[b]{0.24\columnwidth}
			\centering
			\includegraphics[width=1.65in]{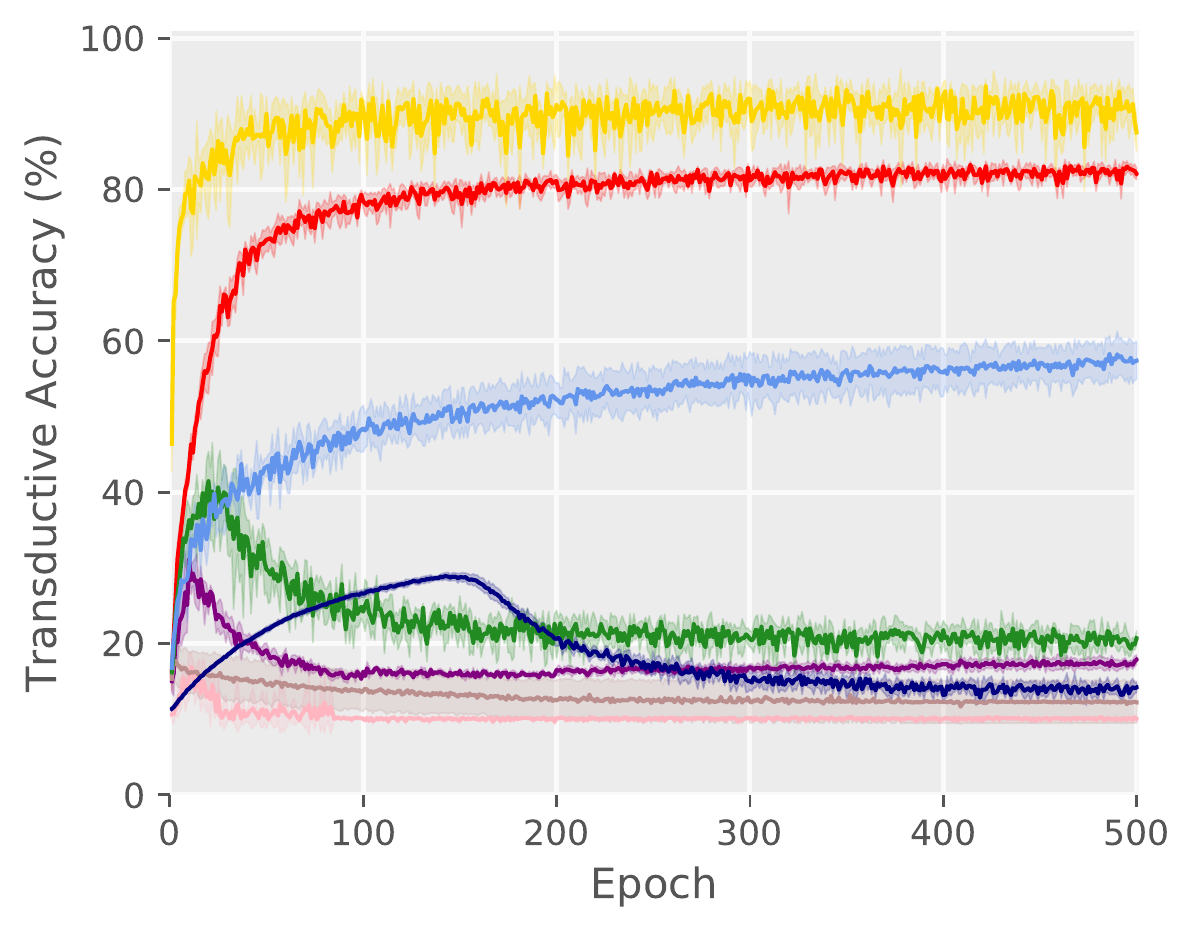}
			\centerline{\quad CIFAR, ResNet, $q=0.7$}
	\end{minipage}}
	\caption{Transductive accuracy for various models and datasets. Dark colors show the mean accuracy of 5 trials and light colors show standard deviation. Fashion is short for Fashion-MNIST, Kuzushiji is short of Kuzushiji-MNIST, CIFAR is short of CIFAR-10.}
	\label{fig:benchmark_binomial_transductive}
	\vskip -0.2in
\end{figure*}

\begin{figure*}[!t]
	\vskip 0.2in
	\subfigure{
		\begin{minipage}[b]{0.24\columnwidth}
			\centering
			\includegraphics[width=1.6in]{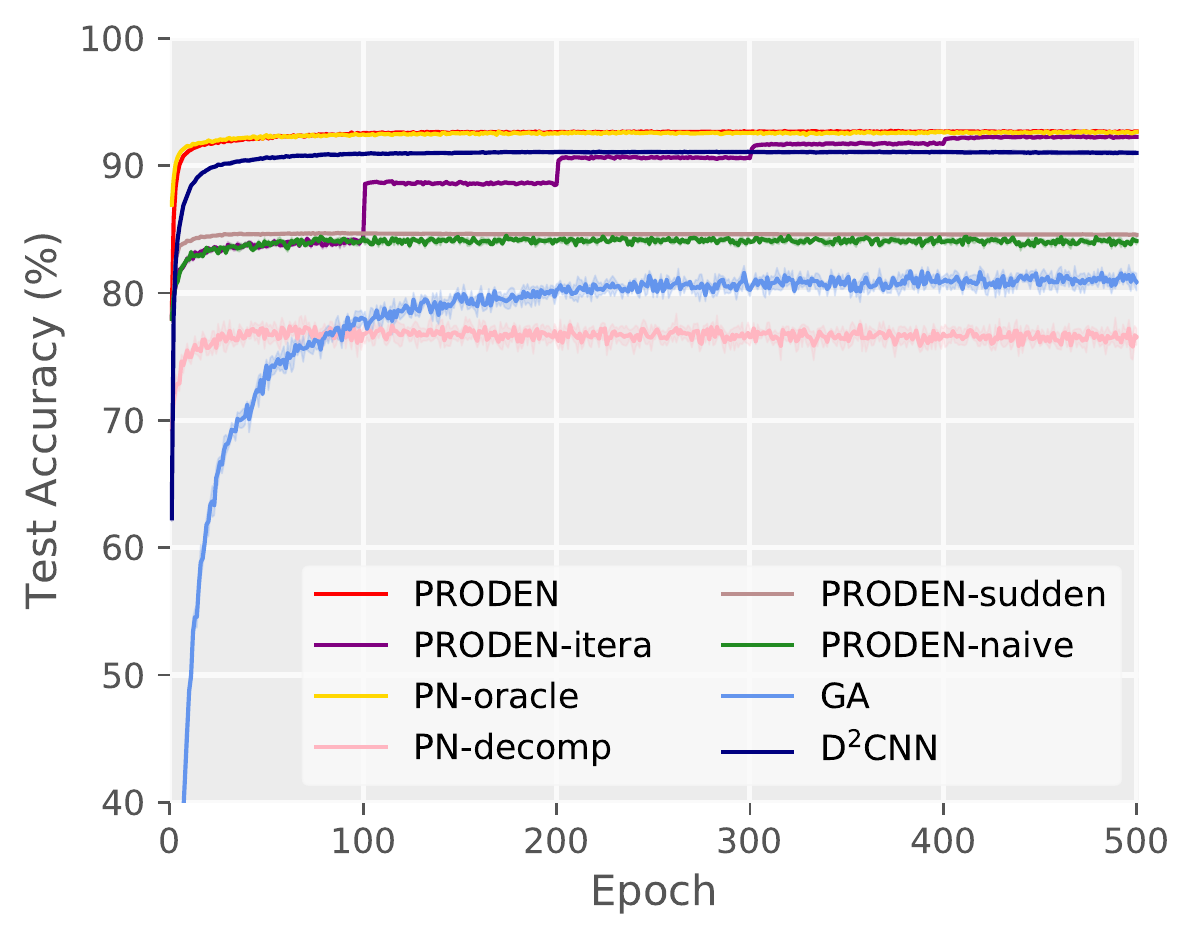}
			\centerline{MNIST, Linear, $q=0.5$}
	\end{minipage}}
	\subfigure{
		\begin{minipage}[b]{0.24\columnwidth}
			\centering
			\includegraphics[width=1.6in]{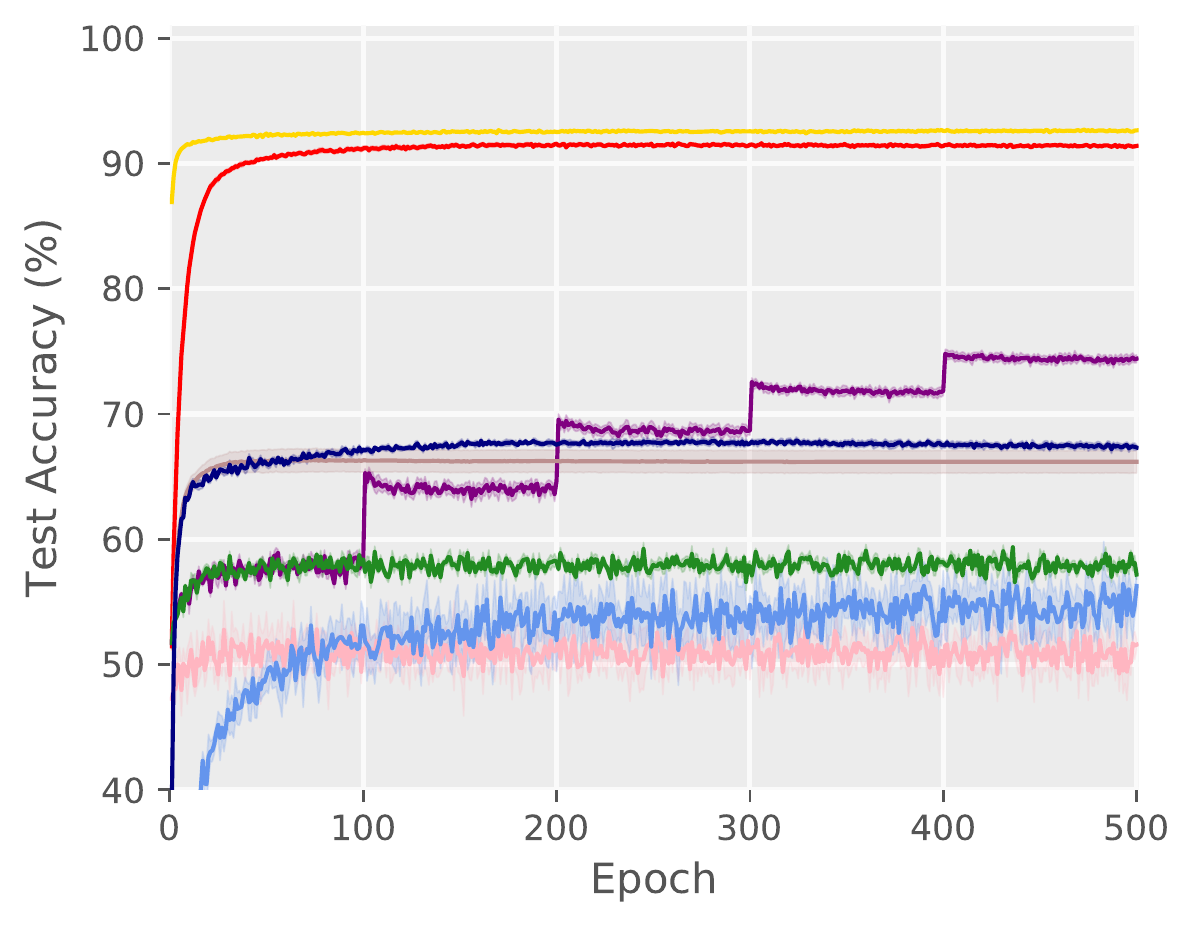}
			\centerline{\quad MNIST, Linear, $q=0.9$}
	\end{minipage}}
	\subfigure{
		\begin{minipage}[b]{0.24\columnwidth}
			\centering
			\includegraphics[width=1.6in]{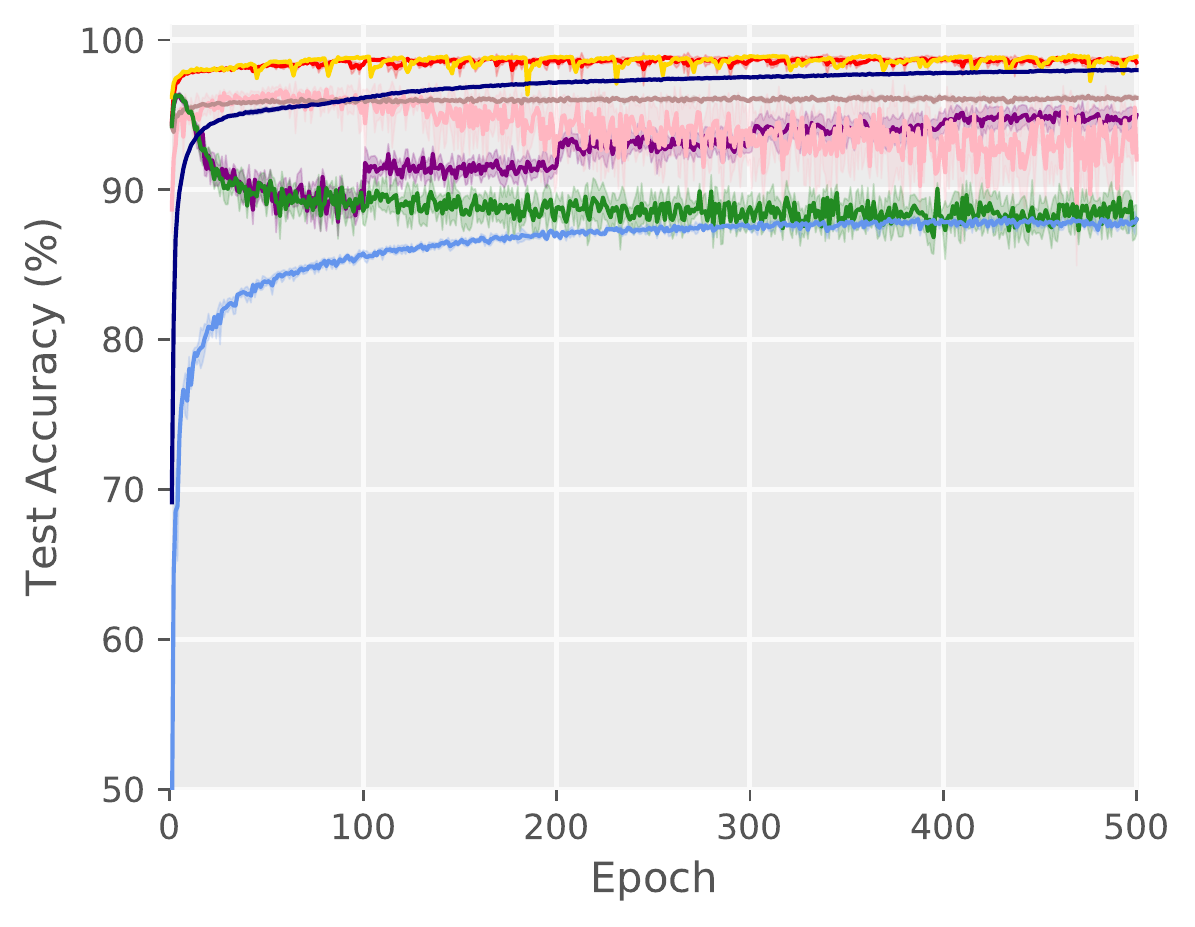}
			\centerline{\quad MNIST, MLP, $q=0.5$}
	\end{minipage}}
	\subfigure{
		\begin{minipage}[b]{0.24\columnwidth}
			\centering
			\includegraphics[width=1.6in]{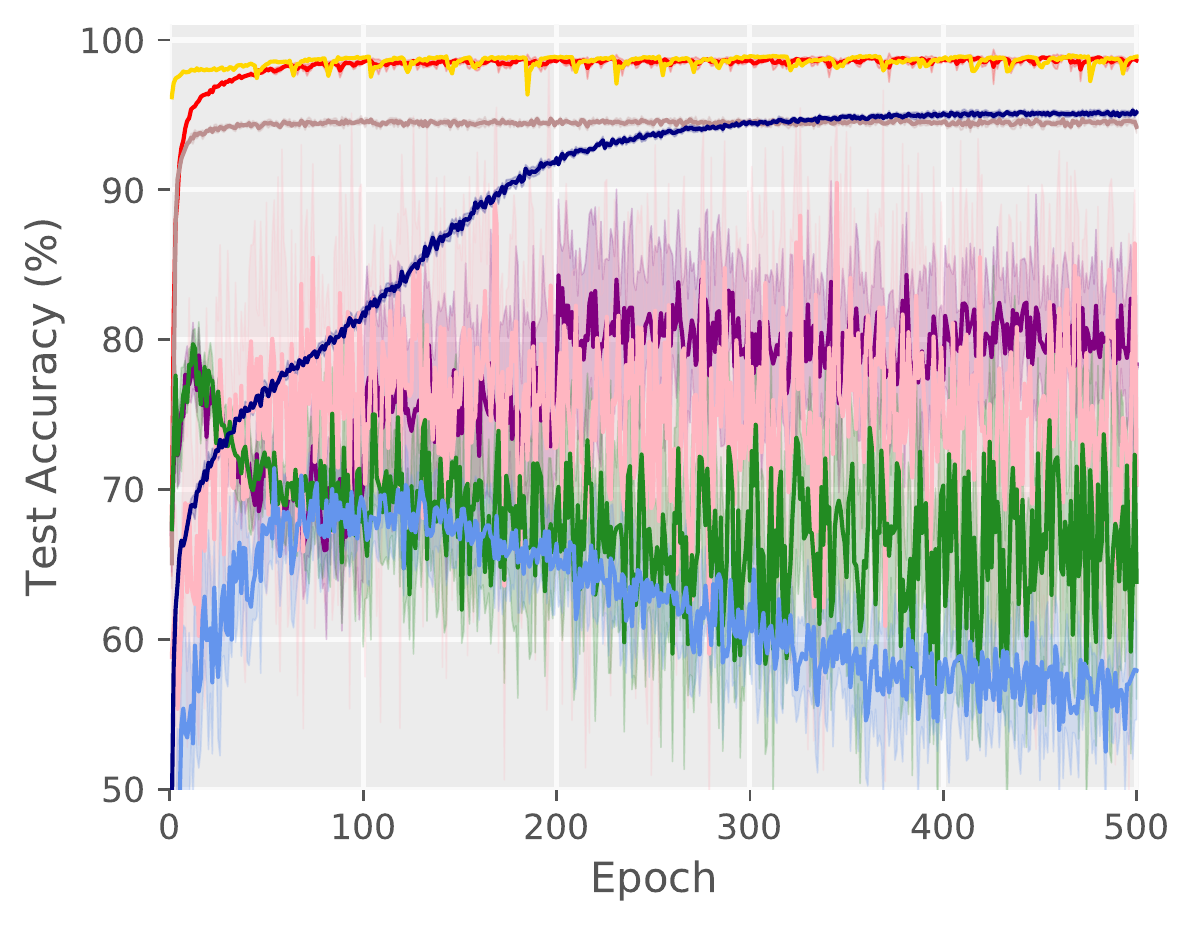}
			\centerline{\quad MNIST, MLP, $q=0.9$}
	\end{minipage}}
	
	\subfigure{
		\begin{minipage}[b]{0.24\columnwidth}
			\centering
			\includegraphics[width=1.6in]{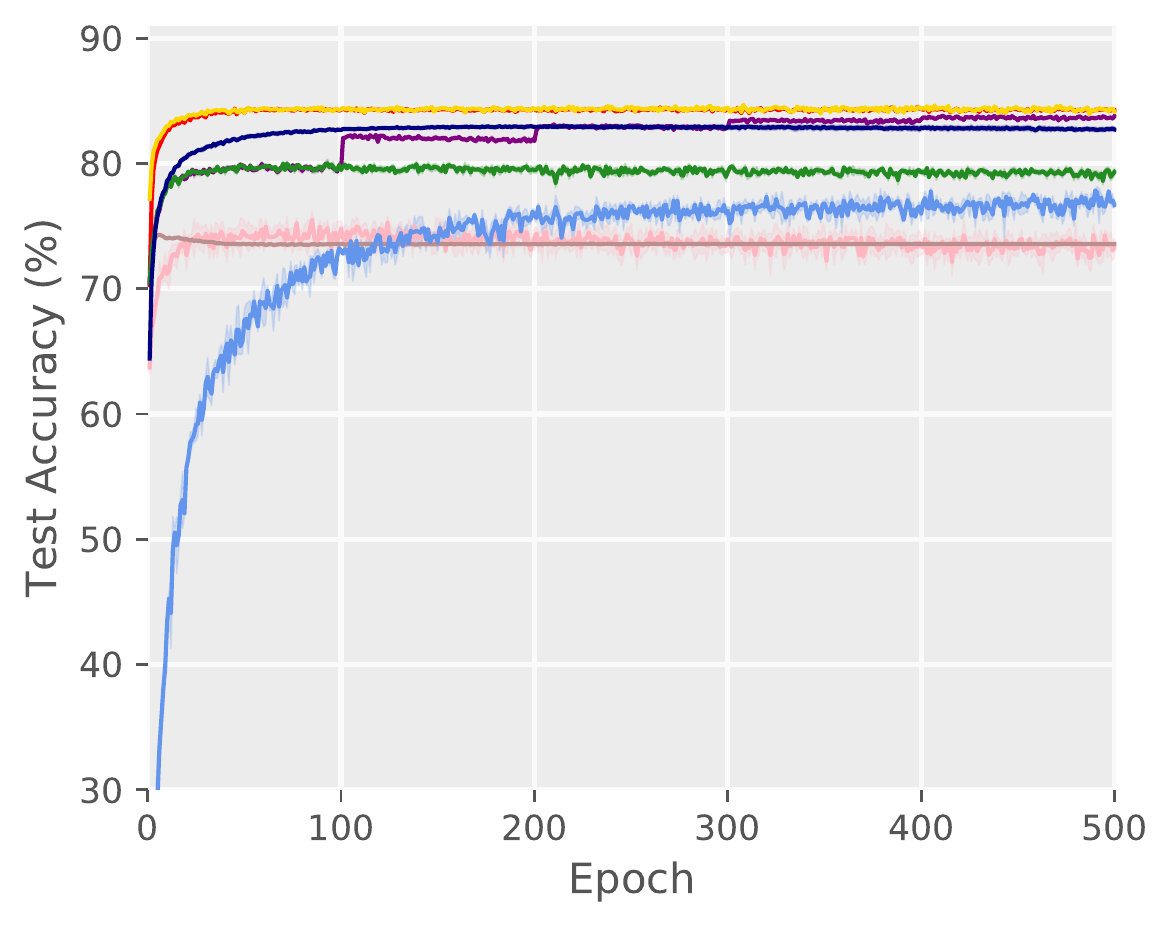}
			\centerline{Fashion, Linear, $q=0.5$}
	\end{minipage}}
	\subfigure{
		\begin{minipage}[b]{0.24\columnwidth}
			\centering
			\includegraphics[width=1.6in]{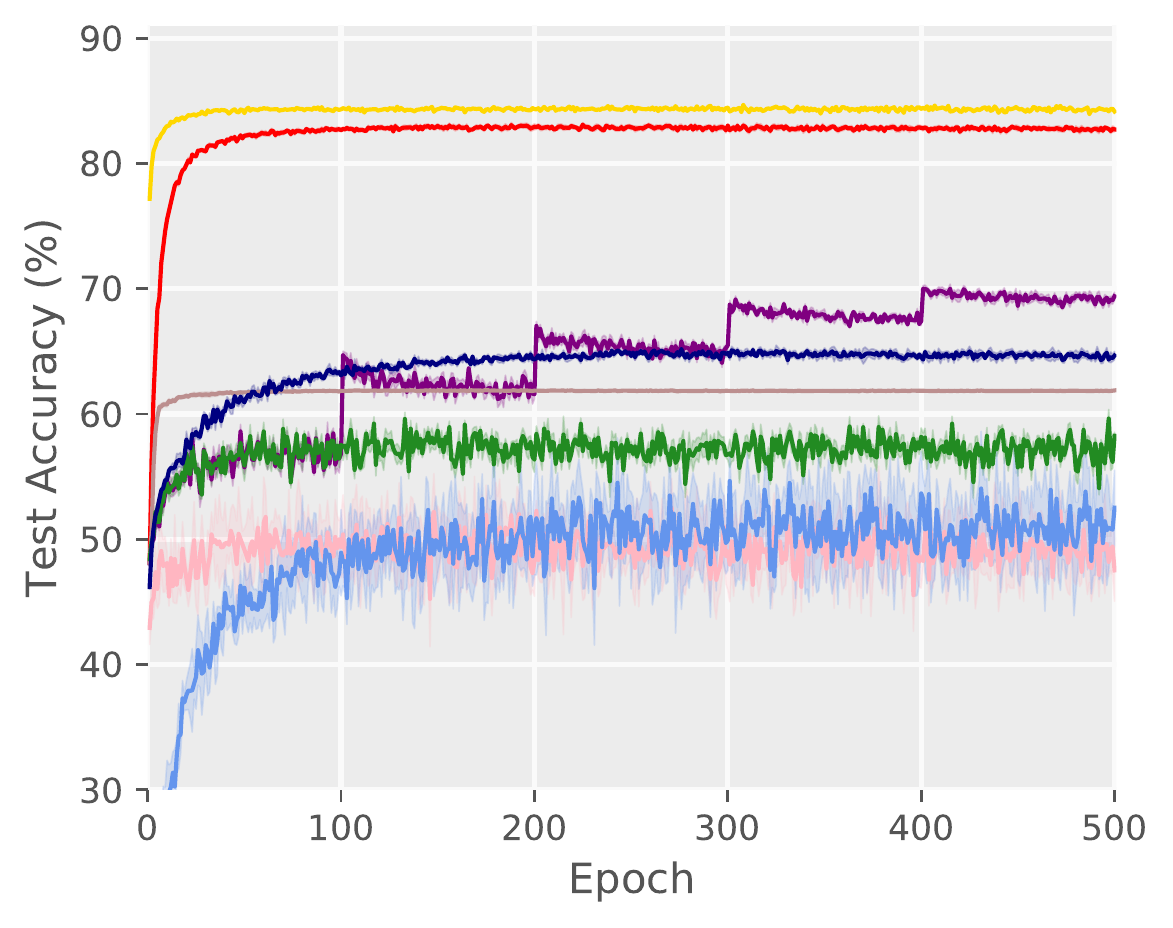}
			\centerline{\quad Fashion, Linear, $q=0.9$}
	\end{minipage}}
	\subfigure{
		\begin{minipage}[b]{0.24\columnwidth}
			\centering
			\includegraphics[width=1.6in]{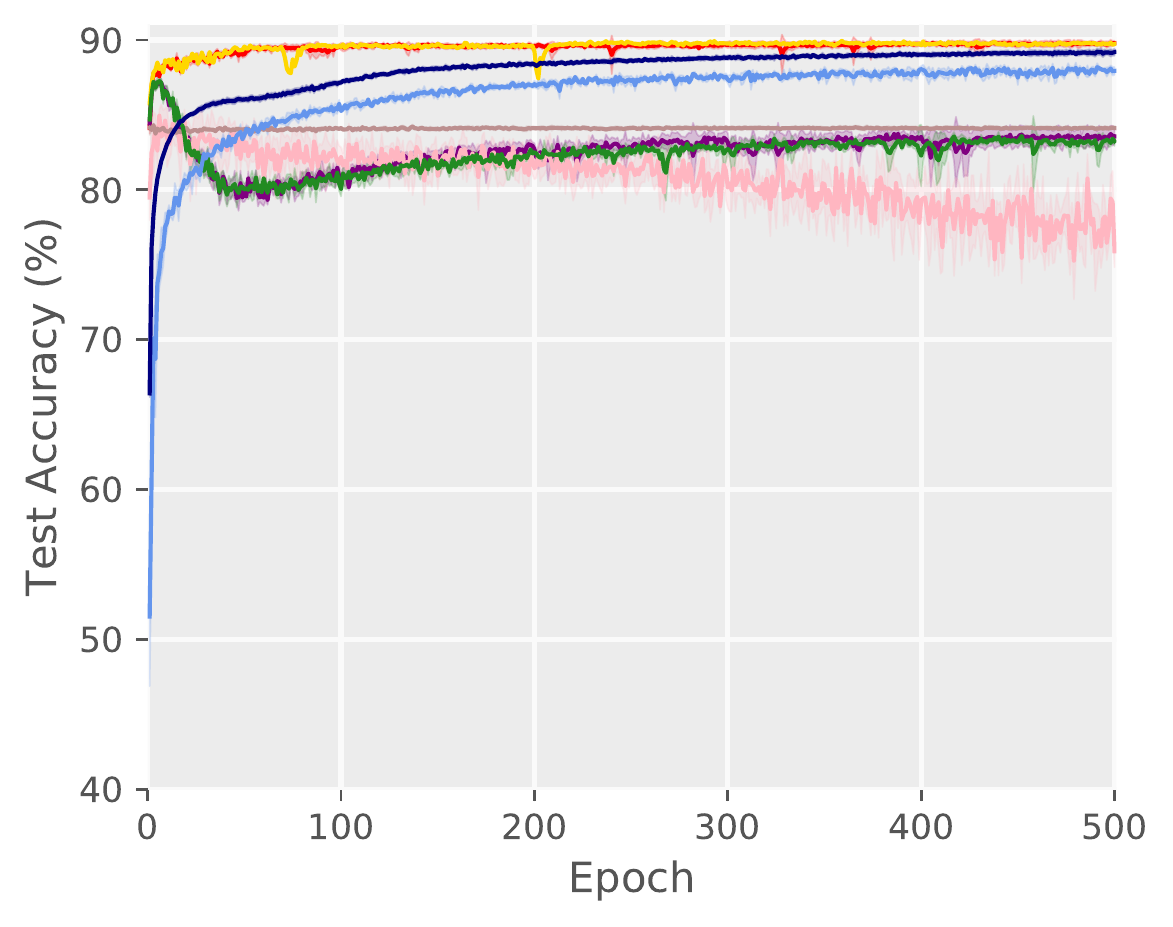}
			\centerline{\quad Fashion, MLP, $q=0.5$}
	\end{minipage}}
	\subfigure{
		\begin{minipage}[b]{0.24\columnwidth}
			\centering
			\includegraphics[width=1.6in]{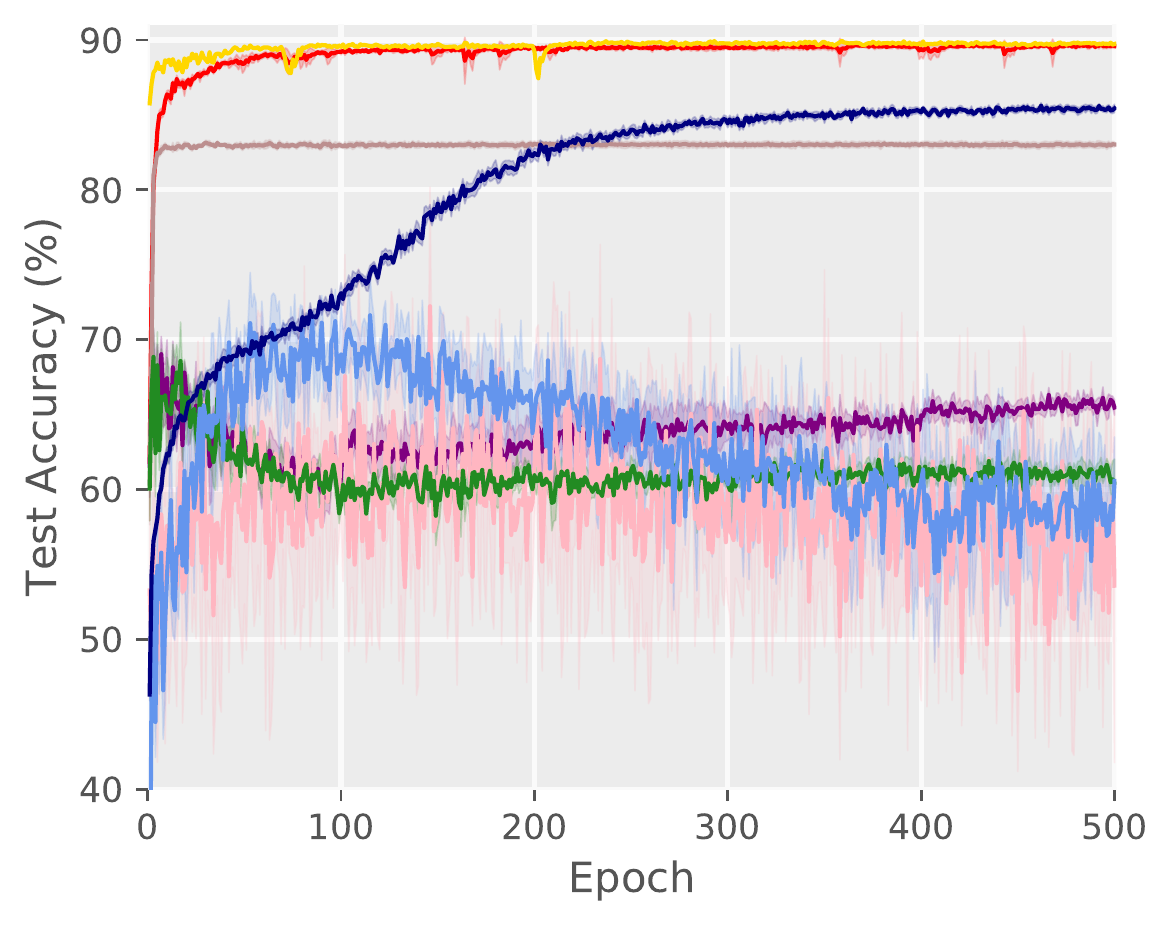}
			\centerline{\quad Fashion, MLP, $q=0.9$}
	\end{minipage}}
	
	\subfigure{
		\begin{minipage}[b]{0.24\columnwidth}
			\centering
			\includegraphics[width=1.6in]{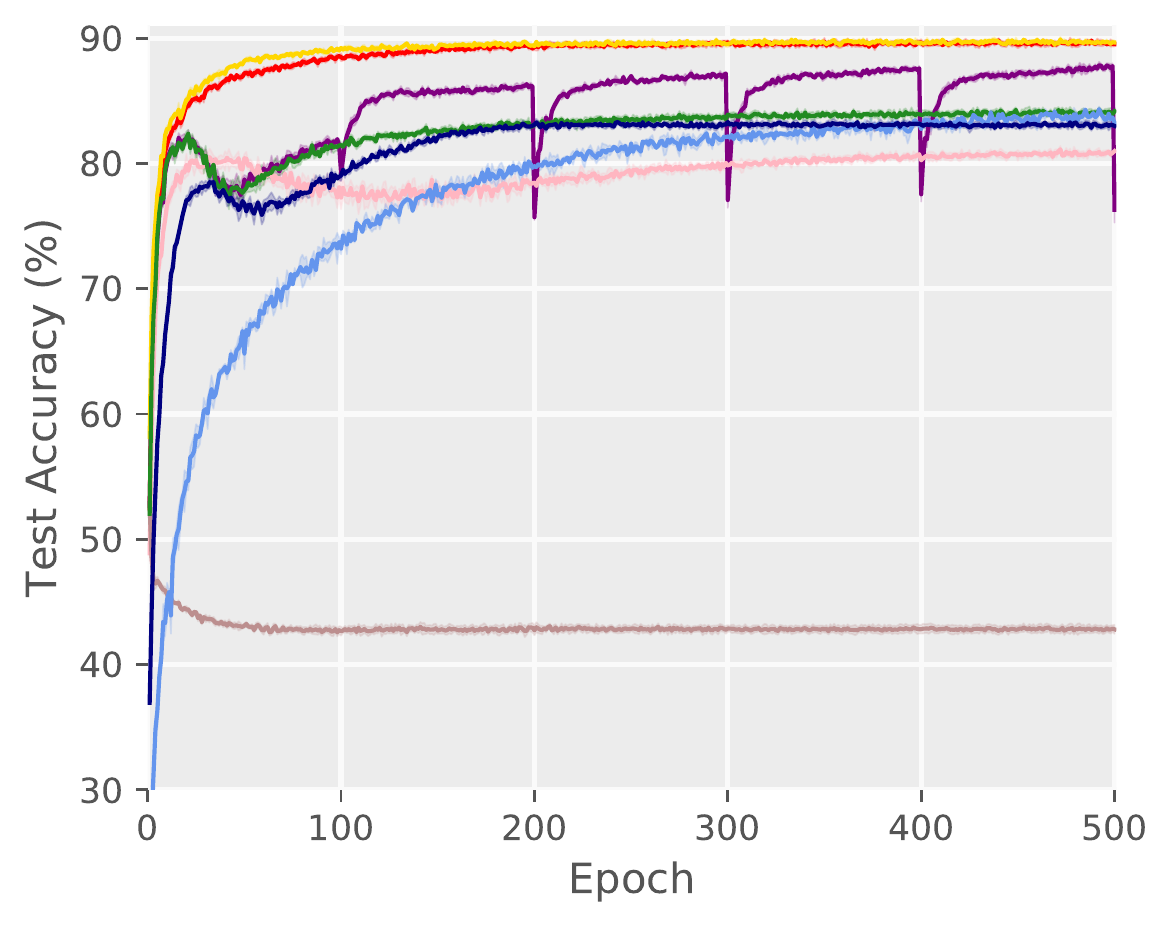}
			\centerline{CIFAR, ConvNet, $q=0.5$}
	\end{minipage}}
	\subfigure{
		\begin{minipage}[b]{0.24\columnwidth}
			\centering
			\includegraphics[width=1.6in]{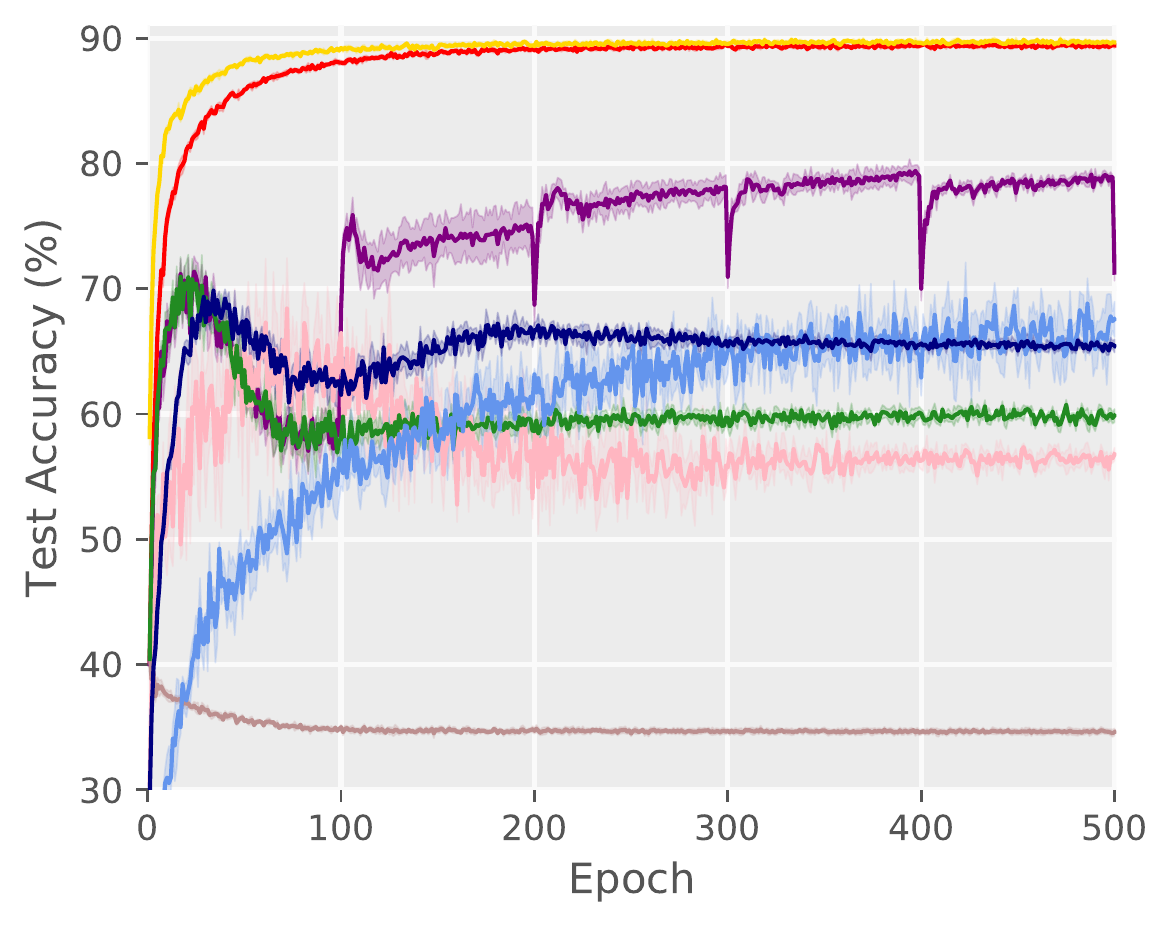}
			\centerline{\quad CIFAR, ConvNet, $q=0.9$}
	\end{minipage}}%
	\subfigure{
		\begin{minipage}[b]{0.24\columnwidth}
			\centering
			\includegraphics[width=1.6in]{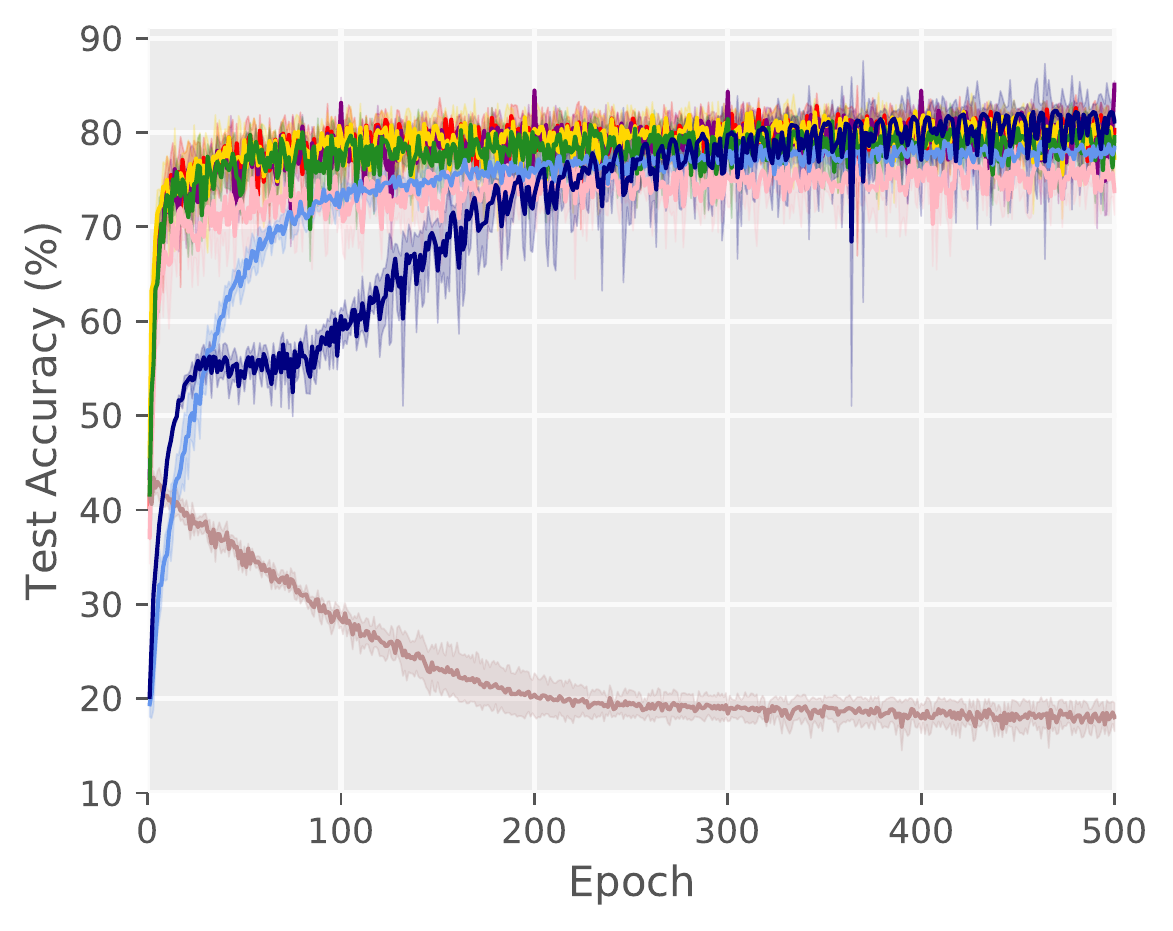}
			\centerline{\quad CIFAR, ResNet, $q=0.5$}
	\end{minipage}}
	\subfigure{
		\begin{minipage}[b]{0.24\columnwidth}
			\centering
			\includegraphics[width=1.6in]{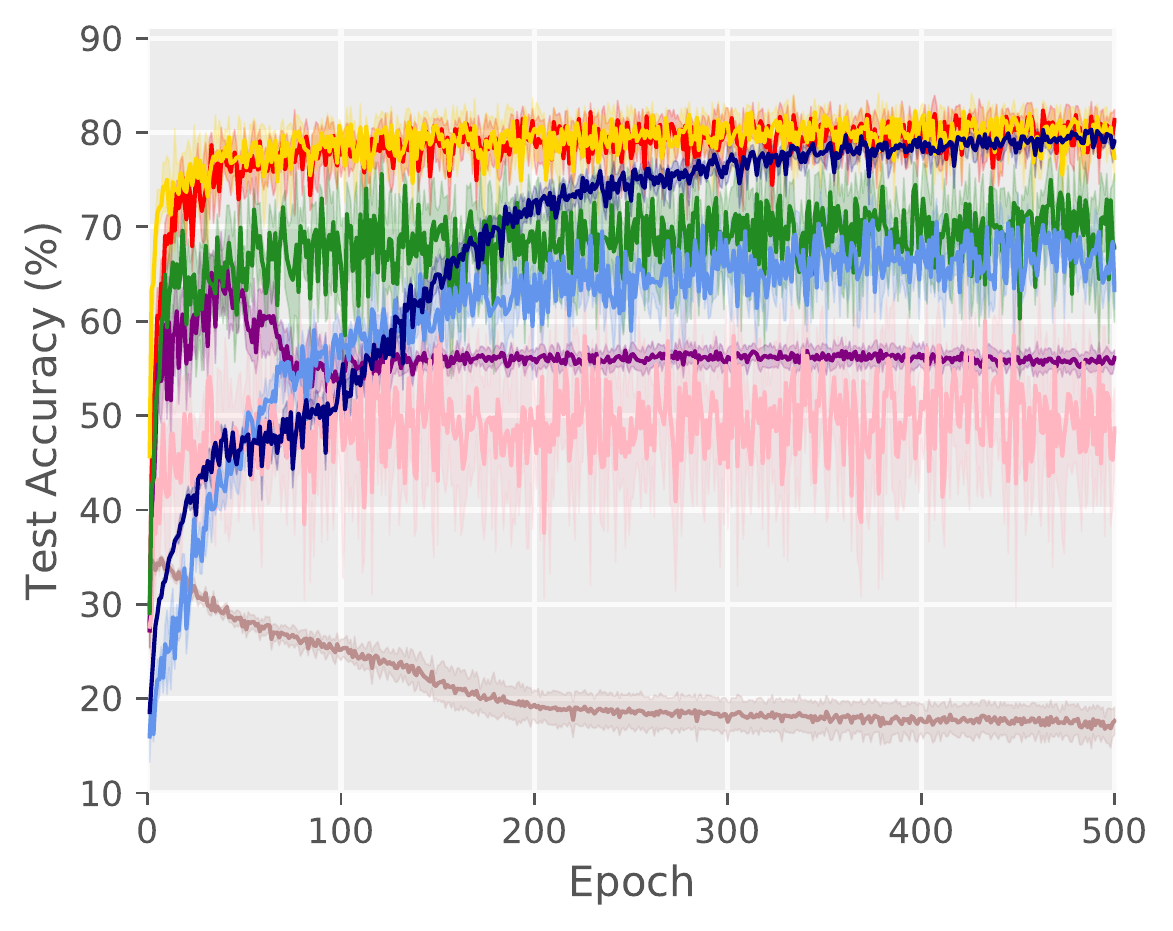}
			\centerline{\quad CIFAR, ResNet, $q=0.9$}
	\end{minipage}}
	\caption{Test accuracy on MNIST, Fashion-MNIST, and CIFAR-10 in the pair case.}
	\label{fig:benchmark_pair_2}
	\vskip -0.2in
\end{figure*}

\subsection{Test Results in the Pair Case}

Figure~\ref{fig:benchmark_pair_2} illustrates the test results on the benchmark datasets in the pair case. They show a similar phenomenon to Figure~\ref{fig:benchmark_pair_test} that PRODEN is affected slightly and CCN is affected severely when the ambiguity degree goes large.

\section{UCI datasets}

\subsection{Characteristic of the UCI Datasets and Setup}
Table~\ref{tab:uci} summaries the characteristic of the UCI datasets. We normalized these dataset by the Z-scores by convention and use the linear model trained by SGD with momentum 0.9. 

\begin{table*}[t]
	\caption{Summary of UCI datasets and models.}
	\label{tab:uci}
	\vskip 0.15in
	\begin{center}	
		\renewcommand\arraystretch{1.2}
		\begin{small}
			\setlength{\tabcolsep}{7.5mm}{
				\begin{tabular}{c|cccc|c}
					\toprule
					Dataset & \# Train & \# Test & \# Feature & \# Class & Model $\bm g(x;\Theta)$ \\
					\midrule
					Yeast & 1,335 & 149 & 8 & 10 & Linear model \\
					Texture & 4,950 & 550 & 40 & 11 & Linear model \\
					Dermatology & 329 & 37 & 34 & 6 & Linear model \\
					Synthetic-Control & 540 & 60 & 60 & 6 & Linear model \\
					20Newsgroups & 16,961 & 1,885 & 300 & 20 & Linear model \\
					\bottomrule
			\end{tabular}}
		\end{small}
	\end{center}
	\vskip -0.1in
\end{table*}

\subsection{Comparing Methods}

The comparing PLL methods are listed as follows.
\begin{itemize}[topsep=0ex,itemsep=0ex,leftmargin=*]
	\item \emph{SURE} \cite{feng2019partial}: an iterative EM-based method [suggested configuration: $\lambda, \beta \in \{0.001, 0.01, 0.05, 0.1, 0.3, 0.5, 1\}$].
	\item \emph{CLPL} \cite{cour2011learning}: a parametric method that transforms the PLL problem to the binary learning problem [suggested configuration: SVM with squared hinge loss].
	\item \emph{ECOC} \cite{zhang2017disambiguation}: a disambiguation-free method that adapts the binary decomposition strategy to PLL [suggested configuration: $L = log_2(l)$].
	\item \emph{PLSVM} \cite{Nguyen2008classification}: a SVM-based method that differentiates candidate labels from non-candidate labels by maximizing the margin between them [suggested configuration: $\lambda \in \{10^{-3}, \ldots, 10^3\}$].
	\item \emph{PLkNN} \cite{Hullermeier2006learning}: a non-parametric approach that adapts k-nearest neighbors method to handle partially labeled data [suggested configuration: $k \in \{5, 6, \ldots, 10\}$].
	\item \emph{IPAL} \cite{zhang2015solving}: a non-parametric method that applies the label propagation strategy to iteratively update the weight of each candidate label [suggested configuration: $\alpha = 0.95, k = 10, T = 100$].
\end{itemize}

\subsection{Results}

Tabel~\ref{tab:uci_pair_test} provides additional experiments to investigate the performances of each comparing methods on the UCI datasets with the pair flipping strategy. It shows that PRODEN generally achieves superior performance against other parametric comparing methods. Our advantage is a less obvious compared with the non-parametric method IPAL, whereas the performance of PRODEN could be easily increased by employing a deeper network.

\begin{table*}[!t]
	\centering
	\caption{Test accuracy (mean$\pm$std) on the UCI datasets in the pair case.}
	\label{tab:uci_pair_test}
	\vskip 0.15in
	\renewcommand\arraystretch{1.2}
	\begin{small}
		\setlength{\tabcolsep}{3.3mm}{
			\begin{tabular}{c|cccccc}
				\toprule
				&q & Yeast & Texture & Dermatology & Synthetic Control & 20Newsgroups\\
				\midrule
				\multirow{2}{*}{PRODEN}     &0.5              & \bf{56.38$\pm$4.71\%}     & \bf{99.71$\pm$0.12\%}      & \bf{96.16$\pm$3.27\%}   & 98.05$\pm$1.58\%         & 78.05$\pm$0.97\% \\ 
				&0.9 & \bf{44.03$\pm$4.12\%}     & \bf{99.35$\pm$0.31\%}      & 93.68$\pm$6.80\%  & 96.33$\pm$1.12\%         & \bf{70.71$\pm$0.72\%}   \\ 
				\midrule
				\multirow{2}{*}{PRODEN-itera}              &0.5     & 56.26$\pm$4.74\%     & 99.13$\pm$0.40\%$\bullet$      & 93.15$\pm$2.56\%   & 87.40$\pm$7.25\%$\bullet$         & 76.90$\pm$0.92\% \\
				&0.9 & 42.62$\pm$5.52\%     & 78.74$\pm$4.09\%$\bullet$      & 68.14$\pm$6.89\%$\bullet$   & 57.90$\pm$5.08\%$\bullet$         & 57.15$\pm$0.71\%$\bullet$ \\ 
				\midrule
				\multirow{2}{*}{GA}  &0.5& 24.39$\pm$4.40\%$\bullet$    &  94.25$\pm$0.48\%$\bullet$  & 73.81$\pm$6.18\%$\bullet$  &   64.68$\pm$3.08\%$\bullet$   & 63.68$\pm$0.67\%$\bullet$  \\
				&0.9& 16.50$\pm$3.43\%$\bullet$    &  61.85$\pm$2.29\%$\bullet$  & 48.03$\pm$11.42\%$\bullet$  &   37.15$\pm$6.91\%$\bullet$   & 45.86$\pm$0.95\%$\bullet$  \\
				\midrule
				\multirow{2}{*}{D$^2$CNN}  &0.5& 56.38$\pm$4.71\%    &  98.71$\pm$0.28\%$\bullet$  & 95.89$\pm$3.75\% &   78.87$\pm$11.94\%$\bullet$   & 74.38$\pm$0.90\%$\bullet$  \\
				&0.9& 41.52$\pm$7.03\%    &  86.45$\pm$4.87\%$\bullet$  & 87.84$\pm$6.58\%  &   62.92$\pm$12.36\%$\bullet$   & 64.16$\pm$0.36\%$\bullet$  \\
				\midrule
				\multirow{2}{*}{SURE}   &0.5& 51.69$\pm$3.81\%   &  98.18$\pm$0.17\%$\bullet$   &  95.71$\pm$2.49\%  &  78.67$\pm$5.26\%$\bullet$    &  70.21$\pm$0.88\%$\bullet$ \\ 
				&0.9& 37.61$\pm$3.40\%$\bullet$   &  98.00$\pm$0.42\%$\bullet$   &  93.42$\pm$6.23\%   &  52.33$\pm$6.49\%$\bullet$    &  61.01$\pm$0.93\%$\bullet$ \\
				\midrule
				\multirow{2}{*}{CLPL} &0.5& 55.06$\pm$4.74\% & 98.80$\pm$0.22\%$\bullet$   & 94.52$\pm$3.36\%   &  75.83$\pm$4.29\%$\bullet$ &  77.92$\pm$0.76\% \\
				&0.9& 40.66$\pm$4.69\% & 90.21$\pm$4.77\%$\bullet$   & 87.67$\pm$3.06\%   &  52.33$\pm$4.65\%$\bullet$ &  65.03$\pm$0.32\%$\bullet$ \\ 
				\midrule
				\multirow{2}{*}{ECOC}  &0.5&  54.55$\pm$3.92\%  & 99.47$\pm$0.17\%$\bullet$   &     94.71$\pm$2.29\%    &  96.67$\pm$2.08\% & \bf{78.67$\pm$1.11\%} \\
				&0.9&  42.51$\pm$5.19\%  & 69.69$\pm$4.82\%$\bullet$   &     92.97$\pm$6.61\%    &  94.50$\pm$1.32\% & 70.11$\pm$1.63\% \\
				\midrule
				\multirow{2}{*}{PLSVM} &0.5& 52.59$\pm$2.04\% & 93.38$\pm$2.22\%$\bullet$   &  93.97$\pm$4.50\%   &  91.83$\pm$3.08\%$\bullet$ & 76.21$\pm$2.30\% \\
				&0.9& 41.89$\pm$3.90\% & 82.24$\pm$6.58\%$\bullet$   &  93.15$\pm$4.22\%    &  80.67$\pm$9.42\%$\bullet$ & 70.76$\pm$2.16\% \\
				\midrule
				\multirow{2}{*}{PL$k$NN}    &0.5& 53.60$\pm$2.81\% & 97.11$\pm$0.28\%$\bullet$  &  94.52$\pm$3.06\%   & 94.83$\pm$2.60\%$\bullet$ & 43.77$\pm$0.72\%$\bullet$  \\
				&0.9& 43.80$\pm$4.50\% & 92.98$\pm$0.64\%$\bullet$  &  92.05$\pm$4.38\%   & 89.83$\pm$4.54\%$\bullet$ & 38.77$\pm$0.90\%$\bullet$   \\
				\midrule
				\multirow{2}{*}{IPAL}  &0.5& 51.80$\pm$5.08\%    &  99.30$\pm$0.37\%  & 95.89$\pm$2.74\%  &   \bf{98.50$\pm$1.37\%}   & 76.83$\pm$0.51\%  \\
				&0.9& 36.60$\pm$2.58\%$\bullet$    &  98.95$\pm$0.58\%$\bullet$  & \bf{95.34$\pm$2.29\%}$\circ$  &   \bf{98.50$\pm$1.09\%}$\circ$   & 69.15$\pm$0.73\% \\
				\bottomrule
			\end{tabular}
		}
	\end{small}
\end{table*}

\begin{table*}[]
	\caption{Summary of real-world partial label datasets.}
	\label{tab:realworld}
	\vskip 0.15in
	\begin{center}
		\renewcommand\arraystretch{1.2}
		\begin{small}
			\setlength{\tabcolsep}{1.6mm}{
				\begin{tabular}{c|ccccc|c}
					\toprule
					Dataset       & \# Examples & \# Feature & \# Class & \# Avg. CLs & Task Domain  & Model $\bm g(x; \Theta)$   \\ 
					\midrule
					Lost          & 1122        & 108        & 16       & 2.23        & automatic face naming \cite{panis2014overview} & Linear model \\
					BirdSong      & 4998        & 38         & 13       & 2.18        & bird song classification \cite{briggs2012rank} & Linear model \\
					MSRCv2        & 1758        & 48         & 23       & 3.16        & object classification \cite{liu2012conditional} & Linear model \\
					Soccer Player & 17472       & 279        & 171      & 2.09        & automatic face naming \cite{zeng13learning} & Linear model \\
					Yahoo! News        & 22991       & 163        & 219      & 1.91        & automatic face naming \cite{guillaumin2010multiple} & Linear model \\
					\bottomrule
			\end{tabular}}
		\end{small}
	\end{center}
	\vskip -0.1in
\end{table*}

\section{Characteristic of the Real-world Datasets and Setup}

Tabel~\ref{tab:realworld} summarizes the characteristic of the real-world datasets and the corresponding models. The preprocessing, model and optimizer were the same as UCI datasets.

\end{document}